\documentclass[twocolumn]{article}

\usepackage[margin=1in]{geometry}
\usepackage{amsthm,amsmath,amsfonts,amssymb}
\usepackage[numbers]{natbib}
\usepackage{amssymb,amsthm}
\usepackage{booktabs}
\usepackage{graphicx}
\usepackage{algorithm}
\usepackage{algorithmic}
\usepackage{hyperref}
\usepackage{enumitem}
\theoremstyle{plain}

\newtheorem{theorem}{Theorem}[section]
\newtheorem{lemma}[theorem]{Lemma}

\newtheorem{proposition}[theorem]{Proposition}

\theoremstyle{definition}
\newtheorem*{assumption}{Assumption}
\newtheorem*{condition}{Condition}

\theoremstyle{remark}

\newtheorem{remark}[theorem]{Remark}

\newcommand{\reals}{\mathbb{R}}
\newcommand{\gauss}{\mathcal{N}}
\newcommand{\E}{\mathbb{E}}
\newcommand{\Var}{\mathrm{Var}}

\newcommand{\pf}{\#}                        
\newcommand{\target}{\pi}
\newcommand{\refm}{\rho}                    
\newcommand{\pz}{\pi_z}                     
\newcommand{\hpi}{\hat{\pi}}
\newcommand{\Herm}{H}
\newcommand{\dict}{G}
\newcommand{\Poi}{\hat{G}}
\newcommand{\op}[1]{\mathsf{#1}}

\newcommand{\QGI}{Q_{\mathrm{GI}}}

\newcommand{\he}{\mathrm{he}}
\newcommand{\schA}{\mathrm{A}}
\newcommand{\schB}{\mathrm{B}}
\newcommand{\schAa}{\mathrm{A}^{\mathrm{a}}}
\newcommand{\Fisher}{\mathcal F}
\title{Transformed Samplers with Variance Reduction}
\author{Siran Liu\thanks{Department of Statistical Science, University College London, UK. Email: \href{mailto:siran.liu.21@ucl.ac.uk}{siran.liu.21@ucl.ac.uk}}
\and Michalis K. Titsias\thanks{Google DeepMind, UK. Email: \href{mailto:mtitsias@gmail.com}{mtitsias@gmail.com}}
\and Petros Dellaportas\thanks{Department of Statistical Science, University College London, UK, and Department of Statistics, Athens University of Economics and Business, Greece. Email: \href{mailto:p.dellaportas@ucl.ac.uk}{p.dellaportas@ucl.ac.uk}}
}
\date{}
\begin{document}

\maketitle

\begin{abstract}
Markov chain Monte Carlo (MCMC) methods are the standard tool for
computing expectations under complex probability distributions. Control
variates reduce the variance of the resulting estimates, but a good
control variate requires solving the Poisson equation of the sampler,
which rarely admits a closed-form solution. Exact solutions are available
when the sampler's kernel has a known spectral decomposition on a simple
reference density. In our work, we extend these solutions to general targets
through a learned change of variables. A bijection, such as a normalizing
flow, is trained so that the target becomes close to the reference in a
latent space, and we show that Markov kernels and their Poisson solutions
are transformed by any bijection. Running such samplers in the latent
space then yields explicit control variates, and the estimator is consistent under mild tail conditions
on the map and target. Importance sampling (IS) from the flow is the limiting case of the same
construction and the control variates apply to it as well. Experiments on
synthetic targets and real posteriors compare the procedure against
state-of-the-art samplers and control variates.


\end{abstract}

\section{INTRODUCTION}
\label{sec:intro}

Markov chain Monte Carlo (MCMC) estimates an expectation
$\target(F)=\int F\,d\target$ by the ergodic average
$\mu_n(F)=\frac1n\sum_{t=0}^{n-1}F(X_t)$ along a chain with
$\target$-invariant kernel $P$. Its accuracy is governed by the asymptotic
variance, and the classical route to reducing it is a control variate built
from the Poisson equation
\begin{equation}
(I-P)\,\Poi \;=\; F-\target(F),
\label{eq:poisson}
\end{equation}
whose solution yields a zero-variance estimator
\citep{henderson1997variance,meyn2009markov,dellaportas2012control}. The analytical
solution is almost never available: even for a Gaussian target and $F(x)=x$,
no closed form is known for random-walk Metropolis (RWM) or the
Metropolis-adjusted Langevin algorithm (MALA) \citep{alexopoulos2023variance}.
Closed-form solutions are known only in special cases
\citep{dellaportas2012control,liu2024variance}, most notably for the
Gaussian-invariant (GI) samplers of \citet{titsias2026gaussian}, whose
proposals are always accepted on a Gaussian target and for which
\eqref{eq:poisson} has exact solutions. On non-Gaussian targets these
solutions are reused through an approximation, which covers only a few test
functions and loses its effect as the target departs from a Gaussian.

A parallel line of work addresses non-Gaussian geometry through a change of
variables. Transport-map MCMC learns a bijection $T$ with
$T_\pf\gauss(0,I)\approx\target$, typically a normalizing flow
\citep{rezende2015variational,papamakarios2021normalizing}, and either runs
the chain in the latent space
\citep{parno2018transport,hoffman2019neutra,cabezas2023transport} or
proposes from the flow \citep{albergo2019flow,gabrie2022adaptive,
brofos2022adaptation,samsonov2022local,cabezas2024markovian}. In all of
these methods the map enters only through the sampler, and the estimator
remains the plain ergodic average.

This paper connects the two lines by showing that when a Markov kernel is
transformed by a bijection, its Poisson solutions are transformed with it
(Section~\ref{sec:transport}). Consequently, any sampler whose kernel has a
known spectral decomposition on a reference distribution, and whose
proposal has a computable conditional expectation, yields exact Poisson
solutions that can be carried to every target reached by a map
(Theorem~\ref{thm:frozen}) and used there as control variates
(Theorem~\ref{thm:validity}). The Gaussian reference with GI kernels and
Hermite polynomials is the cleanest instance and the focus of this paper
(Section~\ref{sec:reference}). Since the
construction depends on the kernel only through its spectrum, it also
covers elliptical slice sampling and, as the kernel with all eigenvalues
zero, importance sampling (IS) from the flow; a short pilot run tells which of
them to use on a given target (Section~\ref{sec:diagnostics}). References
other than the Gaussian, such as Gamma distributions with Laguerre
polynomials and spheres with spherical harmonics, are treated in
Appendix~\ref{app:general}.

A trained flow thus improves the estimator as well as the sampler. It
provides a control variate for any test function, with coefficients
computed from cheap reference draws rather than fitted to the chain,
without additional evaluations of the target and, for the random-walk
sampler, without gradients at sampling time. The control variates of
\citet{titsias2026gaussian} are the special case of an affine map, and the
estimator is consistent under mild tail conditions on the map, so a poor
map costs variance rather than correctness. In turn, the control variates
benefit from importance weights that stay
bounded in the tails, which
the usual training objectives do not enforce, and we add a penalty that
helps (Section~\ref{sec:algorithm}).

The paper is organized as follows. Section~\ref{sec:background} reviews
Poisson control variates, GI samplers and transformed targets, and
Section~\ref{sec:related} discusses related work. Section~\ref{sec:method}
shows how kernels and their Poisson solutions are transformed by a
bijection, identifies the reference kernels whose Poisson equation can be
solved exactly, and derives the resulting control variates together with
their validity (Section~\ref{sec:estimator}). Section~\ref{sec:practical}
covers diagnostics for choosing the estimator, the pruning of the
dictionary in high dimension (Section~\ref{sec:pruning}) and the choice of
step size (Section~\ref{sec:gamma}), and Section~\ref{sec:experiments}
compares the estimators on the same flow with Hamiltonian Monte Carlo (HMC) in
the latent space and zero-variance control variates, on
synthetic targets and real posteriors.

\section{BACKGROUND}
\label{sec:background}

\subsection{Poisson control variates}
\label{sec:bg-cv}

If $\Poi$ solves \eqref{eq:poisson}, then $F+P\Poi-\Poi\equiv\target(F)$,
where $(P\Poi)(x)=\int P(x,dy)\,\Poi(y)$, so adding $P\Poi-\Poi$ to the
ergodic average removes its variance \citep{dellaportas2012control}. In
practice $PG$ is intractable for a candidate $G$, since it integrates over
the accept--reject step. For a Metropolis--Hastings (MH) sampler we estimate
it with one sample, using the proposal $Y_t\sim q(\cdot|X_t)$ drawn at each
step and its acceptance probability $\alpha_t$. The estimator is
\begin{equation}
\mu_{n,G}(F)=\frac1n\sum_{t=0}^{n-1}
\big\{F(X_t)+\theta_1H_{1,t}+\theta_2H_{2,t}\big\},
\label{eq:twoterm}
\end{equation}
with $H_{1,t}=\alpha_t\{G(Y_t)-G(X_t)\}$ and
$H_{2,t}=G(Y_t)-\E_{q(\cdot|X_t)}[G]$. The first term is an unbiased
single-draw estimate of $(P-I)G(X_t)$ and the second is conditionally
centered, so both have stationary mean zero for every integrable $G$, and
the choice of $G$ affects only the variance. In general
$\theta:=(\theta_1,\theta_2)$ can be selected to further minimize the
asymptotic variance. Unless stated otherwise we use $\theta=(1,-1)$, and when
every proposal is accepted the summand equals $F+PG-G$ exactly. The
estimator therefore needs a $G$ close to $\Poi$ for which
$\E_{q(\cdot|x)}[G]$ is available in closed form.

\subsection{Gaussian-invariant samplers}
\label{sec:bg-gi}

Both requirements are met by the GI samplers of \citet{titsias2026gaussian}
on a Gaussian target. They use Gaussian proposals
\begin{equation}
q(y|x)=\gauss\!\big(y\,\big|\,x+\gamma A_x\nabla\log\target(x),\;
(2\gamma-\gamma^2)A_x\big),
\label{eq:gi}
\end{equation}
with step size $\gamma\in(0,2)$, where the preconditioner $A_x$ equals
$\Sigma$ whenever the target is $\gauss(\mu,\Sigma)$. GI-MALA takes
$A_x=\Sigma$ and GI-pMALA a state-dependent $A_x$, such as an approximate
negative inverse Hessian. The gradient-free GI-RWM replaces the mean by
$(1-\gamma)x+\gamma\mu$, which is the preconditioned Crank--Nicolson
proposal \citep{neal1999regression,cotter2013mcmc}. On a Gaussian target all
three coincide, every proposal is accepted, and in the standardized
variable $u=\Sigma^{-1/2}(x-\mu)$ the chain is the autoregression
$u'=\beta u+\sqrt{1-\beta^2}\,\xi$ with $\beta=1-\gamma$ and
$\xi\sim\gauss(0,I)$. Its eigenfunctions are the normalized Hermite polynomials
$\Herm_\kappa(u)=\prod_{j=1}^d\he_{\kappa_j}(u_j)/\sqrt{\kappa_j!}$, indexed by
$\kappa\in\mathbb N^d$ with degree $|\kappa|=\sum_j\kappa_j$, which form an
orthonormal basis of the square-integrable functions under $\gauss(0,I)$:
\begin{equation}
\E\{\Herm_\kappa(u')\,|\,u\}=\beta^{|\kappa|}\,\Herm_\kappa(u).
\label{eq:hermite-eig}
\end{equation}
Expanding $F$ in Hermite polynomials therefore solves \eqref{eq:poisson}
term by term, each degree-$k$ term being divided by $1-\beta^k$, and since
$q$ is Gaussian, $\E_{q(\cdot|x)}[G]$ is available in closed form for every
polynomial $G$. On a non-Gaussian target, however, the proposals are no
longer always accepted and the Hermite polynomials are no longer
eigenfunctions of the kernel. What makes these solutions exact is only that
the kernel has a known spectral decomposition on the Gaussian;
Section~\ref{sec:transport} shows that this property is all that the
construction requires.

\section{RELATED WORK}
\label{sec:related}

Other than the exact solutions used in this paper, the Poisson equation can
be solved approximately, through a weakly approximating chain
\citep{mijatovic2018poisson}, diffusion approximations
\citep{brosse2024diffusion}, empirical variance minimization
\citep{belomestny2020variance}, neural networks
\citep{belomestny2024theoretical} or coupled chains
\citep{douc2026solving,piponi2020hamiltonian}. These methods approximate
the Poisson solution of the given sampler on the target itself, whereas we
solve it exactly for a reference kernel and transform the solution through
the map. 

Beyond the Poisson equation, zero-variance control variates
\citep{assaraf1999zero,mira2013zero,papamarkou2014zero}, control
functionals \citep{oates2017control} and their regularized, semi-exact and
scalable versions \citep{south2023regularized,south2022semi,si2022scalable}
apply the Stein operator to basis functions and fit the coefficients on
the draws; see \citet{south2022postprocessing} for a review. The
martingale control variates of \citet{belomestny2022martingale} use
conditionally centered proposal terms, similar to the second term of
\eqref{eq:twoterm}, also fitted on the chain. Control variates for
IS are likewise fitted on the weighted draws
\citep{owen2000safe}, or built from orthogonal functions as the spherical
harmonics of \citet{leluc2024sliced}. In all of these the accuracy of the
coefficients is set by the number of draws, and the Stein approach needs
the score; our coefficients can be estimated from reference draws independent of
the chain, without a regression, and the gradient-free sampler needs no
score. Section~\ref{sec:experiments} compares both kinds of control
variates on the same flow.

On the sampler side, the preconditioned Crank--Nicolson proposal
\citep{neal1999regression,beskos2008mcmc,cotter2013mcmc}, elliptical slice
sampling \citep{murray2010elliptical}, Hamiltonian Monte Carlo with an
exact Gaussian integrator \citep{beskos2011hybrid} and the auxiliary
gradient-based samplers of \citet{titsias2018auxiliary} leave a Gaussian
prior invariant, chosen in advance, so they perform well only to the
extent that the target resembles it. We instead keep the standard
Gaussian fixed and train the map to bring the target close to it.
HMC on the latent space of the same flow
\citep{hoffman2019neutra} is the natural alternative there and is our main
baseline; flows also serve as proposals for independence sampling and
annealed transport
\citep{noe2019boltzmann,arbel2021annealed,matthews2022continual,midgley2023flow}.
Markov kernels with orthogonal polynomial eigenfunctions are classical
\citep{diaconis2008gibbs} and underlie the non-Gaussian references of
Appendix~\ref{app:general}.

\section{TRANSFORMED CONTROL VARIATES}
\label{sec:method}

Throughout this section, $\refm$ is a reference distribution on a space
$\mathcal Z$ and $T:\mathcal Z\to\mathcal X$ is a bijection. We write
$T_\pf\mu$ for the law of $T(\zeta)$ when $\zeta\sim\mu$, so that
$\hpi_T=T_\pf\refm$ is the distribution generated by the map and
$\pz=(T^{-1})_\pf\target$ is the pull-back of the target, with
$\log\pz(z)=\log\target(T(z))+\log|\det\nabla T(z)|$, so that $Z\sim\pz$
if and only if $T(Z)\sim\target$; $w=d\pz/d\refm$ is the importance weight
of the map, and $w\equiv1$ exactly when $\hpi_T=\target$. For a probability measure $\mu$, $L^2(\mu)$ denotes
the square-integrable functions with inner product
$\langle\cdot,\cdot\rangle_\mu$, and $L^2_0(\mu)$ those with $\mu(h)=0$.  A latent chain $\{Z_t\}$ targets $\pz$ with kernel $P_z$; at each step it
  draws a candidate $Z'_t\sim q(\cdot|Z_t)$ and then moves to $Z_{t+1}$, and
  we set $X_t=T(Z_t)$ and $Y_t=T(Z'_t)$. A test function $F$ on $\mathcal X$ is
evaluated in the latent space as $F\circ T$, that is, $z\mapsto F(T(z))$.

\subsection{Transformation of Kernels and Poisson Solutions}
\label{sec:transport}

We first show that a change of variables carries a sampler on the latent
space, together with its Poisson solutions, to a sampler on the target
space. This is what allows a solution computed where the Poisson equation is tractable to be carried to $\hat\pi_T$ exactly, and used on $\pi$ as an approximation whose quality is set by the map.

\begin{proposition}
\label{prop:equivariance}
\label{prop:poisson-transport}
Let $P_z$ be a $\pz$-invariant kernel.
\begin{enumerate}
\item[(i)] $\{X_t\}$ is a Markov chain with kernel
$P_x(x,B)=P_z\big(T^{-1}(x),T^{-1}(B)\big)$, which is $\target$-invariant,
and $\target$-reversible whenever $P_z$ is $\pz$-reversible.
\item[(ii)] If $G_z$ solves $G_z-P_zG_z=F\circ T-\pz(F\circ T)$ with
$\pz(G_z)=0$, then $G=G_z\circ T^{-1}$ solves $G-P_xG=F-\target(F)$ with
$\target(G)=0$.
\end{enumerate}
\end{proposition}

The proof is given in Appendix~\ref{pf:app:transform}. Part~(i) is the
standard fact underlying transport MCMC
\citep{parno2018transport,hoffman2019neutra}, and part~(ii) shows that
Poisson solutions transform in the same way. The two chains are thus the same chain in different coordinates, so their
Poisson equations coincide, and properties such as irreducibility and
geometric ergodicity transfer between them. We state results in $x$-space,
where the target lives, and implement everything in $z$-space, where
$T^{-1}$ is never needed. In practice the control variate is evaluated in
the latent space as $G_z(Z_t)$ and is built from $F\circ T$, so only the
forward map is needed. Changing coordinates alone does not make the Poisson
equation solvable, but the latent space admits a reference: a
$\refm$-reversible kernel $Q$ whose Poisson equation can be solved exactly
and which coincides with the latent chain when the map is exact
(Section~\ref{sec:reference}). We solve the equation for $Q$ and transform
the solution to the target. 

\subsection{Reference Kernels with a Spectral Decomposition}
\label{sec:reference}

The construction of the reference pair $(\refm,Q)$ requires three
properties of the latent sampler.
\begin{condition}
\leavevmode
\label{cond:spectral}
\begin{enumerate}
\item[(C1)] When $w\equiv1$, the kernel $P_z$ is a $\refm$-reversible kernel
$Q$ with $Q\phi_j=\lambda_j\phi_j$ for an orthonormal basis
$\{\phi_j\}_{j\ge0}$ of $L^2(\refm)$, where $\phi_0=1$,
$\bar\lambda:=\sup_{j\ge1}|\lambda_j|<1$.
\item[(C2)] $\E_{q(\cdot|z)}[G]$ can be computed for every $G$ in the span of
$\{\phi_j\}_{j\in\mathcal K}$, for a finite $\mathcal K\subseteq\{j\ge1\}$.
\item[(C3)] When $w\equiv1$, $Z_{t+1}=Z'_t$.
\end{enumerate}
\end{condition}

(C1) makes the Poisson equation solvable at the reference, (C2) makes the
control variates computable, and (C3) ensures that the solution matches the
kernel actually run. For samplers that propose from $Q$ itself, (C2) holds
automatically, because the conditional expectation of each basis function
is $\lambda_j\phi_j$. Transforming $Q$ by Proposition~\ref{prop:equivariance}
gives a kernel $Q_T$ on $\mathcal X$, which is the kernel the sampler runs
when $\target=\hpi_T$.

\begin{theorem}
\label{thm:frozen}
Under (C1), for every bijection $T$,
\begin{enumerate}
\item[(i)] $Q_T$ is $\hpi_T$-reversible, and positive semidefinite on
$L^2(\hpi_T)$ if and only if $\lambda_j\ge0$ for all $j$;
\item[(ii)] the functions $\dict_j=\phi_j\circ T^{-1}$ form an orthonormal
basis of $L^2(\hpi_T)$ with $Q_T\dict_j=\lambda_j\dict_j$; in particular,
the spectrum of $Q_T$ does not depend on $T$, and its absolute spectral gap
is $1-\bar\lambda$;
\item[(iii)] for $F\in L^2_0(\hpi_T)$ with $F\circ T=\sum_{j\ge1}c_j\phi_j$,
the centered solution of $(I-Q_T)\Poi_F=F$ is
\begin{equation}
\Poi_F=\sum_{j\ge1}\frac{c_j}{1-\lambda_j}\,\dict_j,
\qquad \|\Poi_F\|\le\frac{\|F\|}{1-\bar\lambda}.
\label{eq:transported-poisson}
\end{equation}
\end{enumerate}
\end{theorem}

The proof is in Appendix~\ref{pf:app:frozen}. The map decides where the
kernel lives, while its mixing rate and its Poisson solutions are those of
$Q$ for every $T$. For a target $\target\neq\hpi_T$ the chain has kernel
$P_x\neq Q_T$, and two estimators are available for any reference. The
simplest is self-normalized importance sampling (SNIS) from the flow,
which weights draws from $\refm$ by $w$; it is exact when $w\equiv1$, and
otherwise its variance grows with that of the weights. The simplest
latent sampler instead proposes from $Q$ and accepts with probability
$\min\{1,w(z')/w(z)\}$. Since $Q$ is $\refm$-reversible, this is the MH
kernel for $\pz=w\refm$, and when $w\equiv1$ every proposal is accepted,
so it satisfies (C1)--(C3).

For the Gaussian reference $\refm=\gauss(0,I_d)$,
\begin{equation}
\log w(z)=\log\target(T(z))+\log|\det\nabla T(z)|+\tfrac12\|z\|^2
\label{eq:residual}
\end{equation}
up to a constant, so the pull-back is a latent Gaussian model with prior
$\refm$ and log-likelihood $\log w$, and the GI proposals of
Section~\ref{sec:bg-gi} with $(\mu,\Sigma)=(0,I_d)$ apply to it directly;
we call the resulting samplers transformed GI (TGI) samplers. On $\refm$
the three proposals coincide and reduce to
\begin{equation}
z'=\beta z+\sqrt{1-\beta^2}\;\xi,\qquad \xi\sim\gauss(0,I_d),
\label{eq:cn}
\end{equation}
which we denote by $\QGI$. It is $\refm$-reversible, so every proposal is
accepted and (C3) holds, and by \eqref{eq:hermite-eig} it satisfies (C1)
with $\phi_\kappa=\Herm_\kappa$ and $\lambda_\kappa=\beta^{|\kappa|}$. For
$\gamma\in(0,1]$ the transformed kernel is therefore positive semidefinite
with spectral gap $\gamma$, and
$\Poi_F=\sum_{\kappa\neq0}c_\kappa\dict_\kappa/(1-\beta^{|\kappa|})$ with
$\dict_\kappa=\Herm_\kappa\circ T^{-1}$. With an affine map, the
$\dict_\kappa$ are Hermite polynomials in $x$ and the construction reduces
to that of \citet{titsias2026gaussian}. A nonlinear map bends them to
follow the target, so that a low-degree truncation can represent Poisson
solutions that polynomials in $x$ approximate poorly.

TGI-RWM proposes from $\QGI$ for every map and target, so the target
enters only through the ratio $w(z')/w(z)$ and no gradients are needed.
At $\gamma=1$ it is the independence sampler with the flow as proposal:
$\lambda_\kappa=0$ for $\kappa\neq0$ and $\Poi_F=F-\target(F)$, the
solution that \citet{liu2024variance} use when the proposal is close to
the target, and without the accept--reject step it is SNIS. TGI-MALA
adds the drift $\gamma\nabla\log w(z)$, and a preconditioned variant is
given in Appendix~\ref{app:ref-gaussian}; for both the candidate is Gaussian
and the conditional expectations required by (C2) are in closed form
(Lemma~\ref{lem:condmean}). Transport elliptical slice sampling (TESS;
\citealp{cabezas2023transport}) runs elliptical slice sampling with prior
$\refm$ and likelihood $w$; it has no accept--reject step, $w$ enters
through the slice, and it satisfies (C1)--(C3) with the same
eigenfunctions and the eigenvalues $\E[\cos^{|\kappa|}\vartheta]$,
$\vartheta\sim\mathrm{U}(0,2\pi)$ (Appendix~\ref{app:ref-ess}). A
mismatch between $\hpi_T$ and $\target$ affects SNIS through the
weights, TGI-RWM only through rejections, TESS through the shrinking of
the slice, and TGI-MALA also through its drift.

\subsection{Control Variates and Estimator}
\label{sec:estimator}

With $T$ frozen, each step of a latent MH sampler provides the state $Z_t$,
the candidate $Z'_t$ and the acceptance probability $\alpha_t$. The
estimator \eqref{eq:twoterm} is therefore applied in the latent space, with
$F\circ T$ in place of $F$ and $G=\sum_{j\in\mathcal K}\eta_j\phi_j$. We call
$\{\phi_j\}_{j\in\mathcal K}$ the \emph{dictionary} of the control variate.
For the Gaussian reference, the dictionary consists of Hermite polynomials
$\Herm_\kappa$ indexed by a finite set $\mathcal K$ of multi-indices, by
default $\{\kappa:1\le|\kappa|\le K\}$, and $G$ costs a few Hermite
evaluations per step. TESS has no accept--reject step and uses
$G(Z_{t+1})-G(Z_t)$ in place of the first correction, with the initial
point on the ellipse as $Z'_t$ (Appendix~\ref{app:ref-ess}). For SNIS, which is a ratio of two averages, the same dictionary corrects
both,
\begin{equation}
\hat\mu^{\mathrm{IS}}=\frac{\sum_i w_iF(T(\zeta_i))-\sum_i\sum_{j\in\mathcal K}a_j\phi_j(\zeta_i)}
{\sum_i w_i-\sum_i\sum_{j\in\mathcal K}b_j\phi_j(\zeta_i)},\qquad
\zeta_i\sim\refm,
\label{eq:snis-cv}
\end{equation}
with $a_j$ and $b_j$ the coefficients of $w\,(F\circ T)$ and of $w$ on
$\phi_j$, estimated on the same draws. Both corrections have mean zero
under $\refm$, so the ratio remains consistent for every $a$ and $b$;
when $w\equiv1$ the $b_j$ vanish and the $a_j$ reduce to $c_j$. 

The plug-in choice $\eta_j=c_j/(1-\lambda_j)$, with
$c_j=\E_\refm[(F\circ T)\,\phi_j]$, makes $G$ the truncation of
\eqref{eq:transported-poisson} in latent coordinates and requires
$F\circ T\in L^2(\refm)$. The coefficients are
Monte Carlo averages over draws from $\refm$ that require evaluations of
$T$ but not of $\target$, so they are computed once before the run.
Alternatively, $\eta$ can be fitted by minimizing the empirical variance of
the summand on burn-in samples and then frozen, which adapts the control
variate to the actual target at the price of a short pilot run.

\begin{assumption}
\label{ass:main}
\leavevmode
\begin{enumerate}
\item[(A1)] $T:\reals^d\to\mathcal{X}$ is a $C^1$ diffeomorphism.
\item[(A2)] $m\le w(z)\le M$ for all $z$ and some $0<m\le M<\infty$.
\item[(A2$^\ast$)] $w$ is continuous and there are $\alpha,R>0$ such that
$w(y)\ge w(x)$ whenever $\|x\|>R$ and $\|y\|\le\alpha\|x\|$.
\item[(A3)] $F\circ T\in L^{2+\epsilon}(\pz)$ and $\phi_j\in L^{2+\epsilon}(\pz)$
for $j\in\mathcal K$, for some $\epsilon>0$.
\end{enumerate}
\end{assumption}

(A2) requires the map to match the tails of $\target$ up to constant
factors, which a coupling flow cannot do unless $\target$ itself has
Gaussian tails (Section~\ref{sec:algorithm}), and (A2$^\ast$), used for
TESS \citep{natarovskii2021geometric}, that $w$ eventually decreases
along rays. We therefore also allow the latent chains to replace each
step, with probability $\delta\in[0,1)$, by an independence MH step with
proposal $t_\nu(0,I_d)$, as in \citet{brofos2022adaptation} but mixing
kernels rather than distributions; the estimator averages the summands
of the other steps only. With $\delta>0$, the bounds on $w$ can be
replaced by a condition on the tails of $\target$ alone:
\begin{assumption}
(A2$'$) $T$ is a coupling flow with bounded log-scales and shifts,
and $\target$ is bounded with $\target(x)\le C(1+\|x\|)^{-(d+\nu)}$,
with $\nu > K(2+\epsilon)$.
\end{assumption}

\begin{theorem}\label{thm:validity}
Assume \textup{(A1)} and \textup{(A3)}, with $\eta$ fixed before the run
for the chains. Then $\hat\mu_n\to\target(F)$ almost surely and
$\sqrt n\{\hat\mu_n-\target(F)\}\to\gauss(0,\sigma^2_\infty)$ with
$\sigma^2_\infty<\infty$ for
\begin{enumerate}
\item SNIS, if $\E_\refm[w^2\{1+(F\circ T)^2\}]<\infty$;
\item TGI-RWM, under \textup{(A2)} with $\delta=0$, or \textup{(A2$'$)}
with $\delta\in(0,1)$;
\item TESS, under \textup{(A2)} and \textup{(A2$^\ast$)} with $\delta=0$,
or \textup{(A2$'$)} with $\delta\in(0,1)$.
\end{enumerate}
\end{theorem}

The chains are geometrically ergodic in the first case and uniformly ergodic in the second. The proofs are given in Appendix~\ref{pf:app:validity}, the same appendix treats TGI-MALA and TGI-pMALA. SNIS needs the weights to have a finite second moment, which a flow
lighter-tailed than the target violates. For the two chains, the independence steps restore
uniform ergodicity in this case. The tail penalty of
Section~\ref{sec:algorithm} is the practical substitute for (A2).

\subsection{Learning the Map}
\label{sec:algorithm}

We use existing methods to obtain $T$, with one addition. An exact map
exists when $\target$ has a positive density on $\mathbb R^d$, given by the
triangular Knothe--Rosenblatt rearrangement
\citep{rosenblatt1952remarks,knothe1957contributions}, which triangular
transport maps \citep{parno2018transport,spantini2018couplings} can
represent directly, and normalizing flows
\citep{dinh2017density,huang2018neural,jaini2019sum,papamakarios2021normalizing}
are flexible enough to approximate such targets; affine coupling flows,
for instance, are universal distributional approximators
\citep{teshima2020coupling}. Constrained parameters are first mapped to
$\mathbb R^d$ by a standard reparametrization, and whitening the latent space
with a pilot estimate of the covariance provides a good initialization.

The choice of objective depends on the target. The reverse divergence
$\mathrm{KL}(\hpi_T\,\|\,\target)$ needs only evaluations of $\target$ and
its gradient \citep{rezende2015variational,hoffman2019neutra} and suits
unimodal posteriors. Maximum likelihood on draws produced by the sampler itself, alternating
fitting and sampling as in adaptive flow-based samplers
\citep{parno2018transport,gabrie2022adaptive,brofos2022adaptation,
samsonov2022local,cabezas2024markovian}, suits hierarchical models, where
it can be added to the reverse divergence during training. On multimodal
targets, anchoring part of each minibatch to pilot draws keeps every mode
represented in the map, although samplers run in the latent space remain
less robust to multimodality than flow-based independence proposals
\citep{grenioux2023sampling}. Since continual adaptation
requires the conditions of adaptive MCMC
\citep{andrieu2006ergodicity,roberts2007coupling,andrieu2008tutorial}, we
train the map, freeze it, and then run the evaluation chains, so the
results of Section~\ref{sec:estimator} apply. 

A coupling flow with bounded
scales is asymptotically affine and has Gaussian tails \citep{jaini2020tails}, so for a target
with heavier tails $w$ grows along a few latent directions. Training objectives rarely see these
regions, thus a chain that enters one tends to stay, which dominates the
variance of the control variates. Existing remedies mix the flow with a fixed distribution, which bounds the
weight by construction \citep{brofos2022adaptation}, use mass-covering
objectives that favor weights with finite moments
\citep{li2016renyi,dieng2017variational,wang2018variational}, or add tail
layers \citep{jaini2020tails,hickling2025flexible}, which are effective for
power-law tails along coordinate directions. Instead, we add to the loss the one-sided penalty
\[
\E\big[\max\{0,\log w(\zeta)-\bar\ell-c\}^2\big],
\]
evaluated at points sampled directly in the tails, where $\zeta$ has a uniform
direction and a radius uniform on $[\sqrt d+1,\sqrt d+5]$, from the edge of
the Gaussian shell into the far tail, and $\bar\ell$ is the mean of
$\log w$ over the training batch. It acts only on weights that are too
large, which make chains stick, and vanishes at the exact map. This method
is a soft counterpart of the upper bound in (A2). The penalty is needed only when the flow is lighter-tailed than the target, and Appendix~\ref{app:training-results} gives a rule for when to use it.

\section{PRACTICAL CONSIDERATIONS}
\label{sec:practical}

\subsection{Choosing the Estimator}
\label{sec:diagnostics}

A mismatch between $\hpi_T$ and $\target$ affects each estimator
differently. SNIS suffers through the weights. Where the flow puts too
little mass, the rare draws that land there carry large weights and
dominate the error. TGI-RWM suffers through rejections. With $\gamma$
near one, a rejection means staying at a state of large weight, often
for many steps. TESS moves at every iteration, but each iteration costs
several evaluations. TGI-MALA adds a gradient per step, and its drift
corrects the mean of the pull-back; in $x$-space the map acts as the
position-dependent preconditioner $\nabla T\,\nabla T^\top$
(Table~\ref{tab:variants} in the Appendix). HMC on
the same flow \citep{hoffman2019neutra} is the stronger gradient-based
alternative, but our control variates do not apply to it.

Three diagnostics computed from draws of the flow decide between them.
The importance sampling effective sample size
$\mathrm{ESS}_{IS}=(\sum_i w(\zeta_i))^2/\sum_i w(\zeta_i)^2$ measures
the bulk of the transport. It is blind to the tails and to a missing mode, so we also record the
largest normalized weight relative to its uniform value $1/N$,
$\max w\cdot N=N\max_i w(\zeta_i)/\sum_i w(\zeta_i)$. When $\max w\cdot N$
stays below a threshold we use SNIS. Beyond
that a latent chain is needed. We use TESS if a pilot chain shows a seam
between modes or heavy tails, a gradient-based chain if the mismatch is
spread over many coordinates of a unimodal pull-back, and TGI-RWM as the
gradient-free default. The pilot chain also gives the Fisher divergence
$\mathcal F=\E_{\pz}\|\nabla\log w\|^2$, and a large $\mathcal F$ calls
for a better map rather than a different estimator.

\subsection{High Dimensions: Dictionary Growth and Pruning}
\label{sec:pruning}

The full dictionary $\{\Herm_\kappa:1\le|\kappa|\le K\}$ has
$\binom{d+K}{K}-1\sim d^K/K!$ elements, most of them cross terms: indices
involving a single coordinate number only $dK$, whereas indices involving
$k$ coordinates number $O(d^k)$. A large index set is costly in two ways.
Each evaluation of $G$ costs $O(|\mathcal K|)$, and the coefficients are
estimated from a shared budget of $N$ reference draws, so the aggregate
estimation noise grows with $|\mathcal K|$. Although the control variate has
mean zero for any coefficients, this noise can increase the variance, so an
oversized dictionary degrades the estimator silently. Conversely, a good
transport compresses the dictionary: if $T$ is affine and $F$ is linear,
$F\circ T$ is exactly of first order. The following result quantifies the
trade-off on an exact map.

\begin{proposition}
\label{prop:pruning}
Assume (C1)-(C3) and (A3), and suppose that $w\equiv1$. Let $G=\sum_{j\in\mathcal K}\eta_j\phi_j$ with arbitrary
coefficients, and write $v_j=(1+\lambda_j)/(1-\lambda_j)$. The asymptotic variance of the estimator is
\begin{equation}
\sigma^2_\infty
=\sum_{j\ge1,\,j\notin\mathcal K}v_jc_j^2
+\sum_{j\in\mathcal K}(1-\lambda_j^2)\Big(\eta_j-\frac{c_j}{1-\lambda_j}\Big)^2 .
\label{eq:separable}
\end{equation}
If $\eta_j=\hat c_j/(1-\lambda_j)$, where $\hat c_j$ is an unbiased estimate
of $c_j$ from $N$ reference draws with variance $s_j^2/N$, then
\begin{equation}
\E\big[\sigma^2_\infty\big]
=\sum_{j\ge1,\,j\notin\mathcal K}v_jc_j^2+\sum_{j\in\mathcal K}v_j\frac{s_j^2}{N}.
\label{eq:expected}
\end{equation}
\end{proposition}

In practice we combine two kinds of pruning. Structural pruning keeps the
indices of total degree at most $K$ that involve at most $p$ coordinates, a
set $\mathcal K_{K,p}$ of size $O(d^p)$ that can be enumerated directly.
Magnitude pruning then estimates $\hat c_\kappa$ and $\hat s_\kappa^2$ from
$N_0$ pilot reference draws, keeps the indices with
$\hat c_\kappa^{\,2}>\tau_{prune}\,\hat s_\kappa^2/N_0$, and re-estimates their
coefficients with the full budget. Since
$\E\hat c_\kappa^{\,2}=c_\kappa^2+s_\kappa^2/N_0$, the threshold $\tau_{prune}=2$
recovers the criterion above on average, while $\tau_{prune}=2\log|\mathcal K|$ also
guards against retaining pure noise when $\mathcal K$ is large. Using
reference draws rather than the chain keeps $G$ independent of the samples
it corrects. Since $\gamma$ enters $\eta_\kappa=c_\kappa/(1-\beta^{|\kappa|})$
only through $1-\beta^{|\kappa|}$, the coefficients are estimated once and
reused for every step size. When $w\not\equiv1$, rejections add a further
term to the variance, analyzed in Section~\ref{sec:gamma}.

\subsection{Choosing the Step Size: Mixing versus Estimation}
\label{sec:gamma}

When $\pz=\refm$, every proposal is accepted and $\gamma=1$ gives
i.i.d.\ sampling, which without the accept--reject step is SNIS. When
$\pz\neq\refm$, the efficiency of the estimator is measured by
$\mathrm{ESS}\times\mathrm{VRF}$, where $\mathrm{ESS}$ is the effective
sample size per step of the plain average and $\mathrm{VRF}$ the variance
reduction factor of the control variates, so that asymptotically
$\Var(\hat\mu_n)=\Var_\target(F)/(n\,\mathrm{ESS}\times\mathrm{VRF})$.
The first factor favors a large step, the second a high acceptance rate
and a small step. Acceptance targeting, expected squared jump distance
\citep{pasarica2010adaptively} and the speed measure of
\citet{titsias2019gradient} target the ESS alone, so they tend to select
a step that is too large for the estimator.

We therefore adapt $\gamma$ by maximizing an approximation of
$\mathrm{ESS}\times\mathrm{VRF}$. For the ESS we use the speed measure
$(2\gamma-\gamma^2)a(\gamma)$ of \citet{titsias2019gradient}, where
$a(\gamma)$ is the mean acceptance rate. Without rejections the control
variate would reduce the variance to the residual
$R$ that the dictionary leaves unexplained, reach a VRF of
$V_0=\Var_\target(F)/R$; taking the contribution of rejections to be
$(1-a(\gamma))\Var_\target(F)$ gives
$\mathrm{VRF}(\gamma)\approx V_0/\{1+V_0(1-a(\gamma))\}$, and the
objective
\begin{equation}
J(\gamma)=\log(2\gamma-\gamma^2)+\log a(\gamma)
-\log\big(1+V_0\,(1-a(\gamma))\big).
\label{eq:gammaobj}
\end{equation}
When $\pz=\refm$, $a\equiv1$ and $J$ is maximized at $\gamma=1$. For
small $V_0$ it reduces to the speed measure, and a large $V_0$ lowers
$\gamma$ until the rejection rate is of order $1/V_0$.
$V_0$ comes at no extra cost, since $R$ is the variance of
$F\circ T-\sum_{\kappa\in\widehat{\mathcal K}}\hat c_\kappa\Herm_\kappa$
on fresh reference draws. During burn-in, $\gamma$ follows stochastic
gradient ascent on $J$ (Appendix~\ref{app:alg-stepsize}). Since the
rejection loss is only modeled, $J$ is a heuristic. $J$ needs no run beyond burn-in and is what we recommend in practice;
the step sizes used in each experiment are given in
Appendix~\ref{app:experiments}. The step size, together
with $T$, $\widehat{\mathcal K}$ and $\hat\eta$, is fixed before the
evaluation chains are run. 

\section{EXPERIMENTS}
\label{sec:experiments}

\begin{figure}
\centering
\includegraphics[width=\columnwidth]{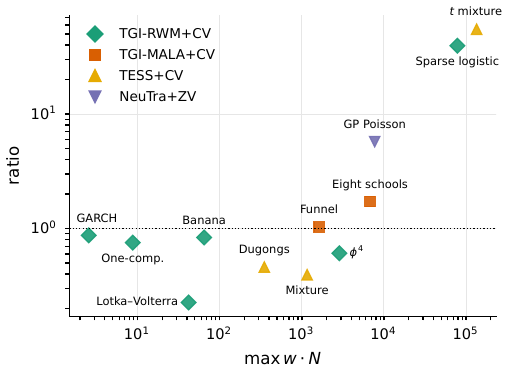}
\caption{When to sample instead of reweighting. Each point is one target
of Table~\ref{tab:wallclock}. The ratio is the precision per second of
the best chain-based estimator on the main test function divided by
that of SNIS+CV, and the marker shows which chain is best; above the
dotted line a chain is more efficient than SNIS+CV. The horizontal axis
is the largest normalized importance weight of the flow times $N$,
over the $100$ replicates of $N=10^6$ draws.}
\label{fig:maxw}
\end{figure}

\begin{table*}[t]
\centering
\small
\setlength{\tabcolsep}{5pt}
\caption{Precision per second on the main test function of each target
(Tables~\ref{tab:exp-synthetic} and~\ref{tab:exp-real}), with the
protocol of Appendix~\ref{app:experiments}. In parentheses, the variance
reduction factor of the control variates, the mean squared error without
them over that with them on the same run. Targets are ordered by the
largest normalized importance weight of the flow times $N$. Entries
within $10\%$ of the best in a row are in bold. $\dagger$: the Monte
Carlo error of the reference value accounts for at least half of the
mean squared error; --: RMSE above one posterior standard
deviation.}
\label{tab:wallclock}
\begin{tabular}{@{}l rrrrr@{}}
\toprule
Target & SNIS & TGI-RWM & TGI-MALA & TESS & NeuTra \\
\midrule
GARCH & \textbf{9.7e7$^\dagger$} (29) & 8.4e7$^\dagger$ (37) & 2.2e7 (64) & 2.4e7 (10) & 5.0e7 (109) \\
One-comp. & \textbf{8.9e5} (1.6) & 6.7e5 (1.7) & 4.6e5 (12) & 4.7e5 (1.6) & 5.8e5 (5.9) \\
Lotka--Volterra & \textbf{1.0e8} (28) & 2.3e7 (14) & 5.6e6 (57) & 8.9e6 (6.4) & 2.1e7 (176) \\
Banana & \textbf{4.1e7} (26) & 3.4e7 (35) & 1.1e7 (117) & 5.9e6 (13) & 1.7e7 (104) \\
Dugongs & \textbf{2.0e7} (5.1) & 5.4e6 (3.5) & 5.1e4 (0.6) & 9.2e6 (4.9) & 1.3e6 (3.3) \\
Mixture & \textbf{3.9e6} (3.3) & 1.5e6 (2.3) & -- & 1.6e6 (2.4) & 1.6e4 (1.9) \\
Funnel & \textbf{1.4e6} (1.9) & 6.6e5 (1.7) & \textbf{1.4e6} (48) & 1.1e6 (2.0) & 8.5e5 (6.0) \\
$\phi^4$ field & \textbf{1.4e6} (1.9) & 8.4e5 (1.9) & -- & 8.0e5 (1.8) & 1.0e4 (2.5) \\
Eight schools & 6.2e5 (1.4) & 6.4e5 (1.7) & \textbf{1.1e6} (16) & 8.7e5 (1.9) & 3.4e5 (4.4) \\
GP Poisson & 4.2e5 (1.3) & 1.0e6 (2.2) & 1.1e6 (21) & 1.1e6 (2.1) & \textbf{2.4e6} (22) \\
Sparse logistic & 1.1e4 (1.0) & \textbf{4.5e5} (3.7) & 2.1e5 (13) & 2.2e5 (2.9) & \textbf{4.2e5} (30) \\
$t$ mixture & 4.1e3 (1.0) & 9.8e4 (1.1) & -- & \textbf{2.2e5} (1.2) & 2.2e3 (3.6) \\
\bottomrule
\end{tabular}
\end{table*}

We evaluate the method of Algorithms~\ref{alg:pipeline}
and~\ref{alg:tgi} (Appendix~\ref{app:algorithms}) by its precision per
second, $1/(\mathrm{RMSE}^2\times\text{seconds per estimate})$, with the
root mean squared error over $100$ replications measured in posterior
standard deviations. We run every estimator on twelve targets. The
\emph{synthetic} targets are the banana of \citet{haario1999adaptive},
the funnel of \citet{neal2003slice}, the Gaussian mixture of
\citet{alexopoulos2023variance} and a $t$ mixture. The \emph{real}
posteriors are GARCH(1,1), the dugongs growth curve, a pharmacokinetic
ODE model and a Gaussian-process Poisson regression from posteriordb
\citep{magnusson2024posteriordb}, the Lotka--Volterra model, the centered
eight schools model, sparse logistic regression on the German credit
data \citep{hoffman2019neutra} and the bimodal $\phi^4$ field of
\citet{gabrie2022adaptive}. For each target one RealNVP flow is trained
and frozen. On it we run SNIS, TGI-RWM, TGI-MALA and TESS with our
control variates, and HMC in the latent space
\citep[NeuTra;][]{hoffman2019neutra} with zero-variance control
variates, each with $10^6$ target evaluations per estimate. The settings
and the results for every test function are in
Appendix~\ref{app:experiments}; Table~\ref{tab:wallclock} reports, for each target, the main test
function, chosen as the one whose performance under TGI-RWM
with the control variates is the median over the test functions of
that target, the better of the two middle ones when their number is even.

Where the flow is accurate, the control variates reduce the variance of
SNIS and TGI-RWM by one to two orders of magnitude. Where the flow leaves a neck or a seam between
modes, the gain falls to a factor of a few, since the test functions
then pull back to functions that low-degree Hermite polynomials
represent poorly, and the advantage of the flow-based estimators comes
from the transport itself.

The rule of Section~\ref{sec:diagnostics} holds up across all twelve
targets (Figure~\ref{fig:maxw}). While $\max w\cdot N$ stays below a
few thousand, reweighting the flow with the control variates is very effective, even on the two multimodal targets, because
the flow has already done the work of finding both modes. Beyond that
threshold the weights stop being informative and a latent chain takes
over. What carries through both regimes is the control variate. One dictionary serves SNIS, TGI-RWM and TESS alike without a single gradient of the target, and wherever the flow is accurate it makes each of them markedly more precise. Table~\ref{tab:wallclock} uses $\delta=0$. With the independence step of Theorem~\ref{thm:validity}, $\delta=0.01$, the precision of TGI-RWM+CV changes by a factor between $0.55$ and $3.1$ across all test functions (Table~\ref{tab:wallclock-full}).

The methods also differ in how far they can fall behind
(Table~\ref{tab:stability}). SNIS with the control variates is the best
method most often but has the most uneven record, falling too much behind where its weights are heaviest. TGI-RWM is rarely the
best, yet it stays within a small factor of the best on every target,
the smallest worst case of all the methods, which makes it the safe
default when the diagnostics are inconclusive.
Codes
to reproduce this experiment can be found at \url{https://github.com/cliusir123/TMCVR_Code}
\section{CONCLUSION}
\label{sec:conclusion}

A bijection carries a Markov kernel and its Poisson solutions together,
so the exact solutions that a Gaussian reference offers become control
variates on every target a flow can reach. Because the construction sees
the kernel only through its spectrum, one dictionary of Hermite polynomials, whose coefficients need no
further evaluation of the target, serves a whole family of estimators that need no gradient of the
target: IS from the flow, the latent random walk and
elliptical slice sampling. Each is consistent and asymptotically normal
under its own conditions, and where the flow is accurate the control
variates make each of them much more precise.
The largest importance weight of the flow tells which member to run. The control variates also prefer the weights of the flow to stay
bounded in the tails, which the usual training objectives do not
enforce and a penalty in training encourages.

The gains depend on the quality of the transport. Training dominates the cost
at small budgets, the tail penalty carries no guarantee, and where the
flow leaves a neck or a seam between modes the control variates gain
little; on posteriors whose mismatch is spread over many coordinates,
HMC on the same flow remains the stronger choice,
and targets with tails as heavy as our $t$ mixture lie beyond the
theory. Conversely, any improvement in training flows benefits the estimators
directly. We see several directions for future work. Flows whose
weights are bounded by construction would meet the assumptions of the
theory directly. The reference need not be Gaussian, since Gamma and
spherical references (Appendix~\ref{app:general}) could be combined into
one that already matches the support and tails of the target. Hamiltonian
dynamics with an exact Gaussian integrator \citep{beskos2011hybrid} keeps
the Hermite polynomials as eigenfunctions and could bring the
gradient-based chain into the family, and a dictionary with functions
that separate modes would extend the gains to multimodal targets.
Finally, the flow could be trained alongside the chains, under the
conditions of adaptive MCMC.

\bibliographystyle{plainnat}
\bibliography{reference}

\clearpage
\appendix
\onecolumn
\setcounter{section}{0}
\setcounter{equation}{0}
\setcounter{figure}{0}
\setcounter{table}{0}
\setcounter{algorithm}{0}
\renewcommand{\thesection}{S\arabic{section}}
\renewcommand{\theequation}{S\arabic{equation}}
\renewcommand{\thefigure}{S\arabic{figure}}
\renewcommand{\thetable}{S\arabic{table}}
\renewcommand{\thealgorithm}{S\arabic{algorithm}}
\begin{center}
{\LARGE Supplementary Materials: Transformed Samplers with Variance Reduction}
\end{center}
\addcontentsline{toc}{section}{Supplementary Material}

\section{PROOFS}
\label{app:proofs}

\subsection{Notation}
\label{app:notation}

We recall the notation of Section~\ref{sec:method}. Throughout,
$(\mathcal Z,\mathcal B_{\mathcal Z})$ and $(\mathcal X,\mathcal B_{\mathcal X})$
are measurable spaces, equipped with their Borel $\sigma$-algebras when they
are subsets of $\reals^d$. The reference $\refm$ is a probability measure on
$\mathcal Z$ and $T:\mathcal Z\to\mathcal X$ is a bijection with $T$ and
$T^{-1}$ measurable. The map pushes forward the
reference to $\hpi_T=T_\pf\refm$ and pulls back the target to
$\pz=(T^{-1})_\pf\target$. Assuming $\pz\ll\refm$, we define the importance
weight $w=d\pz/d\refm$. We denote by $\op U$ the composition
operator associated with $T$, which maps a measurable $g:\mathcal X\to\mathbb R$
to $\op Ug=g\circ T:\mathcal Z\to\mathbb R$. Since $T$ is a 
bijection, $\op U$ is invertible with $\op U^{-1}h=h\circ T^{-1}$ for
measurable $h:\mathcal Z\to\mathbb R$. By change of variables, $\op U$ is an
isometry from $L^p(\hpi_T)$ onto $L^p(\refm)$ for every $p\in[1,\infty]$, and
$\mathbb E_{\target}[g]=\mathbb E_{\pz}[\op Ug]=\mathbb E_{\refm}[w\,\op Ug]$. The latent chain $\{Z_t\}$ has kernel $P_z$,
it proposes $Z'_t\sim q(\cdot|Z_t)$ at step $t$ and moves to
$Z_{t+1}$, we also have $X_t=T(Z_t)$. For the solution
$G=\sum_{j\in\mathcal K}\eta_j\phi_j$ and coefficients
$\theta=(\theta_1,\theta_2)$, the estimator of Section~\ref{sec:estimator}
is $\hat\mu_n=n^{-1}\sum_{t=0}^{n-1}S_t$ with
\begin{equation}
S_t=F(T(Z_t))+\theta_1H_{1,t}+\theta_2H_{2,t},
\label{eq:app-summand}
\end{equation}
where $H_{1,t}=\alpha_t\{G(Z'_t)-G(Z_t)\}$, $H_{2,t}=G(Z'_t)-\E_{q(\cdot|Z_t)}[G]$ for an MH
sampler with acceptance probability $\alpha_t=\alpha(Z_t,Z'_t)$, and
$H_{1,t}=G(Z_{t+1})-G(Z_t)$ otherwise. For a probability measure $\mu$, $L^2(\mu)$ denotes
the square-integrable functions with inner product
$\langle\cdot,\cdot\rangle_\mu$, and $L^2_0(\mu)$ those with $\mu(h)=0$. For a kernel $P$ reversible with respect to $\mu$, the
Dirichlet form and the spectral gap are
\[
\mathcal E_P(f)=\tfrac12\iint\big(f(y)-f(x)\big)^2\mu(dx)P(x,dy),\qquad
\mathrm{Gap}(P)=\inf_{\Var_\mu(f)>0}\frac{\mathcal E_P(f)}{\Var_\mu(f)},
\]
and the absolute spectral gap is
$\mathrm{Gap}_{\mathrm{abs}}(P)=1-\|P\|_{L^2_0(\mu)}$, where
$\|P\|_{L^2_0(\mu)}=\sup\{|\langle h,Ph\rangle_\mu|:h\in L^2_0(\mu),\,\|h\|_\mu=1\}$
is the operator norm of $P$ on $L^2_0(\mu)$.
For the Gaussian reference, we use the probabilists' Hermite polynomials
$\he_n$, defined by the generating function
\begin{equation}
\exp(tx-t^2/2)=\sum_{n\ge0}\he_n(x)\,\frac{t^n}{n!},\qquad t,x\in\reals,
\label{eq:app-hermite-gen}
\end{equation}
and their normalized tensor products
\begin{equation}
\label{eq:app-hermite-basis}
\Herm_\kappa(z)=\prod_{j}\he_{\kappa_j}(z_j)/\sqrt{\kappa_j!}, \qquad \kappa\in\mathbb N^d,
\end{equation}
form an orthonormal basis of
$L^2(\gauss(0,I_d))$ \citep[Ch.~1]{nualart2006malliavin}. If
$z'=\beta z+s\xi$ with $\xi\sim\gauss(0,I_d)$, $\beta\in[-1,1]$ and
$s^2=1-\beta^2$, then
\begin{equation}
\E[\Herm_\kappa(z')\,|\,z]=\beta^{|\kappa|}\Herm_\kappa(z),
\label{eq:app-hermite-cond}
\end{equation}
with the convention $0^0=1$. To see this, by the independence of the coordinates it suffices
to take $d=1$, then replace $x$ by $\beta x+s\xi$ in
\eqref{eq:app-hermite-gen}, taking expectations and using
$\E[e^{ts\xi}]=e^{t^2s^2/2}$ gives
$\sum_{n}\E[\he_n(\beta x+s\xi)]\,t^n/n!=\exp\{(t\beta)x-(t\beta)^2/2\}
=\sum_n\he_n(x)(t\beta)^n/n!$, where the exchange of expectation and
summation is justified by the dominated convergence with $\sum_n|\he_n(x')|\,|t|^n/n!\le e^{|t||x'|+t^2/2}$.
Identifying coefficients yields \eqref{eq:app-hermite-cond}.

\subsection{Proof of Proposition~\ref{prop:equivariance}}
\label{pf:app:transform}

The proof uses only that $T$ and $T^{-1}$ are measurable, so it holds for
general measurable spaces $\mathcal Z$ and $\mathcal X$.

(i) Since $T$ is measurable, $T^{-1}(B)$ is measurable for every measurable
$B\subseteq\mathcal X$, so $P_x(x,\cdot)$ is well defined and is the
pushforward of $P_z(T^{-1}(x),\cdot)$ under $T$, thus it is a probability
measure. Since $T^{-1}$ is measurable, $x\mapsto P_z(T^{-1}(x),T^{-1}(B))$
is measurable. Thus $P_x$ is a Markov kernel. For $h\in L^1(\target)$, since $P_x(T(z),\cdot)$ is the pushforward of
$P_z(z,\cdot)$ under $T$,
\[
(P_xh)(T(z))=\int h(T(z'))\,P_z(z,dz')=(P_z\op Uh)(z)
\quad\text{for $\pz$-a.e.},
\]
that is, $\op UP_x=P_z\op U$ on $L^1(\target)$. Invariance follows from
\[
\target(P_xh)=\pz(\op UP_xh)=\pz(P_z\op Uh)=\pz(\op Uh)=\target(h).
\]
The measure $\target(dx)P_x(x,dy)$ is the image of $\pz(dz)P_z(z,dz')$
under $T\times T$. Writing $\sigma(a,b)=(b,a)$, reversibility means
invariance under $\sigma$, and since $(T\times T)\circ\sigma=\sigma\circ(T\times T)$,
reversibility of $P_z$ implies that of $P_x$. Since $T$ and $T^{-1}$ are measurable, $(X_0,\dots,X_t)$ and
$(Z_0,\dots,Z_t)$ generate the same $\sigma$-algebra, so
\[
\mathbb P(X_{t+1}\in B\,|\,X_0,\dots,X_t)=P_z(Z_t,T^{-1}(B))=P_x(X_t,B),
\]
and $\{X_t\}$ is a Markov chain with kernel $P_x$.

(ii) Since $\target=T_\pf\pz$, $\op U$ is an isometry from $L^1(\target)$ onto
$L^1(\pz)$, so $G=\op U^{-1}G_z\in L^1(\target)$, and by change of variables
$\pz(F\circ T)=\target(F)$ and $\target(G)=\pz(G_z)=0$. Applying
$\op U^{-1}$ to both sides of $\op UP_x=P_z\op U$ gives
$P_x\op U^{-1}=\op U^{-1}P_z$, therefore
\[
(I-P_x)G=\op U^{-1}(I-P_z)G_z=\op U^{-1}\{F\circ T-\target(F)\}=F-\target(F).
\]\qed

\begin{remark}
\label{rem:app-mh}
If $P_z$ is the MH kernel with target $\pz$ and proposal $q$, then $P_x$ is the
MH kernel with target $\target$ and proposal
$q_x(x,\cdot)=T_\pf q(T^{-1}(x),\cdot)$, and the acceptance probability of a
move $T(z)\to T(z')$ equals that of $z\to z'$. With densities and a
diffeomorphism $T$,
\[
q_x(x'|x)=q\big(T^{-1}x'\,|\,T^{-1}x\big)\,|\det\nabla T^{-1}(x')|,\qquad
\target(x')=\pz(T^{-1}x')\,|\det\nabla T^{-1}(x')|,
\]
and in the ratio $\target(x')q_x(x|x')/\{\target(x)q_x(x'|x)\}$ the Jacobian
factors are canceled out. Without densities the same holds in the general
formulation of \citet{tierney1998note}, because the push-forward by
$T\times T$ preserves the Radon--Nikodym derivative that defines the
acceptance probability.
\end{remark}

\begin{remark}[Affine maps]
\label{rem:affine}
Two consequences of Proposition~\ref{prop:equivariance}(ii) connect it
with earlier work. Let $T(z)=\mu+Lz$ with $L$ invertible and
$\Sigma=LL^\top$, so that $\hpi_T=\gauss(\mu,\Sigma)$ and
$\nabla\log\pz(z)=L^\top\nabla\log\target(T(z))$.
\begin{enumerate}
\item Since $z'=\beta z+s\xi$ gives $x'=\beta x+(1-\beta)\mu+sL\xi$, the
transformed GI kernel is
$Q_T(x,\cdot)=\gauss\big(\beta x+(1-\beta)\mu,(1-\beta^2)\Sigma\big)$,
and its Poisson solutions are those of \citet{titsias2026gaussian}:
$\Poi_F(x)=(x-\mu)/\gamma$ for $F(x)=x$, and
$\Poi_F(x)=\{(x-\mu)(x-\mu)^\top-\Sigma\}/(2\gamma-\gamma^2)$ entrywise
for $F(x)=(x-\mu)(x-\mu)^\top$, as $\E[x'-\mu\,|\,x]=\beta(x-\mu)$,
$\E[(x'-\mu)(x'-\mu)^\top\,|\,x]=\beta^2(x-\mu)(x-\mu)^\top+(1-\beta^2)\Sigma$
and $1-\beta^2=2\gamma-\gamma^2$.
\item For any target and $c>0$, the MH kernels with proposals
$\gauss(x,c^2\Sigma)$ and
$\gauss\big(x+\tfrac{c^2}{2}\Sigma\nabla\log\target(x),c^2\Sigma\big)$
are the transforms of those with proposals $\gauss(z,c^2I)$ and
$\gauss\big(z+\tfrac{c^2}{2}\nabla\log\pz(z),c^2I\big)$, by
Remark~\ref{rem:app-mh}. If $\target=\gauss(\mu,\Sigma)$ and
$\hat F^{(i)}_0$ is a centered Poisson solution of $z\mapsto z_i$ for the
latent kernel, then by linearity
$G(x)=\sum_iL_{ji}\hat F^{(i)}_0(L^{-1}(x-\mu))$ is one for $F(x)=x_j$,
which for lower-triangular $L$ and $j=1$ is Proposition~3 of
\citet{alexopoulos2023variance}. An approximate solution with residual
$e$, $(I-P_z)G_0=F_0-\pz(F_0)+e$, transforms to one with residual
$e\circ T^{-1}$ and $\|e\circ T^{-1}\|_{L^p(\target)}=\|e\|_{L^p(\pz)}$,
which is why an approximation fitted once on the standard normal, as in
\citet{alexopoulos2023variance}, has the same quality for every
Gaussian.
\end{enumerate}
\end{remark}

\subsection{Proof of Theorem~\ref{thm:frozen}}
\label{pf:app:frozen}

By Proposition~\ref{prop:equivariance}(i) with $P_z=Q$ and $\pz=\refm$,
$Q_T=\op U^{-1}Q\op U$ is $\hpi_T$-reversible. The map $\op U$ is unitary from
$L^2(\hpi_T)$ onto $L^2(\refm)$, since
$\|g\|^2_{L^2(\hpi_T)}=\|g\circ T\|^2_{L^2(\refm)}$ and $\op U$ has inverse
$\op U^{-1}$, and
$\langle g_1,Q_Tg_2\rangle_{\hpi_T}=\langle\op Ug_1,Q\op Ug_2\rangle_{\refm}$.

(i) and (ii). Since $\dict_j=\op U^{-1}\phi_j$ and $\op U$ is unitary,
$\{\dict_j\}$ is an orthonormal basis of $L^2(\hpi_T)$, and
$Q_T\dict_j=\op U^{-1}Q\phi_j=\lambda_j\dict_j$. For $h=\sum_ja_j\dict_j$, we
have $\op Uh=\sum_ja_j\phi_j$, so
$\langle h,Q_Th\rangle_{\hpi_T}=\langle\op Uh,Q\op Uh\rangle_{\refm}
=\sum_j\lambda_ja_j^2$. Thus $Q_T$ is positive semidefinite if and only if
every $\lambda_j\ge0$. Since $Q_T-\lambda I=\op U^{-1}(Q-\lambda I)\op U$ for
every $\lambda$, $Q_T$ and $Q$ have the same spectrum, which does not depend
on $T$. Since $\dict_0=1$, $L^2_0(\hpi_T)$ consists of those $h$ with
$a_0=0$, so
$\sup\{|\langle h,Q_Th\rangle|:h\in L^2_0,\|h\|=1\}=\sup_{j\ge1}|\lambda_j|=\bar\lambda$,
and the absolute spectral gap of $Q_T$ is $1-\bar\lambda$.

(iii). Write $\Poi_F=\sum_{j\ge1}b_j\dict_j$ with $b_j=c_j/(1-\lambda_j)$. For
$j\ge1$, $|\lambda_j|\le\bar\lambda<1$ gives $1-\lambda_j\ge1-\bar\lambda>0$,
so $|b_j|\le|c_j|/(1-\bar\lambda)$. Applying $\op U^{-1}$ to $F\circ T=\sum_{j\ge1}c_j\phi_j$ gives
$F=\sum_{j\ge1}c_j\dict_j$, with $\|F\|^2=\sum_{j\ge1}c_j^2$ by orthonormality, so
\[
\sum_{j\ge1}b_j^2\le\frac{1}{(1-\bar\lambda)^2}\sum_{j\ge1}c_j^2
=\frac{\|F\|^2}{(1-\bar\lambda)^2}<\infty,
\]
so the series defining $\Poi_F$ converges in $L^2(\hpi_T)$ and, again by
orthonormality, $\|\Poi_F\|^2=\sum_{j\ge1}b_j^2\le\|F\|^2/(1-\bar\lambda)^2$,
which is the norm bound. Since $I-Q_T$ is bounded and
$(I-Q_T)\dict_j=(1-\lambda_j)\dict_j$, it can be applied term by term:
\[
(I-Q_T)\Poi_F=\sum_{j\ge1}(1-\lambda_j)b_j\dict_j=\sum_{j\ge1}c_j\dict_j=F.
\]
As $\Poi_F$ has no component along $\dict_0=1$, it is centered. If $G$ is another centered solution, then
$D=G-\Poi_F=\sum_{j\ge1}d_j\dict_j$ satisfies
$(I-Q_T)D=(I-Q_T)G-(I-Q_T)\Poi_F=F-F=0$. Since
$(I-Q_T)D=\sum_{j\ge1}(1-\lambda_j)d_j\dict_j$, this gives
$(1-\lambda_j)d_j=0$ and hence $d_j=0$ for all $j\ge1$, that is, $D=0$. Therefore the centered solution is unique.\qed

\subsection{Proof of Theorem~\ref{thm:validity}}
\label{pf:app:validity}

\subsubsection{SNIS}
\label{SNIS}

Let $\zeta_1,\dots,\zeta_n\sim\refm$ be independent, $w_i=w(\zeta_i)$,
$F_i=F(T(\zeta_i))$, $\phi_i=\phi(\zeta_i)$, $Z=\E_\refm[w]$ and
$\mu=\target(F)$, so that $\E_\refm[wF\circ T]=Z\mu$. The estimator
\eqref{eq:snis-cv} is $\bar N_n/\bar D_n$ with
$\bar N_n=n^{-1}\sum_i\{w_iF_i-\hat a_n^\top\phi_i\}$ and
$\bar D_n=n^{-1}\sum_i\{w_i-\hat b_n^\top\phi_i\}$, where $\hat a_n$ and
$\hat b_n$ are the least-squares coefficients, with intercept, of
$w_iF_i$ and of $w_i$ on $\phi_i$.

For fixed $(a,b)$ set $U_i=w_iF_i-a^\top\phi_i$ and $V_i=w_i-b^\top\phi_i$.
They are square-integrable by assumption, since $\phi_j\in L^2(\refm)$ and
$\E|wF\circ T\phi_j|<\infty$ by Cauchy--Schwarz. The strong law gives
$\bar U_n\to Z\mu$ and $\bar V_n\to Z$ almost surely, hence
$\bar U_n/\bar V_n\to\mu$, and the central limit theorem for
$(\bar U_n,\bar V_n)$ with the delta method for $(u,v)\mapsto u/v$, whose
gradient at $(Z\mu,Z)$ is $(1/Z,-\mu/Z)$, gives
\[
\sqrt n\Big(\frac{\bar U_n}{\bar V_n}-\mu\Big)\to
\gauss\Big(0,\frac{1}{Z^2}\Var_\refm\big[w(F\circ T-\mu)-(a-\mu b)^\top\phi\big]\Big).
\]
By orthonormality of $\phi$ the variance is minimized over $c=a-\mu b$ at
$c^*=\E_\refm[w(F\circ T-\mu)\phi]$.

For the fitted coefficients, the strong law gives $\hat a_n\to a^*=\E_\refm[wF\circ T\phi]$
and $\hat b_n\to b^*=\E_\refm[w\phi]$ almost surely, since the sample
covariance of the $\phi_i$ tends to $I$. Writing
$\bar N_n=n^{-1}\sum_i(w_iF_i-a^{*\top}\phi_i)-(\hat a_n-a^*)^\top\bar\phi_n$,
the second term is $o(1)$ almost surely and $o_p(n^{-1/2})$, because
$\sqrt n\,\bar\phi_n=O_p(1)$; the same holds for $\bar D_n$. Hence the
estimator has the limits of the fixed-coefficient case with
$(a,b)=(a^*,b^*)$, and $a^*-\mu b^*=c^*$, so
\[
\sigma^2_\infty=\frac{1}{Z^2}\Big\{\Var_\refm\big[w(F\circ T-\mu)\big]-\|c^*\|^2\Big\}.
\]
If $w\equiv1$, then $Z=1$, $c^*_j=c_j$ and $\sigma^2_\infty=\sum_{j\notin\mathcal K}c_j^2$,
which is Proposition~\ref{prop:pruning} with $\lambda_j\equiv0$.

\subsubsection{TGI}
\label{pf:app:tgi}

\paragraph{Ergodicity under (A2).} TGI-RWM, and TGI-MALA and TGI-pMALA under
the additional assumptions below, propose from
$\gauss(\beta z+b(z),s^2A_z)$, and we treat them at once through the
following lemma.

\begin{lemma}
\label{lem:app-perturbed}
Let $\gamma\in(0,1)$, $s^2=1-\beta^2$, and let the proposal be
$\gauss\big(\beta z+b(z),s^2A_z\big)$ with $A_z$ diagonal. Under (A2), assume that
$b$ and $A$ are continuous, write $E_z=A_z-I$, and that for constants $C_b,C_E$ and all $z$
\begin{equation}
\|b(z)\|\le\frac{C_b}{1+\|z\|},\qquad
\|E_z\|=\|A_z-I\|\le\frac{C_E}{1+\|z\|^2},\qquad
a_{\min}I\preceq A_z\preceq a_{\max}I .
\label{eq:app-perturb}
\end{equation}
Then the MH kernel $P$ with target $\pz$ is $V$-uniformly geometrically
ergodic with $V(z)=1+\|z\|^2$.
\end{lemma}

\begin{proof}
Write $q_0$ for the density of $\QGI$, and
$z'=\beta z+b(z)+sA_z^{1/2}\xi$ with $\xi\sim\gauss(0,I_d)$. Since $\pz=w\refm$ and $\QGI$ is $\refm$-reversible,
the MH ratio is $\{w(z')/w(z)\}e^{\Delta}$ with
\[
\Delta=\log\frac{q(z\,|\,z')}{q_0(z\,|\,z')}-\log\frac{q(z'\,|\,z)}{q_0(z'\,|\,z)}.
\]
Expanding the Gaussian densities,
\[
\log\frac{q(z'\,|\,z)}{q_0(z'\,|\,z)}
=\frac{s^2\xi^\top E_z\xi+\|b(z)\|^2+2s\,b(z)^\top A_z^{1/2}\xi}{2s^2}
-\frac12\sum_j\log a_{z,j},
\]
and with $u=z-\beta z'=s^2z-\beta b(z)-\beta sA_z^{1/2}\xi$ and $b'=b(z')$,
\[
\log\frac{q(z\,|\,z')}{q_0(z\,|\,z')}
=-\frac{u^\top(A_{z'}^{-1}-I)u-2b'^\top A_{z'}^{-1}u+b'^\top A_{z'}^{-1}b'}{2s^2}
-\frac12\sum_j\log a_{z',j},
\]
where $\|A_{z'}^{-1}-I\|\le\|E_{z'}\|/a_{\min}$.

\emph{Uniform bound.} We show that $\Delta$ is bounded when $\|z\|$ is large
and the noise is moderate. Fix $R\ge1$, let $\mathsf B_R=\{\|\xi\|\le R\}$, and
set
\[
K=C_b+s\sqrt{a_{\max}},\qquad
\kappa=\max\Big\{1,\frac{2K}{\beta},\frac{K}{s^2}\Big\},\qquad
\ell=\max\{|\log a_{\min}|,|\log a_{\max}|\}.
\]
Suppose $\|z\|\ge\kappa R$ and $\xi\in\mathsf B_R$. Then $\|z\|\ge R$, and
since $\|b(z)\|\le C_b$, $\|sA_z^{1/2}\xi\|\le s\sqrt{a_{\max}}R$,
\[
\|b(z)+sA_z^{1/2}\xi\|\le C_b+s\sqrt{a_{\max}}R\le KR.
\]
As $z'=\beta z+b(z)+sA_z^{1/2}\xi$ and $u=s^2z-\beta\{b(z)+sA_z^{1/2}\xi\}$,
using $\|z\|\ge2KR/\beta$, $\|z\|\ge KR/s^2$ and $\beta\le1$,
\[
\|z'\|\ge\beta\|z\|-KR\ge\frac{\beta}{2}\|z\|,\qquad
\|u\|\le s^2\|z\|+KR\le2s^2\|z\|.
\]
Therefore, by \eqref{eq:app-perturb} at $z'$,
\[
\|E_{z'}\|\le\frac{C_E}{\|z'\|^2}\le\frac{4C_E}{\beta^2\|z\|^2},\qquad
\|b'\|\le\frac{C_b}{\|z'\|}\le\frac{2C_b}{\beta\|z\|},
\]
and $\|A_{z'}^{-1}\|\le1/a_{\min}$, $\|A_{z'}^{-1}-I\|\le\|E_{z'}\|/a_{\min}$.

For the forward term, $\|z\|\ge R$ and \eqref{eq:app-perturb} at $z$ give
\[
|\xi^\top E_z\xi|\le\frac{C_ER^2}{1+\|z\|^2}\le C_E,\qquad
\|b(z)\|^2\le C_b^2,\qquad
|b(z)^\top A_z^{1/2}\xi|\le\frac{C_b\sqrt{a_{\max}}R}{1+\|z\|}\le C_b\sqrt{a_{\max}},
\]
and $\frac12|\sum_j\log a_{z,j}|\le d\ell/2$, so
\[
\Big|\log\frac{q(z'\,|\,z)}{q_0(z'\,|\,z)}\Big|
\le\frac{s^2C_E+C_b^2+2sC_b\sqrt{a_{\max}}}{2s^2}+\frac{d\ell}{2}.
\]
For the backward term, the bounds at $z'$ and $\|u\|\le2s^2\|z\|$ give
\[
|u^\top(A_{z'}^{-1}-I)u|\le\frac{\|E_{z'}\|\|u\|^2}{a_{\min}}
\le\frac{16C_Es^4}{a_{\min}\beta^2},\qquad
|2b'^\top A_{z'}^{-1}u|\le\frac{2\|b'\|\|u\|}{a_{\min}}
\le\frac{8C_bs^2}{a_{\min}\beta},
\]
and $b'^\top A_{z'}^{-1}b'\le C_b^2/a_{\min}$, so
\[
\Big|\log\frac{q(z\,|\,z')}{q_0(z\,|\,z')}\Big|
\le\frac{1}{2s^2a_{\min}}\Big(\frac{16C_Es^4}{\beta^2}+\frac{8C_bs^2}{\beta}+C_b^2\Big)
+\frac{d\ell}{2}.
\]
Adding the two bounds gives $|\Delta|\le C_0$ for a constant $C_0$ that
depends on neither $z$ nor $R$. By (A2), for $\|z\|\ge\kappa R$ the acceptance probability on
$\mathsf B_R$ is then
\[
\min\Big\{1,\frac{w(z')}{w(z)}e^{\Delta}\Big\}\ge\frac mMe^{-C_0}=:c.
\]

\emph{Drift.} Since a rejected move leaves $V$ unchanged,
\[
PV(z)-V(z)=\E_\xi\big[\alpha\{V(z')-V(z)\}\big],
\]
which we split according to $\mathsf B_R$ and its complement.

On $\mathsf B_R$, $\|z'\|\le\beta\|z\|+KR$. Let
$\bar\beta=\{(1+\beta^2)/2\}^{1/2}$, so that $\beta<\bar\beta<1$ and
$\bar\beta^2=1-s^2/2$, and let
$r_2(R)=\max\{\kappa R,\,KR/(\bar\beta-\beta)\}$. For $\|z\|\ge r_2(R)$ we
have $\|z'\|\leq\beta\|z\|+KR\le\bar\beta\|z\|$, so that
\[
V(z')-V(z)=\|z'\|^2-\|z\|^2\le(\bar\beta^2-1)\|z\|^2=-\frac{s^2}{2}\|z\|^2<0 .
\]
Since $\|z\|\ge\kappa R$, the acceptance probability is at least $c$ on
$\mathsf B_R$, and as $V(z')-V(z)<0$ there,
$\alpha\{V(z')-V(z)\}\le c\{V(z')-V(z)\}$. Therefore
\[
\E_\xi\big[\alpha\{V(z')-V(z)\}\mathbf 1_{\mathsf B_R}\big]
\le-\frac{cs^2}{2}\|z\|^2\,\mathbb P(\mathsf B_R).
\]

On $\mathsf B_R^c$, we use $\alpha\{V(z')-V(z)\}\le\|z'\|^2$ and
$\|z'\|\le\|z\|+C_b+s\sqrt{a_{\max}}\|\xi\|$, so that
$\|z'\|^2\le3(\|z\|^2+C_b^2+s^2a_{\max}\|\xi\|^2)$. Thus we have
\[
\E_\xi\big[\alpha\{V(z')-V(z)\}\mathbf 1_{\mathsf B_R^c}\big]
\le3(\|z\|^2+C_b^2)\,\mathbb P(\mathsf B_R^c)
+3s^2a_{\max}\E_\xi\big[\|\xi\|^2\mathbf 1\{\|\xi\|>R\}\big]
\le\varepsilon_RV(z),
\]
with
$\varepsilon_R=3\max\{1,C_b^2\}\,\mathbb P(\|\xi\|>R)
+3s^2a_{\max}\E[\|\xi\|^2\mathbf 1\{\|\xi\|>R\}]$, which does not depend on
$z$ and tends to $0$ as $R\to\infty$ by dominated convergence with $\|\xi\|^2\mathbf 1\{\|\xi\|>R\}\leq\|\xi\|^2$ and $\E_\xi[\|\xi\|^2]$ finite.

Combining the two parts, for $\|z\|\ge r_2(R)$,
\[
PV(z)-V(z)\le-\frac{cs^2}{2}\|z\|^2\,\mathbb P(\mathsf B_R)+\varepsilon_RV(z).
\]
Since $c$ does not depend on $R$, and since
$\mathbb P(\mathsf B_R)\to1$ and $\varepsilon_R\to0$ as $R\to\infty$, there
are $R_1,R_2\ge1$ such that $\mathbb P(\mathsf B_R)\ge\frac12$ for
$R\ge R_1$ and $\varepsilon_R\le cs^2/16$ for $R\ge R_2$, we then fix
$R=\max\{R_1,R_2\}$. For
$\|z\|\ge1$ we have $\|z\|^2\ge V(z)/2$, so with $r_3=\max\{1,r_2(R)\}$,
\[
PV(z)\le V(z)-\frac{cs^2}{8}V(z)+\frac{cs^2}{16}V(z)
=\Big(1-\frac{cs^2}{16}\Big)V(z),\qquad\|z\|\ge r_3 .
\]
On $\mathsf C=\{\|z\|\le r_3\}$, since a move is either accepted or rejected
and $\alpha\in[0,1]$,
\[
PV(z)=\E_\xi\big[\alpha V(z')+(1-\alpha)V(z)\big]\le\E_\xi V(z')+V(z).
\]
Using $\|z'\|^2\le3(\|z\|^2+C_b^2+s^2a_{\max}\|\xi\|^2)$, $\E_\xi\|\xi\|^2=d$
and $\|z\|\le r_3$,
\[
\E_\xi V(z')+V(z)
=2+\E_\xi\|z'\|^2+\|z\|^2
\le2+3(r_3^2+C_b^2+s^2a_{\max}d)+r_3^2,
\]
so $PV$ is bounded on $\mathsf C$. Together,
$PV\le\lambda V+b\mathbf 1_{\mathsf C}$ with $\lambda=1-cs^2/16<1$ and
$b=\sup_{\mathsf C}PV<\infty$.

\emph{Small sets.} On the bounded set $\mathsf C$, $\pz$ is bounded away from
$0$ and $\infty$ by (A2), and by \eqref{eq:app-perturb} the Gaussian proposal
density is bounded below on $\mathsf C\times\mathsf C$, thus $\mathsf C$ is
small, and since the proposal density is positive everywhere, the chain is
Lebesgue-irreducible and aperiodic.
The drift condition and \citet[Theorem~15.0.1]{meyn2009markov} give the claim.
\end{proof}

For the gradient-based samplers we assume in addition:
\begin{assumption}
\leavevmode
\begin{enumerate}
\item[(A4)] $\sup_z(1+\|z\|)\,\|\nabla\log w(z)\|<\infty$;
\item[(A5)] $w$ is twice continuously differentiable and
$\sup_z(1+\|z\|)^2\,\|\nabla^2\log w(z)\|<\infty$.
\end{enumerate}
\end{assumption}
They require the gradient and curvature of $\log w$ to decay in the tails,
so that the gradient-based proposals return to $\QGI$ far from the origin.

Now we verify \eqref{eq:app-perturb} for each sampler for $\gamma\in(0,1)$.
Write
\[
K_1=\sup_z(1+\|z\|)\,\|\nabla\log w(z)\|,\qquad
K_2=\sup_z(1+\|z\|)^2\,\|\nabla^2\log w(z)\|,
\]
which are finite by (A4) and (A5). TGI-RWM has $b\equiv0$ and $A_z\equiv I$,
so \eqref{eq:app-perturb} holds trivially. TGI-MALA has
$b=\gamma\nabla\log w$ and $A_z\equiv I$, so
$\|b(z)\|\le\gamma K_1/(1+\|z\|)$. For TGI-pMALA, write $h_{z,j}=\partial_j^2\log w(z)$, so that
$a_{z,j}=(1-h_{z,j})^{-1}$ is clipped to $[a_{\min},a_{\max}]$, and
$|h_{z,j}|\le\|\nabla^2\log w(z)\|\le K_2/(1+\|z\|^2)$. Let
$
h_0=\min\Big\{1/2,\ 1/a_{\min}-1,\ 1-1/a_{\max}\Big\}>0 .
$
If $|h_{z,j}|\le h_0$, then $(1-h_{z,j})^{-1}\in[a_{\min},a_{\max}]$,
so the clipping is inactive, and since $1-h_{z,j}\ge\frac12$,
\[
|a_{z,j}-1|=\frac{|h_{z,j}|}{1-h_{z,j}}\le2|h_{z,j}|.
\]
Let $r_0=\max\{0,K_2/h_0-1\}^{1/2}$, so that
$K_2/(1+\|z\|^2)\le h_0$ for $\|z\|\ge r_0$. Then
$|a_{z,j}-1|\le2K_2/(1+\|z\|^2)$ for $\|z\|\ge r_0$, while for $\|z\|<r_0$
the clipping gives
$|a_{z,j}-1|\le c_a\le c_a(1+r_0^2)/(1+\|z\|^2)$ with
$c_a=\max\{a_{\max}-1,1-a_{\min}\}$. Thus
$\|A_z-I\|\le C_E/(1+\|z\|^2)$ with $C_E=\max\{2K_2,c_a(1+r_0^2)\}$, and
$a_{\min}I\preceq A_z\preceq a_{\max}I$ by clipping. Its mean is $z+\gamma A_z\{-z+\nabla\log w(z)\}=\beta z+b(z)$ with
$b(z)=-\gamma E_zz+\gamma A_z\nabla\log w(z)$, and
$\|z\|(1+\|z\|)\le2(1+\|z\|^2)$ gives
$\|b(z)\|\le\gamma(2C_E+a_{\max}K_1)/(1+\|z\|)$. Lemma~\ref{lem:app-perturbed}
applies in all three cases. For $\gamma=1$ the TGI-RWM proposal is $\refm$
itself, and by (A2) the acceptance probability $\min\{1,w(z')/w(z)\}$ is at
least $m/M$, so $P(z,\cdot)\ge(m/M)\refm(\cdot)$ for every $z$ and the chain
is uniformly ergodic \citep[Theorem~16.0.2]{meyn2009markov}; for the
gradient-based proposals, $\beta=0$ removes the contraction used in the
proof, and the acceptance probability may decay with $\|z\|$.

\paragraph{LLN and CLT under (A2).} With $\delta=0$, $M_t\equiv1$ and
$\hat\mu_n$ is the plain average of the summands. We use the chain
$W_t=(Z_t,Z'_t,Z_{t+1})$, which is Markov because $W_{t+1}$ depends on $W_t$
only through $Z_{t+1}$. Its stationary law $\bar\pi$ is the law of
$(Z,Z',Z_{\mathrm{next}})$ with $Z\sim\pz$. Writing $P_W$ for the kernel of
$W$ and $R(z,\cdot)$ for the law of $(z,Z',Z_{\mathrm{next}})$ started at
$z$, we have $P_W^{n+1}(w_0,\cdot)=P_z^n(z_1,\cdot)R$ for
$w_0=(z_0,z'_0,z_1)$, and $\bar\pi=\pz R$, so
$\|P_W^{n+1}(w_0,\cdot)-\bar\pi\|_{\mathrm{TV}}
\le\|P_z^n(z_1,\cdot)-\pz\|_{\mathrm{TV}}$, and $W$ inherits geometric
ergodicity from $Z$ \citep[see also][]{dellaportas2012control}. The summand
$S_t$ is a function of $W_t$.

Under $\bar\pi$, $Z_t$ and $Z_{t+1}$ both have law $\pz$. For an MH sampler,
$Z_{t+1}=Z'_t$ with probability $\alpha_t$ and $Z_{t+1}=Z_t$ otherwise, so
$H_{1,t}=\E[G(Z_{t+1})-G(Z_t)\,|\,Z_t,Z'_t]$; for samplers without an
accept--reject step, $H_{1,t}=G(Z_{t+1})-G(Z_t)$. In both cases
$\bar\pi(H_1)=\pz(G)-\pz(G)=0$. By definition $\E[H_{2,t}\,|\,Z_t]=0$, so
$\bar\pi(H_2)=0$, and $\bar\pi(S)=\pz(F\circ T)=\target(F)$.

It remains to check that $S_t$ has a finite moment of order $2+\epsilon$
under $\bar\pi$, which the central limit theorem for geometrically ergodic
chains requires \citep[Theorem~9]{jones2004markov}. The term $F(T(Z_t))$ is
covered by (A3), and the next lemma handles the control variates.

\begin{lemma}
\label{lem:app-moments}
Under (A2) and (A3), $G\in L^{2+\epsilon}(\pz)\cap L^{2+\epsilon}(\refm)$.
Suppose that the candidate kernel $Q$ satisfies one of the following:
\begin{enumerate}
\item[(a)] $Q$ is $\refm$ invariant;
\item[(b)] $Q(z,\cdot)$ is Gaussian with mean bounded by $C(1+\|z\|)$ and
bounded covariance, and $G$ is a polynomial.
\end{enumerate}
Then $G(Z'_t)$ and $\E_{q(\cdot|Z_t)}[G]$ have finite moments of order
$2+\epsilon$ under $\bar\pi$.
\end{lemma}

\begin{proof}
Since $G$ is a finite combination of the $\phi_j$, (A3) and Minkowski's
inequality give $G\in L^{2+\epsilon}(\pz)$, and since $\refm\le\pz/m$ by
(A2), also $G\in L^{2+\epsilon}(\refm)$. Under $\bar\pi$,
$\E|G(Z')|^{2+\epsilon}=\int\E_{q(\cdot|z)}|G|^{2+\epsilon}\,\pz(dz)$, and
Jensen's inequality bounds $\E|\E_{q(\cdot|Z)}[G]|^{2+\epsilon}$ by the same
quantity, so it suffices to show that this integral is finite. Under (a),
$\pz\le M\refm$ and the invariance of $\refm$ give
$\int\E_{q(\cdot|z)}|G|^{2+\epsilon}\,\pz(dz)
\le M\refm(|G|^{2+\epsilon})<\infty$. Under (b), write $z'=\mu(z)+\Sigma(z)^{1/2}\xi$ with $\xi\sim\gauss(0,I_d)$,
$\|\mu(z)\|\le C(1+\|z\|)$ and $\|\Sigma(z)\|\le C_\Sigma$. If $G$ has degree
$k$, then $|G(x)|\le C_G(1+\|x\|)^k$ and
$1+\|z'\|\le(1+C+\sqrt{C_\Sigma}\|\xi\|)(1+\|z\|)$, combining the two bounds and taking expectations over $\xi$,
\[
\E_{q(\cdot|z)}|G|^{2+\epsilon}
\le C_G^{2+\epsilon}\,\E_\xi\Big[\big(1+C+\sqrt{C_\Sigma}\,\|\xi\|\big)^{k(2+\epsilon)}\Big]\,(1+\|z\|)^{k(2+\epsilon)}
=C'(1+\|z\|)^{k(2+\epsilon)},
\]
where $C'<\infty$ since $\|\xi\|$ has moments of all orders. Since
$\pz\le M\refm$,
$
\int\E_{q(\cdot|z)}|G|^{2+\epsilon}\,\pz(dz)
\le MC'\,\E_{\refm}\big[(1+\|z\|)^{k(2+\epsilon)}\big]<\infty .
$
\end{proof}

TGI-RWM satisfies (a), since $\QGI$ is $\refm$-reversible, and TGI-MALA and
TGI-pMALA satisfy (b) with the Hermite dictionary. For samplers without an accept--reject step,
(a) typically holds as well, for example, the initial point on the ellipse
in elliptical slice sampling is $\refm$ invariant.

Since $|\alpha_t|\le1$, (A3) and Lemma~\ref{lem:app-moments} give $S_t$ a
finite moment of order $2+\epsilon$ under $\bar\pi$. The ergodic theorem for
Harris-recurrent chains gives $\hat\mu_n\to\target(F)$ almost surely
\citep[Theorem~17.0.1]{meyn2009markov}, and the central limit theorem for
geometrically ergodic chains with $2+\epsilon$ moments
\citep[Theorem~9]{jones2004markov} gives
$\sqrt n(\hat\mu_n-\target(F))\to\gauss(0,\sigma^2_\infty)$ with
$\sigma^2_\infty<\infty$.

\paragraph{Ergodicity under (A2')} We first translate the first part of
(A2$'$) into growth bounds.

\begin{lemma}
\label{lem:app-growth}
Let $T$ be a composition of $L$ affine coupling layers
$y_A=x_A$, $y_B=x_B\odot e^{s(x_A)}+t(x_A)$ with $\|s\|_\infty\le s_{\max}$
and $\|t\|\le C_t$, followed by an invertible coordinatewise affine layer.
Then there are constants $c_0,c_1,c_2,D>0$ such that
$c_1\|z\|-c_0\le\|T(z)\|\le c_2\|z\|+c_0$ and $|\log|\det\nabla T(z)||\le D$
for all $z$.
\end{lemma}

\begin{proof}
Each coupling layer satisfies $\|y\|\le e^{s_{\max}}\|x\|+C_t$, and its
inverse $x_B=(y_B-t(y_A))\odot e^{-s(y_A)}$ satisfies
$\|x\|\le e^{s_{\max}}\|y\|+e^{s_{\max}}C_t$. Composing the $L$ layers and the
final affine layer gives $\|T(z)\|\le c_2\|z\|+c_0$ and
$\|T^{-1}(x)\|\le c_2'\|x\|+c_0'$, and the second bound gives
$\|T(z)\|\ge\|z\|/c_2'-c_0'/c_2'$; enlarging $c_0$ if needed gives the stated
form with $c_1=1/c_2'$. The log-determinant is the sum of the log-scales plus
a constant, so it is bounded by $Lds_{\max}$ plus that constant.
\end{proof}

Write $P_0$ for the TGI kernel and $P_\nu$ for the independence kernel with
proposal $q_\nu=t_\nu(0,I_d)$, so that $P=(1-\delta)P_0+\delta P_\nu$. Both
are MH kernels with target $\pz$, so $P$ is $\pz$-invariant. We show that
$\pz\le M_\nu q_\nu$. By Lemma~\ref{lem:app-growth},
$\pz(z)=\target(T(z))|\det\nabla T(z)|\le e^D\target(T(z))$. For
$\|z\|\ge R_0=2c_0/c_1$ we have $\|T(z)\|\ge c_1\|z\|/2$, so by (A2$'$)
\[
\pz(z)\le e^DC_\target\Big(1+\frac{c_1}{2}\|z\|\Big)^{-(d+\nu)}
\le e^DC_\target\max\Big\{1,\Big(\frac{2}{c_1}\Big)^{d+\nu}\Big\}(1+\|z\|)^{-(d+\nu)},
\]
while for $\|z\|<R_0$, $\pz(z)\le e^D\sup\target$. Since
$1+\|z\|^2/\nu\le\max\{1,1/\nu\}(1+\|z\|)^2$, the $t$ density satisfies
$q_\nu(z)\ge c_\nu(1+\|z\|)^{-(d+\nu)}$ for a constant $c_\nu>0$, so
$M_\nu=\sup_z\pz(z)/q_\nu(z)<\infty$. The independence kernel then
satisfies, for every $z$,
\[
P_\nu(z,dz')\ge\min\Big\{q_\nu(z'),\frac{\pz(z')\,q_\nu(z)}{\pz(z)}\Big\}dz'
=\pz(z')\min\Big\{\frac{q_\nu(z')}{\pz(z')},\frac{q_\nu(z)}{\pz(z)}\Big\}dz'
\ge\frac{1}{M_\nu}\,\pz(dz'),
\]
hence $P(z,\cdot)\ge\delta P_\nu(z,\cdot)\ge(\delta/M_\nu)\,\pz(\cdot)$. The
whole space is therefore small, and
$\sup_z\|P^n(z,\cdot)-\pz\|_{\mathrm{TV}}\le(1-\delta/M_\nu)^n$
\citep[Theorem~16.2.4]{meyn2009markov}. The argument uses only the
$\pz$-invariance of $P_0$, so it holds for every $\gamma\in(0,1]$ and for
all three TGI samplers, without (A4) or (A5).

\paragraph{LLN and CLT under (A2')} Let $B_t\in\{0,1\}$ indicate an
independence step, so $M_t=1-B_t$, and augment the chain  to
$W_t=(B_t,Z_t,Z'_t,Z_{t+1})$. Given $Z_{t+1}$, the next coin, candidate and
state are drawn with fresh randomness, so $W$ is a Markov chain, and as
before $\|P_W^{n+1}(w_0,\cdot)-\bar\pi\|_{\mathrm{TV}}
\le\|P^n(z_1,\cdot)-\pz\|_{\mathrm{TV}}$, so $W$ is uniformly ergodic, hence
positive Harris recurrent. Write $\mu=\target(F)$ and
\[
\hat\mu_n-\mu=\frac{n^{-1}\sum_tM_t(S_t-\mu)}{n^{-1}\sum_tM_t}.
\]
Since $B_t$ is independent of $Z_t$, under $\bar\pi$ we have $Z_t\sim\pz$
given $B_t=0$, and the transition is then $P_0$. Under
(A2), $\E[H_{1,t}\,|\,B_t=0]=\pz(P_0G)-\pz(G)=0$ by the invariance of $P_0$,
and $\E[H_{2,t}\,|\,Z_t,B_t=0]=0$ by definition, so
$\E[S_t\,|\,B_t=0]=\mu$ and $f(W_t)=M_t(S_t-\mu)$ has $\bar\pi(f)=0$. For the
moments, the bound $\pz\le M\refm$ of (A2) is replaced by
$\pz\le M_\nu q_\nu$. For the gradient-based samplers we assume, instead of
(A4) and (A5), only

\begin{assumption}
(A4$'$) $\|\nabla\log w(z)\|\le C_w(1+\|z\|)^r$ for some $r\ge1$.
\end{assumption}

which allows $w$ to grow in the tails. Since $A_z$ is clipped to
$[a_{\min},a_{\max}]$, the TGI candidate given $z$ is then Gaussian with
mean $m_0(z)$, $\|m_0(z)\|\le C_0(1+\|z\|)^r$, and covariance
$\Sigma_0(z)\preceq s^2a_{\max}I$, with $r=1$ for TGI-RWM.

\begin{lemma}
\label{lem:app-moments-mix}
Under (A2$'$), if the TGI candidate is as above, $G$ is a polynomial of
degree at most $K$ and $\nu>Kr(2+\epsilon)$, then $M_tG(Z_t)$,
$M_tG(Z'_t)$ and $M_t\E_{q_0}[G\,|\,Z_t]$ have finite moments of order
$2+\epsilon$ under $\bar\pi$.
\end{lemma}

\begin{proof}
Let $p=2+\epsilon$. Given $B_t=0$, the pair $(Z_t,Z'_t)$ has law
$\pz(dz)q_0(z,dz')$. Since $|G(x)|\le C_G(1+\|x\|)^K$, writing
$z'=m_0(z)+\Sigma_0(z)^{1/2}\xi$ and using
$1+\|z'\|\le(1+C_0+s\sqrt{a_{\max}}\|\xi\|)(1+\|z\|)^r$,
\[
\E_{q_0(\cdot|z)}|G|^p\le C'(1+\|z\|)^{Kpr},
\]
since $\|\xi\|$ has moments of all orders, and the same bound holds for
$|G(z)|^p$. Since $\pz\le M_\nu q_\nu$ and $q_\nu$ has finite moments of
every order below $\nu$,
\[
\int\E_{q_0(\cdot|z)}|G|^p\,\pz(dz)
\le M_\nu C'\,\E_{q_\nu}\big[(1+\|z\|)^{Kpr}\big]<\infty
\]
because $Kpr<\nu$; Jensen's inequality bounds $\E|\E_{q_0}[G\,|\,Z_t]|^p$
by the same quantity. Independence steps do not enter $\hat\mu_n$, so no
moment of $G$ under $q_\nu$ is needed.
\end{proof}

Since $|\alpha_t|\le1$, (A3) and Lemma~\ref{lem:app-moments-mix} give
$M_tS_t$ a finite moment of order $2+\epsilon$. By the ergodic theorem
\citep[Theorem~17.0.1]{meyn2009markov}, the numerator tends to $0$ and the
denominator to $1-\delta>0$ almost surely, so $\hat\mu_n\to\mu$ almost
surely. By the central limit theorem \citep[Theorem~9]{jones2004markov},
$n^{-1/2}\sum_tM_t(S_t-\mu)\to\gauss(0,\sigma_f^2)$ with $\sigma_f^2<\infty$,
and Slutsky's lemma gives
$\sqrt n(\hat\mu_n-\mu)\to\gauss(0,\sigma_f^2/(1-\delta)^2)$. \qed

\subsubsection{TESS}

\paragraph{Ergodicity.}
Under (A2) and (A2$^\ast$) with $\delta=0$, (A2) and the continuity in
(A2$^\ast$) make $w$ bounded away from $0$ and $\infty$ on every compact
set, and (A2$^\ast$) is the tail condition of
\citet[Theorem~1]{natarovskii2021geometric}, which then gives
$\|P^n(z,\cdot)-\pz\|_{\mathrm{TV}}\le C\kappa^n(1+\|z\|)$ for some
$C<\infty$ and $\kappa\in(0,1)$. The chain is therefore geometrically
ergodic, and in particular positive Harris recurrent. Under (A2$'$) with
$\delta\in(0,1)$, the minorization argument of the proof in
Appendix~\ref{pf:app:tgi} uses only the $\pz$-invariance of the kernel
that the independence step is mixed with, so it applies to TESS
unchanged and gives uniform ergodicity.

\paragraph{Estimator.}
The augmented chain $W_t=(Z_t,Z'_t,Z_{t+1})$, together with the
indicator $B_t$ of an independence step when $\delta>0$, is a Markov
chain, since given $Z_{t+1}$ the next angle, candidate, threshold and
shrinkage draws are fresh; it inherits the ergodicity of $\{Z_t\}$ as in
the proof in Appendix~\ref{pf:app:tgi}. Under its stationary law $Z_t$
and $Z_{t+1}$ both have law $\pz$, so $H_{1,t}=G(Z_{t+1})-G(Z_t)$ has
mean zero, and $H_{2,t}$ has conditional mean zero given $Z_t$ by
construction.

It remains to bound the moments of order $2+\epsilon$. The terms
$F(T(Z_t))$, $G(Z_t)$ and $G(Z_{t+1})$ are covered by (A3). For the
candidate $Z'_t=Z_t\cos\vartheta_t+V_t\sin\vartheta_t$, conditionally on
$(Z_t,\vartheta_t)$ it has law
$\gauss(Z_t\cos\vartheta_t,\sin^2\vartheta_t\,I_d)$. Under (A2) the
candidate kernel is $\refm$-invariant (Appendix~\ref{app:ref-ess}), so
Lemma~\ref{lem:app-moments}(a) applies. Under (A2$'$), the mean has norm at
most $\|Z_t\|$ and the covariance is at most $I_d$, uniformly in
$\vartheta_t$, so $1+\|Z'_t\|\le(1+\|\xi\|)(1+\|Z_t\|)$ with
$\xi\sim\gauss(0,I_d)$ and
\[
\E\big[|G(Z'_t)|^{2+\epsilon}\,\big|\,Z_t,\vartheta_t\big]
\le C'(1+\|Z_t\|)^{K(2+\epsilon)},
\]
which is integrable under $\pz\le M_\nu q_\nu$ because $\nu>K(2+\epsilon)$,
as in Lemma~\ref{lem:app-moments-mix} with $r=1$; Jensen's inequality
bounds $\E[G(Z'_t)\mid Z_t]$ in the same way. The ergodic theorem
\citep[Theorem~17.0.1]{meyn2009markov} and the central limit theorem for
geometrically ergodic chains \citep[Theorem~9]{jones2004markov}, applied
as in the proof of Theorem~\ref{thm:validity}, give both conclusions.

\subsection{Proof of Proposition~\ref{prop:pruning}}
\label{pf:app:pruning}

Since $w\equiv1$, we have $\pz=\refm$, and by (C1) and (C3) the chain has
kernel $Q$ with $Z'_t=Z_{t+1}$, so $\E_{q(\cdot|Z_t)}[G]=QG(Z_t)$. With
$\theta=(1,-1)$,
$\theta_1H_{1,t}+\theta_2H_{2,t}=\{G(Z_{t+1})-G(Z_t)\}-\{G(Z_{t+1})-QG(Z_t)\}
=QG(Z_t)-G(Z_t)$ for any coefficients $\eta$, so
$S_t=F(T(Z_t))+QG(Z_t)-G(Z_t)$. The asymptotic variance does not depend on
the initial distribution, so we compute it under stationarity,
$Z_0\sim\refm$. Since $F\circ T-\target(F)=\sum_{j\ge1}c_j\phi_j$ and
$QG=\sum_{j\in\mathcal K}\lambda_j\eta_j\phi_j$,
\[
S_t-\target(F)=\sum_{j\ge1}e_j\phi_j(Z_t),\qquad
e_j=\begin{cases}c_j-(1-\lambda_j)\eta_j,&j\in\mathcal K,\\ c_j,&j\notin\mathcal K.\end{cases}
\]
Since
$\mathrm{Cov}(\phi_j(Z_0),\phi_k(Z_t))=\langle\phi_j,Q^t\phi_k\rangle_{\refm}=\delta_{jk}\lambda_j^t$,
we have $\mathrm{Cov}(S_0,S_t)=\sum_{j\ge1}e_j^2\lambda_j^t$, and
\[
\sigma^2_\infty=\mathrm{Var}(S_0)+2\sum_{t\ge1}\mathrm{Cov}(S_0,S_t)
=\sum_{j\ge1}e_j^2\Big(1+2\sum_{t\ge1}\lambda_j^t\Big)=\sum_{j\ge1}v_je_j^2,
\]
where the interchange of sums is justified by
$\sum_{j\ge1}\sum_{t\ge1}e_j^2|\lambda_j|^t\le\bar\lambda(1-\bar\lambda)^{-1}\sum_{j\ge1}e_j^2<\infty$.
For $j\in\mathcal K$,
$e_j^2v_j=(1-\lambda_j)^2\{\eta_j-c_j/(1-\lambda_j)\}^2(1+\lambda_j)/(1-\lambda_j)
=(1-\lambda_j^2)\{\eta_j-c_j/(1-\lambda_j)\}^2$, which gives \eqref{eq:separable}.
With $\eta_j=\hat c_j/(1-\lambda_j)$, we have
$\eta_j-c_j/(1-\lambda_j)=(\hat c_j-c_j)/(1-\lambda_j)$, whose second moment
is $\Var(\hat c_j)/(1-\lambda_j)^2=s_j^2/\{N(1-\lambda_j)^2\}$ since $\hat c_j$
is unbiased, multiplying by $1-\lambda_j^2$ gives $v_js_j^2/N$. Since all terms are nonnegative, taking
expectations term by term gives \eqref{eq:expected}; no independence between
the $\hat c_j$ is needed. \qed

\begin{remark}
\label{rem:app-prune}
Since the terms of \eqref{eq:expected} are separable, keeping index $j$
changes the expected asymptotic variance by $v_j(s_j^2/N-c_j^2)$, so it is
beneficial exactly when $c_j^2>s_j^2/N$; under a budget $|\mathcal K|\le k$,
the expected variance is minimized by keeping the indices with the largest
positive scores $v_j(c_j^2-s_j^2/N)$. The total estimation noise in
\eqref{eq:expected} is at most
$N^{-1}|\mathcal K|\max_{j\in\mathcal K}v_j\max_{j\in\mathcal K}s_j^2$. For the
Gaussian reference, $v_\kappa=(1+\beta^{|\kappa|})/(1-\beta^{|\kappa|})$
decreases in $|\kappa|$ for $\beta\in[0,1)$, so
$v_\kappa\le v_1=(2-\gamma)/\gamma$ for $\gamma\in(0,1]$.
\end{remark}

\newpage
\section{REFERENCE KERNELS AND SAMPLERS}
\label{app:general}

Section~\ref{sec:reference} develops the construction for the Gaussian
reference, where $\QGI$ has the Hermite polynomials as eigenfunctions. The
same argument applies to any reference pair $(\refm,Q)$ with the
structure of (C1): the MH sampler that proposes from $Q$ and accepts with
probability $\min\{1,w(z')/w(z)\}$ targets $\pz=w\refm$, reduces to $Q$ when
$w\equiv1$, and satisfies (C2) since $\E_{q(\cdot|z)}[\phi_j]=\lambda_j\phi_j(z)$.
This appendix lists further reference pairs together with their spectral
decompositions, and treats samplers without an accept--reject step. The degenerate pair with $Q=\refm$, for which $\lambda_j=0$ for all
$j\ge1$, is importance sampling from the flow; it has no chain, so (C3)
is void and the control variate of Section~\ref{sec:estimator} is
applied as in \eqref{eq:snis-cv}. With
$\theta=(1,-1)$, the summand of \eqref{eq:twoterm} for an MH sampler is
\[
S_t=F(T(Z_t))+\E_{q(\cdot|Z_t)}[G]-G(Z_t)-(1-\alpha_t)\{G(Z'_t)-G(Z_t)\}.
\]
For samplers without an accept--reject step, $H_{1,t}=G(Z_{t+1})-G(Z_t)$ and
\[
S_t=F(T(Z_t))+\E_{q(\cdot|Z_t)}[G]-G(Z_t)+G(Z_{t+1})-G(Z'_t),
\]
where $Z'_t$ is an auxiliary point with known conditional law
$q(\cdot\mid Z_t)$; for elliptical slice sampling, $Z'_t$ is the initial
point on the ellipse. When $w\equiv1$, $\alpha_t=1$ and $Z'_t=Z_{t+1}$ in
both cases, so $S_t=F(T(Z_t))+QG(Z_t)-G(Z_t)$. Unless stated otherwise, the
numerical checks use $1000$ independent chains of $5000$ steps started from
the stationary distribution.

\subsection{Gaussian Reference}
\label{app:ref-gaussian}
\label{app:alg-samplers}

Let $\refm=\gauss(0,I_d)$ and let $\QGI$ be the kernel \eqref{eq:cn} with
$\gamma\in(0,2)$, so that $\beta\in(-1,1)$. For $z\sim\refm$ and
$z'\sim\QGI(z,\cdot)$, the pair $(z,z')$ is Gaussian with covariance
$\left(\begin{smallmatrix}I&\beta I\\ \beta I&I\end{smallmatrix}\right)$, which
is symmetric under exchange, so $\QGI$ is $\refm$-reversible. By
\eqref{eq:app-hermite-cond} and the orthonormality of $\{\Herm_\kappa\}$,
(C1) holds with $\phi_\kappa=\Herm_\kappa$, $\lambda_\kappa=\beta^{|\kappa|}$
and $\bar\lambda=\sup_{k\ge1}|\beta|^k=|\beta|$. All eigenvalues are
nonnegative if and only if $\beta\ge0$, so the transformed kernel is
positive semidefinite if and only if $\gamma\le1$, in which case its
spectral gap is $1-\beta=\gamma$; for $\gamma\in(1,2)$ its absolute spectral
gap is $1-|\beta|=2-\gamma$. When $w\equiv1$, $\nabla\log w=0$ and
$\partial_j^2\log w=0$, so the proposals of TGI-MALA and TGI-pMALA reduce
to $\QGI$, for TGI-pMALA because the clipping interval contains one; every
proposal is then accepted and the kernel is $\QGI$, so all three TGI
samplers satisfy (C1) and (C3). TGI-RWM satisfies (C2) because it proposes
from $\QGI$, and TGI-MALA and TGI-pMALA satisfy (C2) by
Lemma~\ref{lem:condmean} below.

The three TGI samplers of Section~\ref{sec:reference} are summarized in
Table~\ref{tab:variants}.

\begin{table}[h]
\centering
\caption{The TGI samplers on the Gaussian reference. The proposal is
$\gauss(\text{mean},\text{cov})$ in the latent space. Cost counts one
evaluation of $\log\target$ as one unit and a joint evaluation with its
gradient as two.}
\label{tab:variants}
\small
\begin{tabular}{@{}lllcc@{}}
\toprule
Sampler & Mean & Covariance & Derivatives & Cost per step \\
\midrule
TGI-RWM & $\beta z$ & $(1-\beta^2)I$ & none & $1$ \\
TGI-MALA & $\beta z+\gamma\nabla\log w(z)$ & $(1-\beta^2)I$ & gradient & $2$ \\
TGI-pMALA & $z+\gamma A_z\nabla\log\pz(z)$ & $(1-\beta^2)A_z$ & diagonal Hessian & $2+2d$ \\
\bottomrule
\end{tabular}
\end{table}

Each sampler evaluates $\log w$ at the candidate through \eqref{eq:residual},
that is, one forward pass of $T$ with its log-determinant and one evaluation
of $\log\target$; the gradient-based samplers add one backward pass. For
TGI-MALA and TGI-pMALA the proposal density is asymmetric, so the acceptance
probability includes the ratio of forward and reverse proposal densities,
which requires $\nabla\log w$ and $A_z$ at the candidate as well. These are
stored and reused at the next step when the candidate is accepted.

TGI-pMALA uses $A_z=\mathrm{diag}(a_{z,1},\dots,a_{z,d})$ with
$a_{z,j}=(1-h_{z,j})^{-1}$ and $h_{z,j}=\partial_j^2\log w(z)$. The
diagonal of the Hessian is computed with $d$ Hessian-vector products through
$\log\target\circ T$ and the log-determinant, and $a_{z,j}$ is clipped to
$[a_{\min},a_{\max}]$ with $a_{\min}<1<a_{\max}$. This is the diagonal form of the preconditioner
of \citet[Eq.~26]{titsias2026gaussian}, who average the $h_{z,j}$ because
their prior covariance is not diagonal.

In $x$-space, linearizing $T$ at $z$ with $J=\nabla T(z)$ gives
$x'-x\approx J(z'-z)$, so a move has noise covariance $(1-\beta^2)JJ^\top$ and,
for TGI-MALA, drift $\gamma JJ^\top\nabla_x\log\target(x)$ up to a term from
$\log|\det J|$. The map therefore acts as the position-dependent
preconditioner $JJ^\top$, which equals $\Sigma$ for the affine map to
$\gauss(\mu,\Sigma)$. For TGI-pMALA the preconditioner becomes $JA_zJ^\top$,
so $A_z$ corrects coordinate-wise scale errors of the map, but not shear or
correlation, which the map itself must capture.

For the gradient-based samplers, the conditional expectation in $H_{2,t}$ is
available in closed form.

\begin{lemma}
\label{lem:condmean}
Let $Z'\sim\gauss\big(\mu,\mathrm{diag}(\sigma_1^2,\dots,\sigma_d^2)\big)$ and
$\tau_j=1-\sigma_j^2$, which may be negative. Then
\begin{equation}
\E[\Herm_\kappa(Z')]=\prod_{j=1}^d\frac{1}{\sqrt{\kappa_j!}}
\sum_{i=0}^{\lfloor\kappa_j/2\rfloor}
\frac{(-1)^i\,\kappa_j!}{i!\,(\kappa_j-2i)!\,2^i}\;\mu_j^{\kappa_j-2i}\,\tau_j^{\,i},
\label{eq:condmean}
\end{equation}
and if $\tau_j=\tau>0$ for all $j$ this equals
$\tau^{|\kappa|/2}\Herm_\kappa(\mu/\sqrt\tau)$. TGI-MALA has
$\mu=\beta z+\gamma\nabla\log w(z)$ and $\sigma_j^2=1-\beta^2$; TGI-pMALA has
$\mu=z+\gamma A_z\nabla\log\pz(z)$ and $\sigma_j^2=(1-\beta^2)a_{z,j}$.
\end{lemma}

\begin{proof}
Coordinates are independent, so take $d=1$ and $x=\mu+\sigma\xi$. By
\eqref{eq:app-hermite-gen},
\[
\sum_{n\ge0}\E[\he_n(x)]\frac{t^n}{n!}=e^{t\mu+t^2\sigma^2/2-t^2/2}
=e^{t\mu}e^{-\tau t^2/2}
=\sum_{a\ge0}\frac{(t\mu)^a}{a!}\sum_{i\ge0}\frac{(-\tau/2)^it^{2i}}{i!},
\]
and the coefficient of $t^n/n!$ gives \eqref{eq:condmean} for every real
$\tau$. If $\tau>0$, then
\[
e^{t\mu-\tau t^2/2}=\sum_n\he_n(\mu/\sqrt\tau)\,(t\sqrt\tau)^n/n!,
\]
which gives the second form.
\end{proof}

\paragraph{Numerical check.}
We check Proposition~\ref{prop:pruning} on a target where every quantity in
it is known in closed form. Let $\target=\mathrm{LogNormal}(0,\sigma^2)$ with
$\sigma=0.5$, $T(z)=e^{\sigma z}$ the exact map, and $F(x)=x$. Then
$\pz=\refm=\gauss(0,1)$, $w\equiv1$, and every TGI-RWM proposal is accepted,
so the chain is the autoregression $z_{t+1}=\beta z_t+\sqrt{1-\beta^2}\,\xi_t$
with $\beta=1-\gamma$. The pulled-back test function
$F\circ T(z)=e^{\sigma z}$ has the normalized Hermite coefficients
\[
c_k=\E_{\refm}\big[e^{\sigma z}\Herm_k(z)\big]
=\frac{e^{\sigma^2/2}\sigma^k}{\sqrt{k!}},\qquad k\ge0,
\]
and $\pi(F)=c_0=e^{\sigma^2/2}$. We use the dictionary
$\{\Herm_1,\dots,\Herm_K\}$ with the exact coefficients and $\theta=(1,-1)$.

\emph{Prediction.} Since $\QGI\Herm_k=\beta^k\Herm_k$, the function
$G_K=\sum_{k=1}^K c_k(1-\beta^k)^{-1}\Herm_k$ satisfies
$G_K-\QGI G_K=F_K-c_0$, where $F_K=\sum_{k=0}^K c_k\Herm_k$ is the degree-$K$
truncation of $F\circ T$. With every proposal accepted, the estimator with
$\theta=(1,-1)$ reduces to
$\frac1n\sum_{t}\big(F\circ T(z_t)-G_K(z_t)+\QGI G_K(z_t)\big)
=c_0+\frac1n\sum_t\big(F\circ T-F_K\big)(z_t)$, the chain average of the
truncation residual. For a chain average of $\sum_k a_k\Herm_k$ the
asymptotic variance is $\sum_k a_k^2 v_k$ with
$v_k=(1+\beta^k)/(1-\beta^k)$, so
\[
\mathrm{VRF}(K)=\frac{\sum_{k\ge1}c_k^2v_k}{\sum_{k>K}c_k^2v_k},
\]
which we evaluate with the sums truncated at $k=80$.

\emph{Simulation.} For $\gamma=0.3$ ($\beta=0.7$) and $\gamma=1$ ($\beta=0$,
independent draws from $\refm$) we run $R=1000$ independent stationary chains
of $n=5000$ steps and compute, for each chain, the plain average of $F\circ T$
and the control-variate estimator. The empirical VRF is the ratio of their
mean squared errors over the $R$ chains, both measured against the exact
value $e^{\sigma^2/2}$. Table~\ref{tab:lognormal} compares the two. They agree
over fourteen orders of magnitude for both step sizes, and the
control-variate estimator is unbiased: its mean over the chains deviates from
$e^{\sigma^2/2}$ by at most $2.0$ standard errors for $\gamma=0.3$ and $1.2$
for $\gamma=1$.

For $K\ge8$ the empirical VRF exceeds the prediction, by a factor of about
$1.8$ at $K=10$ for both step sizes, including the independent draws at
$\gamma=1$, so the excess is not caused by the autocorrelation of the chain.
The residual is essentially $c_{K+1}\Herm_{K+1}$, whose second moment comes
from the far tail of $\refm$: $61\%$ of $\E[\Herm_{11}^2]$ comes from $|z|>4$,
an event of probability $6\times10^{-5}$. The squared error of a single chain
is therefore heavily right-skewed, and its average over a moderate number of
chains is usually too small; at $K=10$ and $\gamma=0.3$ the ratio of the
empirical to the predicted VRF falls from $1.9$ with $1000$ chains to $1.3$
with $80{,}000$.

\begin{table}[ht]
\centering
\caption{VRF of TGI-RWM on $\mathrm{LogNormal}(0,0.5^2)$ with the exact map,
$F(x)=x$ and $\theta=(1,-1)$. Predicted: Proposition~\ref{prop:pruning}.
Empirical: ratio of mean squared errors over $1000$ stationary chains of
$5000$ steps. At $\gamma=1$ the chain draws independently from $\refm$.}
\label{tab:lognormal}
\small
\setlength{\tabcolsep}{4pt}
\begin{tabular}{@{}llcccccc@{}}
\toprule
 & $K$ & 1 & 2 & 4 & 6 & 8 & 10 \\
\midrule
$\gamma=0.3$ & Predicted & 15.6 & 270 & $1.28\times10^{5}$ & $1.03\times10^{8}$ & $1.30\times10^{11}$ & $2.39\times10^{14}$ \\
 & Empirical & 15.1 & 268 & $1.29\times10^{5}$ & $1.17\times10^{8}$ & $1.77\times10^{11}$ & $4.58\times10^{14}$ \\
\midrule
$\gamma=1$ & Predicted & 8.35 & 102 & $3.35\times10^{4}$ & $2.27\times10^{7}$ & $2.63\times10^{10}$ & $4.66\times10^{13}$ \\
 & Empirical & 8.10 & 97.5 & $3.24\times10^{4}$ & $2.35\times10^{7}$ & $3.03\times10^{10}$ & $8.13\times10^{13}$ \\
\bottomrule
\end{tabular}
\end{table}

\subsection{Elliptical Slice Sampling}
\label{app:ref-ess}

Elliptical slice sampling \citep{murray2010elliptical} with prior $\refm$ and
likelihood $w$ also fits this framework; with $T$ a flow it is transport
elliptical slice sampling \citep{cabezas2023transport}. We take as candidate
the initial point on the ellipse, $Z'_t=Z_t\cos\vartheta+V\sin\vartheta$ with
$\vartheta\sim\mathrm{U}(0,2\pi)$ and $V\sim\gauss(0,I_d)$. Given
$\vartheta$, this is the GI move with $\beta=\cos\vartheta$ and
$s=\sin\vartheta$, where the sign of $s$ is irrelevant because $V$ is
symmetric, so \eqref{eq:app-hermite-cond} gives
$\E[\Herm_\kappa(Z'_t)\,|\,Z_t,\vartheta]=\cos^{|\kappa|}\vartheta\,\Herm_\kappa(Z_t)$.
Averaging over $\vartheta$, the candidate kernel $Q$ satisfies (C1) and
(C2) with $\phi_\kappa=\Herm_\kappa$ and
\[
\lambda_\kappa=\E[\cos^{|\kappa|}\vartheta]=
\begin{cases}0,&|\kappa|\text{ odd},\\ \binom{|\kappa|}{|\kappa|/2}2^{-|\kappa|},&|\kappa|\text{ even},\end{cases}
\]
and $Q$ is $\refm$-reversible because it is reversible conditionally on
$\vartheta$, as $\QGI$ is. All $\lambda_\kappa$ are nonnegative and the
largest for $|\kappa|\ge1$ is $1/2$, attained at $|\kappa|=2$, so $Q_T$ is
positive semidefinite with spectral gap $1/2$. When $w\equiv1$ the slice
contains the whole ellipse, so the initial point is accepted and (C3)
holds. The sampler has no accept--reject step, so
$H_{1,t}=G(Z_{t+1})-G(Z_t)$.

\paragraph{Numerical check.}
We check these properties on a target where the transport is exact. Let
$\target=\mathrm{LogNormal}(0,\sigma^2)$ with $\sigma=0.5$, $T(z)=e^{\sigma z}$
and $F(x)=x$, so that $w$ is constant and $F\circ T(z)=e^{\sigma z}$ has the
Hermite coefficients $c_k=e^{\sigma^2/2}\sigma^k/\sqrt{k!}$. We run the
elliptical slice sampler with prior $\refm$ and likelihood $w$, including
the threshold and the shrinking of the bracket, on $1000$ independent
stationary chains of $5000$ steps.

Every step evaluates exactly one point: the initial point on the ellipse is
always accepted, as (C3) requires. Table~\ref{tab:ess-check} (top) compares
the eigenvalues $\lambda_k=\E[\cos^k\vartheta]$ with the lag-one
autocorrelation of $\Herm_k(Z_t)$ along the chains. The odd orders are
uncorrelated, $\lambda_2=1/2$ is the largest, so the spectral gap is $1/2$,
and the estimates agree with $\lambda_k$ to within $0.04$; the largest
deviation, at $k=6$, comes from the heavy tails of $\Herm_6$ under $\refm$.

With $G_K=\sum_{k=1}^K c_k(1-\lambda_k)^{-1}\Herm_k$ and
$\theta=(1,-1)$, Proposition~\ref{prop:pruning} with
$v_k=(1+\lambda_k)/(1-\lambda_k)$ predicts the VRF of $F(x)=x$ as for
$\QGI$; the ratio of mean squared errors over the chains agrees with it
to within 7\% for $K\le6$, and the estimator is unbiased. The VRF at
$K=1$ is only 3.6 because the residual is then dominated by
$c_2\Herm_2$, the slowest component, with $v_2=3$. For $K\ge8$ the
residual variance comes from rare draws with $|z|>4$, and the average of
the heavily right-skewed squared errors over 1000 chains underestimates
it, as in Appendix~\ref{app:ref-gaussian}.

\begin{table}
\centering
\caption{Elliptical slice sampling on $\mathrm{LogNormal}(0,0.5^2)$ with the
exact map, $1000$ stationary chains of $5000$ steps. Top: eigenvalues of the
candidate kernel on $\Herm_k$ and the lag-one autocorrelation of
$\Herm_k(Z_t)$. Bottom: VRF of $F(x)=x$ with the dictionary
$\{\Herm_1,\dots,\Herm_K\}$ and $\theta=(1,-1)$, from
Proposition~\ref{prop:pruning} and from the ratio of mean squared errors.}
\label{tab:ess-check}
\small
\setlength{\tabcolsep}{4pt}
\begin{tabular}{@{}lcccccc@{}}
\toprule
$k$ & 1 & 2 & 3 & 4 & 5 & 6 \\
\midrule
$\lambda_k=\E[\cos^k\vartheta]$ & 0 & 0.5 & 0 & 0.375 & 0 & 0.3125 \\
Autocorrelation of $\Herm_k(Z_t)$ & 0.001 & 0.499 & $-0.002$ & 0.368 & $-0.009$ & 0.274 \\
\midrule
$K$ & 1 & 2 & 4 & 6 & 8 & 10 \\
\midrule
VRF, Proposition~\ref{prop:pruning} & 3.58 & 117 & $3.94\times10^{4}$ & $2.71\times10^{7}$ & $3.17\times10^{10}$ & $5.62\times10^{13}$ \\
VRF, simulation & 3.62 & 109 & $3.81\times10^{4}$ & $2.91\times10^{7}$ & $4.69\times10^{10}$ & $1.09\times10^{14}$ \\
\bottomrule
\end{tabular}
\end{table}

\subsection{Gamma Reference}
\label{app:ref-gamma}

For parameters on $(0,\infty)$ with a power-law behaviour near zero and
exponential tails, a Gamma reference matched to the target is an
alternative to a log transform followed by a Gaussian reference. Let
$\refm=\mathrm{Gamma}(a,1)^{\otimes d}$ and let $Q$ act on each coordinate as
the transition over time $-\log\beta$ of the Cox--Ingersoll--Ross process
$dX=(a-X)\,dt+\sqrt{2X}\,dW$ \citep{cox1985theory}. Its transition law over
time $t$ is a scaled noncentral chi-square distribution \citep{cox1985theory},
which at $t=-\log\beta$ is the Gamma--Poisson mixture
\[
N\sim\mathrm{Poisson}\big(\beta z/(1-\beta)\big),\qquad
z'\sim(1-\beta)\,\mathrm{Gamma}(a+N,1).
\]
The process is reversible with respect to $\mathrm{Gamma}(a,1)$
\citep{diaconis2008gibbs}, and its generator is
$\mathcal Lf=xf''+(a-x)f'$. The generalized Laguerre polynomial
$L^{(a-1)}_k$ solves $xy''+(a-x)y'=-ky$, so
$e^{t\mathcal L}L^{(a-1)}_k=e^{-kt}L^{(a-1)}_k$, which at $t=-\log\beta$ is
$\beta^kL^{(a-1)}_k$. Since the Laguerre polynomials are orthogonal and
complete in $L^2(\mathrm{Gamma}(a,1))$, (C1) holds with the normalized
polynomials
\[
\phi_\kappa(z)=\prod_j\frac{L^{(a-1)}_{\kappa_j}(z_j)}{\sqrt{\Gamma(\kappa_j+a)/(\kappa_j!\,\Gamma(a))}},
\qquad\lambda_\kappa=\beta^{|\kappa|},
\]
and (C2) holds since $\E_{Q(z,\cdot)}[\phi_\kappa]=\beta^{|\kappa|}\phi_\kappa(z)$.
As $\beta\in(0,1)$, all eigenvalues are nonnegative, so $Q_T$ is positive
semidefinite with spectral gap $1-\beta$.

\paragraph{Numerical check.}
With $\target=\mathrm{LogNormal}(0,0.5^2)$, the reference
$\mathrm{Gamma}(2,1)$, the exact quantile map, $F(x)=x$ and $\beta=0.6$,
the VRF over 1000 chains of 5000 steps is 444, 1970, 6610 and 13200 for
$K=1,\dots,4$, against 465, 1950, 6790 and 13300 from
Proposition~\ref{prop:pruning}; with the misspecified maps
$T_s(z)=T(z/s)$ it falls to about 25 for $s=0.9,1.1$ and 6 for
$s=0.8,1.25$. Table~\ref{tab:gamma-vs-gauss} compares the two references on
$\tfrac12\mathrm{Gamma}(2,1)+\tfrac12\mathrm{Gamma}(5,1)$ with
moment-matched affine maps and with exact quantile maps: with the
former, the Gamma reference matches the target near zero, keeps $w$
bounded and is about 50 times more efficient, whereas with exact maps
the two perform alike.

\begin{table}
\centering
\caption{Gaussian and Gamma references on
$\tfrac12\mathrm{Gamma}(2,1)+\tfrac12\mathrm{Gamma}(5,1)$, $F(x)=x$, $K=6$,
over $1000$ chains of $5000$ steps started from exact draws of $\pz$. For
each row, $\beta\in[0,0.95]$ is chosen on a grid to maximize
ESS$\times$VRF. The last column is the largest value of $\log w$ minus its
$99.9\%$ quantile on $2\times10^5$ target draws; for the Gaussian
moment-matched map $\log w$ is unbounded and this gap grows with the number
of draws.}
\label{tab:gamma-vs-gauss}
\setlength{\tabcolsep}{4pt}
\begin{tabular}{@{}llcccccc@{}}
\toprule
Reference & Map & $\beta$ & Acceptance & ESS$/n$ & VRF & ESS$\times$VRF$/n$ & Tail gap \\
\midrule
Gaussian & moment-matched & 0.5 & 0.78 & 0.11 & 1.9 & 0.20 & 14.4 \\
Gamma & moment-matched & 0.1 & 0.92 & 0.60 & 16 & 9.7 & 0.37 \\
Gaussian & exact & 0 & 1.00 & 0.97 & $9.3\times10^3$ & $9.0\times10^3$ & 0 \\
Gamma & exact & 0.2 & 1.00 & 0.60 & $1.1\times10^4$ & $6.9\times10^3$ & 0 \\
\bottomrule
\end{tabular}
\end{table}

\subsection{Circle and Sphere}
\label{app:ref-manifold}

For angular parameters, the reference can be the uniform distribution on a
circle or a sphere, and flows on these spaces are available
\citep{rezende2020normalizing}.

\emph{Circle.} Let $\refm$ be the uniform distribution and let $Q$ be the
step $\theta'=\theta+\sqrt\tau\,\xi$ modulo $2\pi$ with $\xi\sim\gauss(0,1)$.
Its density depends only on $\theta'-\theta$ and is symmetric, so $Q$ is
$\refm$-reversible. Since
$\E[e^{ik\theta'}\,|\,\theta]=e^{ik\theta}\E[e^{ik\sqrt\tau\xi}]=e^{-k^2\tau/2}e^{ik\theta}$,
(C1) holds with the orthonormal Fourier basis
$1,\sqrt2\cos k\theta,\sqrt2\sin k\theta$ ($k\ge1$) and eigenvalues
$e^{-k^2\tau/2}$; these are positive, so $Q_T$ is positive semidefinite with
spectral gap $1-e^{-\tau/2}$.

\emph{Sphere.} Let $\refm$ be the uniform distribution on $S^2$ and let $Q$
move $x$ to $x\cos\varphi+v\sin\varphi$, with $v$ uniform on the unit circle
of the tangent plane at $x$ and $\varphi\in(0,\pi)$ fixed. Then $Q(x,\cdot)$
is uniform on the circle $\{y:\langle x,y\rangle=\cos\varphi\}$, so it depends
only on $\langle x,y\rangle$ and is symmetric, hence $\refm$-reversible. By the
Funk--Hecke formula \citep{atkinson2012spherical}, such a zonal kernel acts
on each spherical harmonic of degree $\ell$ by multiplication with
$P_\ell(\cos\varphi)$, where $P_\ell$ is the Legendre polynomial, so (C1)
holds with the real spherical harmonics. Since $|\cos\varphi|<1$, we have
$|P_\ell(\cos\varphi)|<1$ for $\ell\ge1$ and $P_\ell(\cos\varphi)\to0$, so
$\bar\lambda=\sup_{\ell\ge1}|P_\ell(\cos\varphi)|$ is attained and smaller than
one. Some eigenvalues may be negative, for instance
$P_2(\cos\varphi)<0$ when $\cos^2\varphi<1/3$, so $Q_T$ need not be positive
semidefinite; replacing the fixed step by the heat kernel, with eigenvalues
$e^{-\ell(\ell+1)t/2}$, restores positivity.

In both cases, (C2) holds for the MH sampler that proposes from $Q$, by the
eigenvalue relation.

\paragraph{Numerical check.} Table~\ref{tab:manifold} reports the MH
sampler that proposes from $Q$, with quantile maps, of the polar angle for
the sphere, for a mixture of von Mises distributions on the circle and a
von Mises--Fisher distribution on $S^2$. With the exact map every proposal
is accepted, and the ESS and VRF agree with Proposition~\ref{prop:pruning}.
On these compact spaces $w$ is always bounded, but a map whose distribution
is too concentrated makes $w$ vary over a wide range, here
$M/m=e^{6}\approx400$ for $\kappa'=6$, so chains stall in the regions the map
underweights, as seen in the low ESS.

\begin{table}
\centering
\caption{The MH sampler proposing from $Q$ on the circle
($0.7\,\mathrm{vM}(0.5,4)+0.3\,\mathrm{vM}(3.5,2)$, $F=\cos\theta$, Fourier
degree $6$) and on the sphere ($\mathrm{vMF}$ with $\kappa=3$, $F=x\cdot\mu$,
harmonics of degree at most $4$), with $1000$ chains of $5000$ steps. Each
map is the quantile map of a model $\pi'$, so that $T_\pf\refm=\pi'$ and
$w=(\target/\pi')\circ T$: on the circle, $\pi'$ is the target mixture with
both concentrations scaled; on the sphere, $\pi'=\mathrm{vMF}(\kappa')$.
Predicted values are from Proposition~\ref{prop:pruning} for the exact
maps.}
\label{tab:manifold}
\small
\setlength{\tabcolsep}{4pt}
\begin{tabular}{@{}llcccccc@{}}
\toprule
& & & & \multicolumn{2}{c}{ESS$/n$} & \multicolumn{2}{c}{VRF} \\
\cmidrule(lr){5-6}\cmidrule(lr){7-8}
Target & Map & Step & Acceptance & observed & predicted & observed & predicted \\
\midrule
Circle & exact & $\tau=0.5$ & 1.00 & 0.15 & 0.15 & 1900 & 1810 \\
Circle & exact & $\tau=2$ & 1.00 & 0.52 & 0.53 & 502 & 518 \\
Circle & concentrations $\times0.5$ & $\tau=2$ & 0.80 & 0.31 & -- & 5.4 & -- \\
Circle & concentrations $\times2$ & $\tau=2$ & 0.77 & 0.39 & -- & 2.8 & -- \\
Sphere & exact & $\varphi=0.6$ & 1.00 & 0.11 & 0.12 & 783 & 678 \\
Sphere & exact & $\varphi=1.2$ & 1.00 & 0.54 & 0.55 & 63 & 59 \\
Sphere & $\kappa'=1.5$ & $\varphi=1.2$ & 0.74 & 0.32 & -- & 3.1 & -- \\
Sphere & $\kappa'=6$ & $\varphi=1.2$ & 0.69 & 0.02 & -- & 1.04 & -- \\
\bottomrule
\end{tabular}
\end{table}

\newpage
\section{ALGORITHMS AND IMPLEMENTATION}
\label{app:algorithms}

This appendix gives the procedure for a general reference pair satisfying
(C1)-(C3) (Appendix~\ref{app:general}), followed by the
details specific to the Gaussian reference. The settings used in the
experiments are given in Appendix~\ref{app:experiments}.

\subsection{General Procedure}
\label{app:alg-general}

Algorithm~\ref{alg:pipeline} prepares the map and chooses the sampler, and
Algorithm~\ref{alg:tgi} builds the control variate and returns the estimate.
Both only require the reference pair $(\refm,Q)$, its eigenfunctions
$\phi_j$ and eigenvalues $\lambda_j$, and the conditional expectations of
(C2).

\begin{algorithm}
\caption{Preparing the map and choosing the estimator}
\label{alg:pipeline}
\begin{algorithmic}[1]
\REQUIRE unnormalized $\log\target$, reference $\refm$, family of maps,
training budget
\STATE map constrained parameters to the support of $\refm$
\STATE train $T$ (Section~\ref{sec:algorithm}), then freeze $T$
\STATE draw $\zeta_1,\dots,\zeta_N\sim\refm$ and compute
$\mathrm{ESS}_{\mathrm{IS}}$ and the largest normalized weight from
$w(\zeta_i)$
\STATE run a short pilot chain on $\pz$ that proposes from $Q$, and compute
the tail gap of $\log w$
\STATE choose SNIS, TGI-RWM or TESS by the rule of
Section~\ref{sec:diagnostics}
\RETURN frozen map $T$ and estimator
\end{algorithmic}
\end{algorithm}

\begin{algorithm}
\caption{Estimation with transformed control variates}
\label{alg:tgi}
\begin{algorithmic}[1]
\REQUIRE frozen map $T$, estimator, test function $F$, candidate index set
$\mathcal K_0$, reference budgets $N_0$ and $N$, threshold $\tau_{prune}$, burn-in
length $n_b$, run length $n$
\STATE draw $N_0$ points $\zeta_i\sim\refm$ and compute
$\hat c_j=N_0^{-1}\sum_iF(T(\zeta_i))\phi_j(\zeta_i)$ and the sample variance
$\hat s_j^2$ of $F(T(\zeta_i))\phi_j(\zeta_i)$, for $j\in\mathcal K_0$
\STATE keep $\mathcal K=\{j\in\mathcal K_0:\hat c_j^{\,2}>\tau_{prune}\,\hat s_j^2/N_0\}$
\STATE draw $N$ fresh points from $\refm$ and re-estimate $\hat c_j$ for
$j\in\mathcal K$
\STATE set $\eta_j=\hat c_j/(1-\lambda_j)$ and
$G=\sum_{j\in\mathcal K}\eta_j\phi_j$
\IF{the estimator is SNIS}
\STATE draw $n$ points $\zeta_i\sim\refm$ and \RETURN $\hat\mu^{\mathrm{IS}}$
of \eqref{eq:snis-cv}, with $a$ and $b$ fitted on these draws
\ENDIF
\STATE run $n_b$ burn-in steps, choosing the step size (and optionally
$\theta$) as in Appendix~\ref{app:alg-stepsize}, then fix it
\FOR{$t=0,\dots,n-1$}
\STATE draw the candidate $Z'_t\sim q(\cdot|Z_t)$ and move to $Z_{t+1}$
\STATE $H_{1,t}\leftarrow\alpha_t\{G(Z'_t)-G(Z_t)\}$ for an MH sampler, and
$G(Z_{t+1})-G(Z_t)$ otherwise
\STATE $H_{2,t}\leftarrow G(Z'_t)-\E_{q(\cdot|Z_t)}[G]$, which equals
$G(Z'_t)-\sum_j\lambda_j\eta_j\phi_j(Z_t)$ when the candidate is drawn from $Q$
\STATE $S_t\leftarrow F(T(Z_t))+\theta_1H_{1,t}+\theta_2H_{2,t}$
\ENDFOR
\RETURN $\hat\mu_n=n^{-1}\sum_{t=0}^{n-1}S_t$
\end{algorithmic}
\end{algorithm}

The coefficients $\hat c_j$ depend on $T$ and $F$ but not on the step size,
so lines 1--3 of Algorithm~\ref{alg:tgi} are run once and reused for every
step size and every run, and several test functions share the chain and
differ only in their coefficients. Lines 1--3 require forward passes of $T$
and evaluations of $F$, but no evaluation of $\target$. Selection and
estimation use independent reference draws, so that $\mathcal K$ does not
depend on the coefficients it retains, and all quantities in lines 11--14 are
computed in the latent space, so $T^{-1}$ is never evaluated. In practice
many chains run in parallel and their estimates are averaged.

\subsection{Implementation for the Gaussian Reference}
\label{app:alg-gauss}
\paragraph{Diagnostics}
\label{app:alg-diagnostics}
$\mathrm{ESS}_{\mathrm{IS}}$ and the largest normalized weight are
computed with the log-sum-exp trick, since $w$ is known up to a
constant. The pilot chain gives the Fisher divergence
$\hat{\mathcal F}=n^{-1}\sum_t\|\nabla\log w(Z_t)\|^2$, which needs the
gradient only on the pilot, and the tail gap, the difference between the
largest $\log w$ on the pilot and its 99.9\% quantile; a large gap
indicates an unbounded weight and chains that may stall in the tails.

\paragraph{Dictionary and Pruning}
\label{app:alg-dictionary}
The candidate set is $\mathcal K_{K,p}$ of multi-indices of total degree
at most $K$ involving at most $p$ coordinates, of size
$\sum_{k=1}^p\binom{d}{k}\binom{K}{k}$, enumerated directly, with
$\tau_{prune}=2\log|\mathcal K_{K,p}|$. The one-dimensional values
$\mathrm{he}_n(z_j)$, $n\le K$, are computed by the recurrence
$\mathrm{he}_{n+1}(x)=x\,\mathrm{he}_n(x)-n\,\mathrm{he}_{n-1}(x)$, so
evaluating $G$ costs $O(dK+p|\mathcal K|)$; $\E_{q(\cdot|z)}[G]$ reuses
the same values for TGI-RWM and the factors of
Lemma~\ref{lem:condmean} for TGI-MALA.

\paragraph{Step Size}
\label{app:alg-stepsize}
The objective of Section~\ref{sec:gamma} is maximized with the lower
bound $\min(0,\log r)$ of the acceptance rate \citep{titsias2019gradient}
in place of $\log a(\gamma)$. During burn-in $\log\gamma$
follows stochastic gradient ascent with learning rate $0.05$ and
$\gamma\in[10^{-3},1]$, using the reparametrized gradient of
\citet{titsias2019gradient} for the first two terms and a running
average of the acceptance probabilities for the last; the per-chain
gradients are winsorized at their 10\% and 90\% quantiles before
averaging, without which a few chains at the seam between modes drive
$\gamma$ to its lower bound. The step size of the evaluation run is the
average of $\log\gamma$ over the second half of burn-in. $V_0$ uses
$\mathrm{Var}_\target(F)$ from the burn-in samples and the residual
variance $R$ on reference draws independent of those used for the
coefficients. When $\theta$ is estimated it is fitted on the burn-in
samples by minimizing the empirical variance of the summand and then
frozen.

\newpage
\section{EXPERIMENTAL DETAILS}
\label{app:experiments}

We first give the settings shared by all experiments, then compare training
choices for the map (Appendix~\ref{app:training-results}) and give the
details of the synthetic targets (Appendix~\ref{app:exp-synthetic}) and the
real posteriors (Appendix~\ref{app:exp-real}).

\paragraph{Protocol.}
For each target the flow is trained once and frozen, and replicates are
independent estimates on it, never retrainings. Each estimate uses
$B=10^6$ target evaluations, an evaluation of $\log\target$ counting $1$
and one with its gradient $2$; a draw of SNIS, a step of TGI-RWM and a
candidate of TESS cost $1$, a step of TGI-MALA and a leapfrog step of
HMC cost $2$. A chain-based estimate averages $200$ chains of cost
$5000$ started from $z\sim\gauss(0,I)$. Burn-in is included and takes
$10\%$ of the cost for TGI-RWM and TESS, $20\%$ for TGI-MALA and
$20\%$ for NeuTra. For comparison reason, the step size $\gamma$ of the TGI samplers
and the number $L$ of leapfrog steps of NeuTra are tuned in the same way: they are chosen on a grid by a
pilot of chains of cost $10^4$ on the main test function, which is not
charged; the step size of NeuTra is adapted by dual averaging during
warm-up. Each setting is replicated $R=100$ times, which gives a
relative standard error of about $7\%$ for the RMSE and $14\%$ for the
precision per second, and all results use $\delta=0$ unless stated
otherwise. The main test function of each target is shown in Table~\ref{tab:exp-synthetic} and Table~\ref{tab:exp-real}.

Wall-clock time is the time of a batch of $R$ estimates advanced in
lockstep, divided by $R$. It includes burn-in and, for SNIS and
TGI-MALA, the one-time compilation, and excludes the training of the
flow, the fitting of the control-variate coefficients and the pilot
runs. NUTS, which runs on the CPU, is compared at equal numbers of
target evaluations only.

\paragraph{Other experiments.}
Two groups of experiments depart from this protocol. Figure~\ref{fig:banana-family}
and Table~\ref{tab:indep-limit} measure the efficiency per step rather
than per second, with $500$ independent chains started from
$z\sim\gauss(0,I)$, $500$ burn-in and $5000$ recorded steps, which gives
a relative precision of about $\pm6\%$ for ESS$/n$ and VRF; the
step-size curves use a grid of $25$ values of $\gamma$ and $5$ adaptive
runs of $J$ on the same chains, and the pruning curves use
$K=1,\dots,6$ and five pruning rules. The comparison with NUTS
(Table~\ref{tab:nuts} and Figure~\ref{fig:synthetic-bench}) adapts
$\gamma$ by $J$ during burn-in, at one extra backward pass per step
which is counted, and sets the leapfrog steps of NeuTra by
$L=\lceil d^{1/4}/\epsilon\rceil$ at acceptance $0.9$. There the
flow-based methods use chains of cost $5000$, their number set by the
budget, and $100$ estimates per setting, and NUTS uses chains of cost
$2\times10^4$ and $50$ estimates per setting. Table~\ref{tab:nuts}
reports the RMSE on the real posteriors at $B=10^6$ without the training
of the flow, and Figure~\ref{fig:synthetic-bench} on the banana with
the full training charged, as the lower envelope over training
checkpoints and sampling budgets, with bootstrap $95\%$ intervals from
$2000$ resamples.

\paragraph{Transport.}
\label{app:training}
The transport is a RealNVP flow \citep{dinh2017density} with affine
coupling layers and alternating masks. Each conditioner is a multilayer
perceptron with two hidden layers, $\tanh$ activations and a
zero-initialized output layer, and the log-scale is parametrized as
$s_{\max}\tanh(\cdot)$. A final coordinatewise affine layer $x=m+s\odot v$
is initialized at a Gaussian approximation of the target, the Laplace
approximation or a moment-matched one, so training starts from that
approximation. We use two sizes, SMALL with $8$ layers, width $64$ and $s_{\max}=1.5$,
and LARGE with $12$ layers, width $128$ and $s_{\max}=3$.
Constrained parameters are first mapped to $\reals^d$ by a log or logit
transform. The flows are trained with Adam \citep{kingma2015adam} for
$10^4$ steps with batch size $256$, a learning rate decayed from $10^{-3}$
to $10^{-5}$ by a cosine schedule, and gradient norms clipped at $100$,
with two exceptions: eight schools is trained for $3\times10^4$ steps,
and the sparse logistic regression for $3\times10^5$ steps with batch size
$1024$, an initial learning rate of $3\times10^{-4}$ and a network of $12$
layers of width $256$. The map is frozen before the evaluation runs. Two
training schemes are used, and a third extends the first to multimodal
targets.
\begin{itemize}
\item \emph{Scheme $\schA$} combines the reverse and an adaptive forward
Kullback--Leibler divergence. The reverse term is estimated at each step
from $256$ draws $\zeta\sim\refm$ as
$256^{-1}\sum\{-\log\target(T(\zeta))-\log|\det\nabla T(\zeta)|\}$. For the
forward term, $32$ TGI-RWM chains advance $10$ steps every $10$ training
steps during training, with $\gamma$ adapted to an acceptance rate of
$0.5$. The chain states are stored in $x$-space and pulled back through the
current flow after each update, so the chains are never restarted. Visited
states enter a sliding buffer of $2\times10^4$ samples, and from $10\%$ of
the training steps onwards the forward Kullback--Leibler divergence on the
buffer is added to the loss with weight one; earlier states are discarded.
\item \emph{Scheme $\schB$} minimizes the reverse Kullback--Leibler
divergence alone.
\item \emph{Scheme $\schAa$}, used only for multimodal targets, extends $\schA$ as follows. A pilot of
$16$ NUTS chains ($64$ on the $\phi^4$ field) is started from $\mathrm{Uniform}(-2,2)^d$; coordinates
are whitened with the average within-chain covariance, and the flow is
trained in whitened coordinates. The chains are grouped into modes by
the Mahalanobis distance between their means, with threshold $5$, and
half of each forward Kullback--Leibler minibatch is drawn from the
pilot, with equal weight per group; we call this the \emph{anchor}, since
it keeps every mode represented in the flow whatever mass the current
flow assigns to it. The weights of the modes are left to the reweighting of SNIS or to the chains, and the cost of the pilot is included in
the total cost.
\end{itemize}

\paragraph{Tail penalty.}
 Unless it is switched off (Tables~\ref{tab:exp-synthetic} and~\ref{tab:exp-real}), the loss includes
\[
\frac1m\sum_{i=1}^m\Big[\max\big(0,\ \log w(\zeta_i)-\overline{\log w}-c\big)\Big]^2,
\qquad \zeta_i=r_iu_i,\quad u_i\sim\mathrm{Unif}(\mathbb S^{d-1}),\quad
r_i\sim\mathrm{Unif}\big[\sqrt d+1,\ \sqrt d+5\big],
\]
with $c=1$ and $m=128$ points on a radial band beyond the Gaussian shell, and
$\overline{\log w}$ the mean of $\log w$ over the reverse-KL batch of the
same step, held fixed in the gradient. The penalty is one-sided: only
weights more than $e$ times the central level are penalized, since heavier
flow tails are harmless for the MH step. It is switched on after $10\%$ of
the training steps, its weight rises linearly to one over the next $10\%$,
and its $2m=256$ target evaluations per step are included in the training
cost.

No symmetry of the target, such as $x\mapsto-x$, is used in training,
since the same symmetry would make a symmetrized NUTS exact and the
comparison unfair.

\paragraph{Control variates and step size.}
The dictionary is $\mathcal K_{K,p}$ with $K=4$ and $p=2$
(Appendix~\ref{app:alg-dictionary}); the selection pass uses
$N_0=10^5$ reference draws and the coefficient pass $N=10^5$ fresh
ones, and $\theta=(1,-1)$ is fixed. In the main comparison $\gamma$ is
chosen on the grid of $25$ values between $0.005$ and $1$, spaced by $0.05$
above $0.1$ and denser below and $L$ on $\{1,2,4,8\}$ by the pilot of the
Protocol. Where $\gamma$ is adapted by $J$, in the chain experiments and
the benchmark, $\log\gamma$ follows stochastic gradient ascent with
learning rate $0.05$ on $[10^{-3},1]$, and the per-chain contributions to
the gradient are winsorized at their $10\%$ and $90\%$ quantiles.

\paragraph{Methods.} Chain-based methods follow the Protocol; methods in $x$-space start from the Laplace approximation.
\begin{itemize}
\item \emph{TGI-RWM, TGI-MALA (+CV)} are the samplers of
Table~\ref{tab:variants}. Since a trained flow cannot satisfy the bound on $w$
of Assumption (A2), each step may be replaced, independently with
probability $\delta$, by an independence MH step with the fixed latent
proposal $\nu=t_{10}(0,I_d)$, accepted with probability
$\min\{1,\pz(z')\nu(z)/(\pz(z)\nu(z'))\}$. Such a step costs one target evaluation
(plus the gradient at the new state after an accepted step for
TGI-MALA), records no control-variate summand and does not enter the
step-size adaptation, so the estimator averages over the TGI steps only
(Section~\ref{sec:estimator}). Columns with $\delta$ use $\delta=0.01$,
all other results use $\delta=0$ (Table~\ref{tab:wallclock-full}).
\item \emph{TESS (+CV)} \citep{cabezas2023transport} is elliptical slice
sampling \citep{murray2010elliptical} on the same fixed flow, with prior
$\gauss(0,I)$ and likelihood $w$; its control variates use the initial
point on the ellipse (Appendix~\ref{app:ref-ess}).
\item \emph{SNIS (+CV)} draws $N$ points from $\refm$, weights them by
$w$ and uses \eqref{eq:snis-cv}, with the coefficients of the numerator
and the denominator obtained by least squares on the weighted draws
\citep{owen2000safe}.
\item \emph{NeuTra (+ZVCV)} \citep{hoffman2019neutra} runs HMC on $\pz$
with a fixed number of leapfrog steps $L$, see the settings in the Protocol. NeuTra+ZV applies the zero-variance control variates
of NUTS+ZV, with the score $\nabla\log\pz$, to the latent draws.
\item \emph{NUTS (+ZVCV)} \citep{hoffman2014no} is the NumPyro
implementation \citep{phan2019composable} with Stan defaults: target
acceptance $0.8$, diagonal mass matrix and maximum tree depth $10$. Each
replicate is the chains started at the mean of the Gaussian approximation
plus $\gauss(0,0.1^2)$ jitter, with warm-up
$\min(1000,\,20\%\text{ of the budget})$ iterations included in the cost. NUTS+ZV adds
second-order zero-variance control variates \citep{mira2013zero} on the
same draws: coordinates are standardized with the Gaussian
approximation, $\phi$ ranges over $y_j$ and $y_jy_k$,
$h=\Delta\phi+\nabla\phi\cdot\nabla\log\target$, and the coefficients are
fitted by least squares with intercept, taking the minimum-norm
solution; the gradients come from the NUTS draws at no extra cost.
\end{itemize}

\paragraph{Reference values.} For the synthetic targets the reference
values are exact, from closed forms or exact sampling. For the real
posteriors they come from long NUTS chains with second-order
zero-variance control variates, and their Monte Carlo standard errors
from batch means; these errors enter the $\dagger$ marks of
Table~\ref{tab:wallclock}. Two posteriors need special treatment. On
eight schools, NUTS runs on the non-centred parametrization, which it
explores far better than the centred one targeted by the flow-based
methods; the two describe the same posterior. On the $\phi^4$ field,
NUTS stays in one of the two symmetric modes, so we run it in one mode
and complete its draws by the symmetry $x\mapsto-x$. The test functions
that are odd under this symmetry then have their exact values: $1/2$ for
$\mathbf 1\{M>0\}$ and $0$ for $M$.

\paragraph{Hardware and software.}
All experiments ran on a single workstation with an Intel Core
i7-13700K CPU (24 threads), 64\,GB of memory and one NVIDIA GeForce RTX
4090 GPU (24\,GB). The flows, all latent estimators, including NeuTra, and the control variates are implemented in PyTorch~2.9 in double precision
and run on the GPU, with all chains of a configuration advanced in one
batch. NUTS runs on the CPU with NumPyro~0.22 on JAX~0.10, one chain per
CPU device. The precision per second depends on this hardware, whereas costs in target evaluations do not. posteriordb and NumPyro are released under the Apache 2.0 licence, PyTorch and JAX under BSD and Apache 2.0, and the German credit data under CC BY 4.0.

\subsection{Training the Map}
\label{app:training-results}

We compare the training choices of Appendix~\ref{app:training} on the
banana, where the objective and the tail penalty matter, and on the
Gaussian mixture, where the handling of several modes does. Every variant
is trained once with the default settings ($10^4$ steps, same seed), except
the 1/10-cost variant, a complete training of $10^3$ steps with every
schedule scaled to its length; the variant equal to the default reuses the
default map. For the last checkpoint we report $\mathrm{ESS}_{\mathrm{IS}}/N$,
the Fisher divergence $\mathcal F$, the $99.9\%$ quantile of $\log w$ on the
radial band of the tail penalty relative to its median under $\refm$ and,
for the mixture, the mass $q(x_1>0)$ that $T_\#\refm$ assigns to one mode
(exact value $1/2$). On that map we run TGI-RWM at the step selected by
$J$, two seeds of $2000$ chains with $500$ burn-in and $5000$ recorded
steps, and report the VRF of the main test function, the bias of the pooled
estimate, and the largest robust z-score of a chain's control-variate
estimate, which reveals the rare chains that dominate a heavy-tailed
variance.

\paragraph{Objective and tail penalty.}
On the banana the tail penalty is decisive (Table~\ref{tab:train-banana}).
Without it, under either scheme, a few of $2000$ chains deviate by about
$10^2$ robust standard deviations and the VRF of the main test function is
$3$--$15$ and differs between seeds; with it, the largest deviation is
$4$--$9$, the VRF is $29$--$52$, and the two seeds agree. The tail quantile
of $\log w$ separates the variants ($5.5$--$7.6$ without the penalty,
$0.6$--$1.0$ with it), whereas $\mathrm{ESS}_{\mathrm{IS}}/N$ is
$0.99$--$1.00$ for all of them. With the penalty, scheme $\schA$ improves
on $\schB$ (VRF $50$ against $31$, $\mathcal F=0.015$ against $0.036$), and
the LARGE network on the SMALL one at the same training cost (VRF $50$
against $40$). The penalty is not universal, however. On the hierarchical posteriors it
degrades the LARGE flows, whose tails are not too light to begin with: on
eight schools the effective sample size per step of $\log\tau$ rises from
$0.02$ to $0.12$ without it, on the $\phi^4$ field $\mathrm{ESS}_{\mathrm{IS}}/N$ from $0.31$ to
$0.76$. We therefore switch it off on these posteriors
(Table~\ref{tab:exp-real}). An alternative to this per-target choice is to let the penalty decide for itself. At the checkpoint where it would start, several samples are drawn on the radial band and the penalty is enabled only if the $99.9\%$ quantile of $\log w$ there exceeds the median of $\log w$ under $\rho$ by more than the margin c; below the margin the penalty is inactive on the band and its gradient only constrains the bulk of the map. The rule costs nothing beyond the several evaluations already made at that step and introduces no new parameter.

\paragraph{Several modes.}
Started from the Laplace approximation of one mode, schemes $\schB$ and
$\schA$ never find the other one: $q(x_1>0)$ stays at $0.96$, the TGI
chains stay in that mode, and the estimate of $x_1$ is off by $0.87$
standard deviations (Table~\ref{tab:train-gmix_d10_h6}). Their
$\mathrm{ESS}_{\mathrm{IS}}/N$ ($0.985$) is higher than that of the correct
map ($0.959$) and their VRF is larger, but both describe one mode only.
$\mathrm{ESS}_{\mathrm{IS}}$ cannot see a mode that the flow does not
cover, so for multimodal targets we check the mass the flow assigns to
each mode instead. The multi-start pilot with whitening of scheme $\schAa$ 
recovers both modes. The anchor keeps $q(x_1>0)$ between $0.44$ and $0.52$
during training, against $0.39$ to $0.69$ without it, because otherwise a
mode with slightly less flow mass is visited less, contributes fewer
states to the buffer and loses more mass; in both cases the final
estimates are unbiased, since the MH step corrects the weights of the
modes. The tail penalty matters here too: without it the tail quantile is
$8.4$ instead of $1.6$, and one chain in $2000$ deviates by $97$ robust
standard deviations, which lowers the VRF of one seed from $3.8$ to $1.4$. More starting points help as well: on the $\phi^4$ field, $64$ starts
instead of $16$ give a pilot whose within-mode covariance is better
estimated, and the effective sample size per step of TGI-RWM roughly
doubles, whereas a wider spread of starting points does not help.

\paragraph{Training budget.}
A step of the default training evaluates $\log\target$ with its gradient
on the $256$ reverse-KL samples (cost $512$) and on the $128$ points of
the tail penalty (cost $256$, from $10\%$ of the steps on), and runs the
$32$ buffer chains of scheme $\schA$ (cost $35.2$ per step on average);
over $10^4$ steps this gives $7.8\times10^6$ target evaluations
($8.2\times10^6$ with the pilot of the mixture), and the real posteriors
use $5.5\times10^6$ to $1.6\times10^7$ and $6.1\times10^8$ for the sparse logistic regression (Table~\ref{tab:exp-real}). With a
tenth of the budget, trained to completion, the banana map remains
usable ($\mathrm{ESS}_{IS}/N=0.95$, no stuck chains, no detectable
bias), but $\mathcal F$ grows twentyfold and the VRF falls from about
$50$ to $6$; the mixture map is poor ($\mathrm{ESS}_{IS}/N=0.49$,
$q(x_1>0)=0.64$, VRF $1.3$) and its estimate shows a small but
significant bias. The training evaluates the target on $256$ points at
once, so on one GPU it takes about three minutes on both targets and
$24$ minutes on the Lotka--Volterra posterior, whose every evaluation
solves an ODE.

\begin{table}
\centering
\caption{Training ablation, banana: every variant trained once for $10^4$ steps, except the 1/10-cost one, a complete training of $10^3$ steps with every schedule scaled to its length; flow diagnostics at the last checkpoint (tail $\log w$: 99.9\% quantile over the radial band relative to the median over $\mathcal N(0,I)$) and TGI-RWM chains at $J$'s $\gamma$ (two seeds of 2000 chains, 500 + 5000 steps; main test function). max $|z|$: the largest $|\hat\mu_r-\mathrm{median}\,\hat\mu|/(1.4826\,\mathrm{MAD})$ over the chains' control-variate estimates $\hat\mu_r$ (about 4 for Gaussian spread; larger values mean that a few chains stuck). bias: $(\bar\mu-\mu)/\sigma_\pi$ of the pooled estimate; $^*$: more than four standard errors $\mathrm{sd}(\hat\mu_r)/\sqrt R$.}
\label{tab:train-banana}
\begin{tabular*}{\textwidth}{@{\extracolsep{\fill}}lccccccc@{}}
\toprule
Variant & train cost & ESS$_{\rm IS}/N$ & $\mathcal F$ & tail $\log w$ & VRF & max $|z|$ & bias \\
\midrule
B & $5.1\times10^{6}$ & 0.997 & 0.0287 & 7.63 & 4.19 / 2.92 & 208 & $-1.0\times10^{-4}$ \\
B + tail & $7.4\times10^{6}$ & 0.993 & 0.0361 & 1 & 33.2 / 29.4 & 4 & $+5.0\times10^{-5}$ \\
A & $5.5\times10^{6}$ & 0.997 & 0.0171 & 5.46 & 8.95 / 14.5 & 96 & $-9.9\times10^{-5}$ \\
A + tail (default) & $7.8\times10^{6}$ & 0.998 & 0.0147 & 0.699 & 51.5 / 48.2 & 5 & $-4.6\times10^{-6}$ \\
A + tail, SMALL & $7.8\times10^{6}$ & 0.996 & 0.0245 & 0.648 & 40.5 / 39.8 & 9 & $-6.5\times10^{-5}$ \\
A + tail, 1/10 cost & $7.8\times10^{5}$ & 0.950 & 0.329 & 1.88 & 5.66 / 5.76 & 8 & $+3.4\times10^{-5}$ \\
\bottomrule
\end{tabular*}
\end{table}

\begin{table}
\centering
\caption{Training ablation, Gaussian mixture: every variant trained once for $10^4$ steps, except the 1/10-cost one, a complete training of $10^3$ steps with every schedule scaled to its length (pilot $8\times10^4$); flow diagnostics at the last checkpoint (tail $\log w$: 99.9\% quantile over the radial band relative to the median over $\mathcal N(0,I)$) and TGI-RWM chains at $J$'s $\gamma$ (two seeds of 2000 chains, 500 + 5000 steps; main test function). max $|z|$: the largest $|\hat\mu_r-\mathrm{median}\,\hat\mu|/(1.4826\,\mathrm{MAD})$ over the chains' control-variate estimates $\hat\mu_r$ (about 4 for Gaussian spread; larger values mean that a few chains stuck). bias: $(\bar\mu-\mu)/\sigma_\pi$ of the pooled estimate; $^*$: more than four standard errors $\mathrm{sd}(\hat\mu_r)/\sqrt R$.}
\label{tab:train-gmix_d10_h6}
\setlength{\tabcolsep}{3pt}
\begin{tabular*}{\textwidth}{@{\extracolsep{\fill}}lcccccccc@{}}
\toprule
Variant & train cost & ESS$_{\rm IS}/N$ & $\mathcal F$ & tail $\log w$ & $q(x_1>0)$ & VRF & max $|z|$ & bias \\
\midrule
A & $7.8\times10^{6}$ & 0.986 & 101 & 1.07 & 0.96 & 23.5 / 23.6 & 7 & +0.871$^*$ \\
A$^a$ without anchor & $8.2\times10^{6}$ & 0.920 & 388 & 1.7 & 0.56 & 3.59 / 3.1 & 18 & $+6.5\times10^{-5}$ \\
A$^a$ (default) & $8.2\times10^{6}$ & 0.959 & $6.3\times10^{3}$ & 1.63 & 0.52 & 3.59 / 3.82 & 10 & $-5.1\times10^{-4}$ \\
A$^a$ without tail & $5.9\times10^{6}$ & 0.932 & $5.6\times10^{3}$ & 8.37 & 0.51 & 3.2 / 1.42 & 97 & $+1.6\times10^{-4}$ \\
A$^a$, 1/10 cost & $8.6\times10^{5}$ & 0.489 & 196 & 4.62 & 0.64 & 1.37 / 1.26 & 38 & $+3.8\times10^{-3}$$^*$ \\
B & $7.4\times10^{6}$ & 0.985 & 85.3 & 1.12 & 0.96 & 9.81 / 8.52 & 16 & +0.871$^*$ \\
\bottomrule
\end{tabular*}
\end{table}

\newpage
\subsection{Synthetic Targets}
\label{app:exp-synthetic}

\subsubsection{Targets}
The banana, the funnel and the Gaussian mixture have, besides the trained
flow, an analytic family of transports $T_\vartheta$ whose accuracy is set
by one parameter, with the exact transport as a member; we measure the
accuracy by the Fisher divergence $\Fisher=\E_{\pz}\|\nabla\log w\|^2$.
\begin{itemize}
\item \emph{Banana} \citep{haario1999adaptive}, $d=8$, $p=10$, $b=0.1$:
$\log\target(x)=-x_1^2/(2p)-(x_2+bx_1^2-bp)^2/2-\sum_{k\ge3}x_k^2/2$; test
functions $x_2$ (main), $x_2^2$, $x_1^2$, the mean of the even coordinates
and $\bar x$. Shear family $T_c(z)_1=\sqrt p\,z_1$,
$T_c(z)_2=z_2+bp-(bp-c)z_1^2$, $T_c(z)_k=z_k$: a banana of curvature $c$
in the latent space, $\Fisher=7c^2$, exact at $c=0$, the Laplace affine
map at $c=bp$.
\item \emph{Funnel} \citep{neal2003slice}, $d=10$, $\sigma_v=2$:
$v\sim\gauss(0,\sigma_v^2)$, $x_i\mid v\sim\gauss(0,e^v)$; test functions
$v$ (main), $\tfrac19\sum_i\log(1+x_i^2)$, $v^2$ and $x_1$. Scale family
$v=\sigma_vz_1$, $x_i=e^{(1-\varrho/\sigma_v)v/2}z_i$: a funnel of scale
$\varrho$ in the latent space, exact at $\varrho=0$.
\item \emph{Gaussian mixture} \citep{alexopoulos2023variance}, $d=10$:
$\tfrac12\gauss(m,\Sigma)+\tfrac12\gauss(-m,\Sigma)$, $m=(3,0,\dots,0)$,
$\Sigma$ inverse-Wishart rescaled to largest eigenvalue $25$; test
functions $x_1$ (main), $x_1^2$, $\mathbf 1\{x_1>0\}$ and $\bar x$. Scale
family $T_s(z)=T(sz)$ with $T$ the exact transport: $\pz=\gauss(0,s^{-2}I)$,
$\Fisher=d(1-s^2)^2/s^2$, exact at $s=1$.
\item \emph{$t$ mixture}, $d=10$: the Gaussian mixture with multivariate
$t_5$ components,
$\tfrac12 t_5(m,S)+\tfrac12 t_5(-m,S)$ with $S=\tfrac35\Sigma$, so that
the means and covariances are those of the Gaussian mixture and only the
tails are heavier (polynomial); the same test functions, $x_1$ (main).
It has no analytic family; its flow, whose tails are Gaussian, has
unbounded importance weights.
\end{itemize}

\begin{table}[ht]
\centering
\caption{Settings for the synthetic targets. All flows use the LARGE
network, $10^4$ training steps, batch size $256$ and the tail penalty;
the multi-start pilot of the mixtures costs $4\times10^5$ and is included
in the training cost. $\max w\cdot N$ is the largest normalized
importance weight of the flow times $N=10^6$, the largest over $100$
replicates. $L$ is the number of leapfrog steps of NeuTra and $\gamma$
the step size of TGI-RWM / TGI-MALA, both chosen by pilot runs.}
\label{tab:exp-synthetic}
\small
\begin{tabular}{@{}lclcclcc@{}}
\toprule
Target & $d$ & Scheme & Train.\ cost & $\max w{\cdot}N$ & Main $F$ & $L$ & $\gamma$ \\
\midrule
Banana           & 8  & $\schA$  & $7.8\times10^6$ & $65$            & $x_2$ & 1 & 0.90 / 0.20 \\
Funnel           & 10 & $\schA$  & $7.8\times10^6$ & $1.6\times10^3$ & $v$   & 2 & 1.00 / 0.10 \\
Gaussian mixture & 10 & $\schAa$ & $8.2\times10^6$ & $1.2\times10^3$ & $x_1$ & 2 & 0.95 / 0.20 \\
$t$ mixture      & 10 & $\schAa$ & $8.2\times10^6$ & $1.3\times10^5$ & $x_1$ & 2 & 1.00 / 0.10 \\
\bottomrule
\end{tabular}
\end{table}

\subsubsection{Results}
On the analytic family of the banana (Figure~\ref{fig:banana-family}),
TGI-RWM is the more efficient sampler per unit cost across the whole
family, and both lose efficiency once the Fisher divergence exceeds
about $10^{-1}$, together with $\mathrm{ESS}_{IS}/N$, which therefore
tells at run time which regime applies. On the accurate member the ESS
grows with $\gamma$ and the VRF falls, so their product has an interior
maximum. For TGI-RWM it is flat near $\gamma=1$, and the step chosen by
$J$ loses little; for TGI-MALA, whose drift misdirects large steps, it
peaks near $\gamma=0.35$, and $J$ recovers most of what the speed
measure, which would choose $\gamma=1$, loses. Magnitude pruning with
the threshold $2\log|\mathcal K|$ keeps a VRF of about $2000$ with a
handful of terms at every maximal degree, whereas the full dictionary
loses most of it beyond $K=2$ as its estimated coefficients become
noisy.

When the training of the flow is charged to the flow-based methods
(Figure~\ref{fig:synthetic-bench}), they pass NUTS on the banana at a
total cost of a few times $10^5$, and at $10^7$ SNIS and TGI-RWM with the
control variates are about fifty times more accurate, while NeuTra with
zero-variance control variates levels off an order of magnitude above
them. The right panels show the flow straightening the banana within the
first thousand training steps, by which point $\mathrm{ESS}_{IS}/N$ has
risen from $0.06$ to $0.90$.

\begin{figure}
\centering
\includegraphics[width=\textwidth]{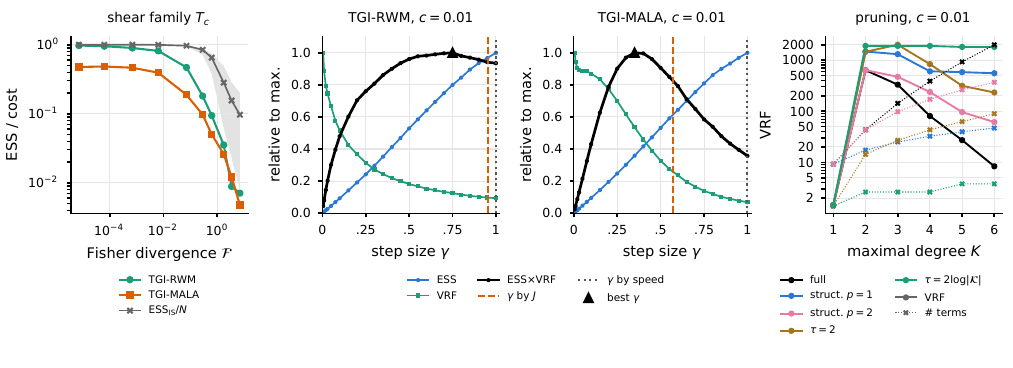}
\caption{Analytic transports of the banana, main test function $x_2$,
from the chain experiments of Appendix~\ref{app:experiments}. From left
to right: the best ESS per unit cost over the step size of TGI-RWM and
TGI-MALA, and $\mathrm{ESS}_{IS}/N$, against the Fisher divergence of the
shear map $T_c$; the ESS, the VRF and their product against the step
size $\gamma$, relative to their maxima, for TGI-RWM and TGI-MALA on the
accurate member $c=0.01$ (dashed: the step chosen by $J$; dotted: by the
speed measure); and the VRF (solid) and the number of kept Hermite terms
(dotted) against the maximal degree $K$ for each pruning rule, with
TGI-MALA on the same member.}
\label{fig:banana-family}
\end{figure}

\begin{figure}
\centering
\includegraphics[width=\textwidth]{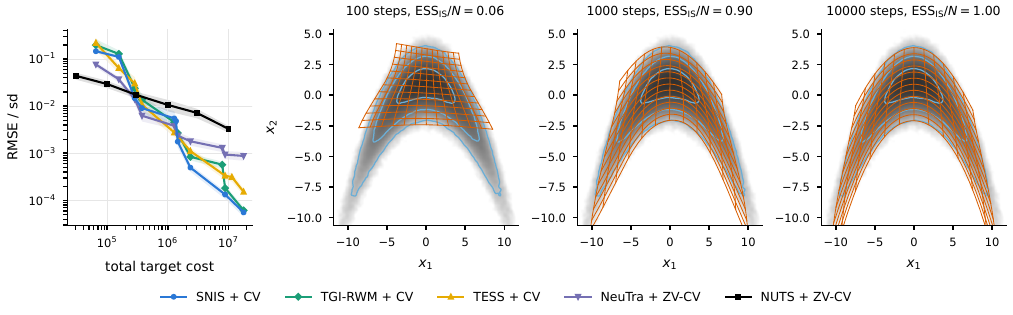}
\caption{The banana with the training of the flow charged. Left: RMSE of
the main test function in posterior standard deviations against the
total number of target evaluations, with bootstrap $95\%$ bands; each
flow-based method shows the lower envelope over training checkpoints and
sampling budgets. Right: a grid of the latent space pushed forward by
the flow after $100$, $1000$ and $10^4$ training steps, over the target
density (grey; contours hold $50$, $90$ and $99\%$ of the mass).}
\label{fig:synthetic-bench}
\end{figure}

\newpage
\subsection{Real Posteriors}
\label{app:exp-real}

Table~\ref{tab:exp-real} gives the scheme, network, tail penalty,
training cost, and selected step sizes of every
posterior. The tail penalty helps the SMALL networks on the near-Gaussian
posteriors and degrades the LARGE flows of the hierarchical and
multimodal ones, where it is switched off. We keep
the centered parametrizations as the models are stated: non-centering
removes the funnel for NUTS when the hierarchy is known, but no such
change exists for the ridge of the one-compartment model, and the
Gaussian-process regression is already non-centered.

\begin{table}[ht]
\centering
\caption{Settings for the real posteriors: dimension, training scheme,
network, tail penalty, training cost in target evaluations, importance
sampling effective sample size of the flow, main test function, leapfrog
steps $L$ selected for NeuTra and step size $\gamma$ of TGI-RWM, both
chosen by pilot runs. All flows are trained for $10^4$ steps with batch
size $256$, except eight schools ($3\times10^4$ steps) and the sparse
logistic regression ($3\times10^5$ steps, batch size $1024$, a LARGE
network of width $256$, marked $^*$); for the $\phi^4$ field the
multi-start pilot uses $64$ NUTS chains and $4\times10^6$ evaluations,
included in the training cost.}
\label{tab:exp-real}
\small
\begin{tabular}{@{}lclllcclcc@{}}
\toprule
Posterior & $d$ & Scheme & Network & Tail & Train.\ cost & $\max w{\cdot}N$ & Main $F$ & $L$ & $\gamma$ \\
\midrule
GARCH(1,1)         & 4  & $\schA$  & SMALL   & yes & $7.8\times10^6$ & $3$ & $\beta_1$          & 1 & 0.90 \\
One-compartment PK & 4  & $\schA$  & SMALL   & yes & $7.8\times10^6$ & $9$ & $\log K_m$         & 1 & 0.99 \\
Dugongs            & 4  & $\schB$  & SMALL   & yes & $7.4\times10^6$ & $3.5\times10^2$ & $\beta$            & 1 & 0.80 \\
Lotka--Volterra    & 8  & $\schA$  & SMALL   & yes & $7.8\times10^6$ & $42$ & $\alpha$           & 1 & 1.00 \\
GP Poisson         & 13 & $\schA$  & LARGE   & no  & $5.5\times10^6$ & $7.7\times10^3$ & $\alpha$           & 2 & 0.99 \\
Eight schools      & 10 & $\schA$  & LARGE   & no  & $1.6\times10^7$ & $6.8\times10^3$ & $\theta_1$         & 1 & 0.90 \\
Sparse logistic    & 51 & $\schB$  & LARGE$^*$ & no & $6.1\times10^8$ & $7.8\times10^4$ & $w_{24}$          & 2 & 0.80 \\
$\phi^4$ field     & 32 & $\schAa$ & LARGE   & no  & $9.5\times10^6$ & $2.9\times10^3$ & $\mathbf 1\{M>0\}$ & 1 & 1.00 \\
\bottomrule
\end{tabular}
\end{table}

\paragraph{Targets.}
\begin{itemize}
\item \emph{GARCH(1,1)}: the posteriordb model $y_t\sim\gauss(\mu,\sigma_t^2)$,
$\sigma_t^2=\alpha_0+\alpha_1(y_{t-1}-\mu)^2+\beta_1\sigma_{t-1}^2$, with
$\alpha_0>0$, $\alpha_1,\beta_1\ge0$, $\alpha_1+\beta_1<1$; test functions
$\mu,\alpha_0,\alpha_1,\beta_1$. Nearly Gaussian with a right skew in
$\alpha_0$.
\item \emph{Dugongs}: the nonlinear growth curve
$Y_i\sim\gauss(\alpha-\beta\lambda^{x_i},\tau^{-1})$ for $27$ dugongs,
$\lambda\in(0.5,1)$, priors as in posteriordb, sampled as
$\alpha,\beta,\mathrm{logit}(2\lambda-1),\log\tau$; test functions
$\alpha,\beta,\lambda,\sigma=\tau^{-1/2}$. Skewed, nearly Gaussian.
\item \emph{One-compartment PK}: the posteriordb posterior
\texttt{one\_comp\_mm\_elim\_abs}, first-order absorption with
Michaelis--Menten elimination, parameters $k_a,K_m,V_m,\sigma$ on the log
scale, ODE solved by RK4 with step $0.02$. The posterior of
$(\log K_m,\log V_m)$ is L-shaped with a ridge along which $K_m$ has no
finite variance, so the test functions are $k_a,\log K_m,\log V_m,\sigma$.
\item \emph{Lotka--Volterra}: hare and lynx pelt counts, lognormal around
the solution of $\dot u=(\alpha-\beta v)u$, $\dot v=(-\gamma+\delta u)v$,
priors as in \citet{carpenter2018predator}; four rates, two initial
populations and two noise scales; test functions are $\alpha$, $\beta$, $\gamma$, $\delta$. Every evaluation solves the ODE.
\item \emph{GP Poisson}: counts $k_i\sim\mathrm{Poisson}(e^{f(t_i)})$ at
$11$ points with a squared-exponential Gaussian-process prior on $f$,
non-centered as in posteriordb; test functions $\ell,\alpha,f_1,f_6$.
\item \emph{Eight schools}: the centered parametrization of
\citet{rubin1981estimation}, coordinates $(\mu,\log\tau,\theta_{1:8})$,
with a funnel between $\tau$ and the $\theta_j$; test functions
$\mu,\log\tau,\theta_1$.
\item \emph{Sparse logistic regression}: the German credit model of
\citet{hoffman2019neutra}, coordinates $(\log\tau,\log\lambda,\beta)$,
$25$ small funnels between the $\lambda_j$ and $\beta_j$; test functions
$\log\tau$ and the effective coefficients $w_j=\tau\lambda_j\beta_j$.
\item \emph{$\phi^4$ field}: the one-dimensional lattice field of
\citet{gabrie2022adaptive}, $N=32$ sites, $a=0.1$, $\beta=20$, bimodal
with equal weights by the symmetry $\phi\mapsto-\phi$, which no method
is allowed to exploit; test functions $\mathbf 1\{M>0\},M,M^2,U/N$.
\end{itemize}

\paragraph{Results.}

\begin{figure}[t]
\centering
\includegraphics[width=\textwidth]{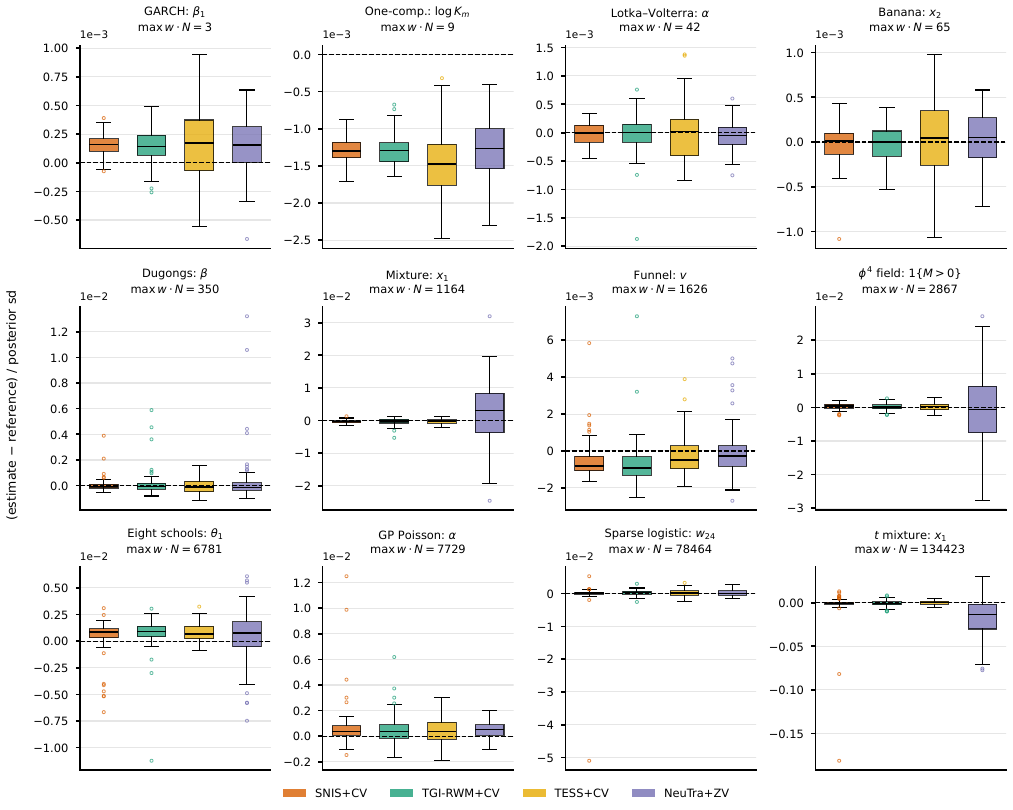}
\caption{The $100$ estimates of the main test function of every target,
in posterior standard deviations from the reference value (dashed line).
Panels are ordered by $\max w\cdot N$; circles are estimates beyond
$1.5$ interquartile ranges from the box.}
\label{fig:boxplots}
\end{figure}

Table~\ref{tab:wallclock-full} extends Table~\ref{tab:wallclock} to every
test function, and the picture of Section~\ref{sec:experiments} holds
function by function. On the targets with light weights SNIS+CV is the
most efficient on nearly every function, and the control variates gain a
factor of up to about $160$; on the targets with heavy weights the gain
falls to a few, the gradient-based chains lead on most functions of the
Gaussian-process and sparse regressions, and TESS+CV leads on the $t$
mixture. With the fixed coefficients $\theta=(1,-1)$ the control
variates can also hurt on a coarse flow, as for TGI-MALA on dugongs.
At equal numbers of target evaluations, NUTS with zero-variance control
variates is less accurate than the best flow-based
estimator on every real posterior and cannot estimate the mode weight
of the $\phi^4$ field (Table~\ref{tab:nuts}).

Figure~\ref{fig:boxplots} shows the $100$ estimates behind the main test
function of every target. On the targets with light weights the four
estimators are unbiased and their spread follows the precision in the
table. As $\max w\cdot N$ grows, SNIS+CV produces isolated estimates far
from the reference, on the sparse regression and the $t$ mixture, while
the chains do not. NeuTra+ZV spreads widely on the three multimodal
targets and is biased on the $t$ mixture, where its chains rarely move
between the modes. On the one-compartment model all four estimators
share an offset of about $10^{-3}$ posterior standard deviations, of the
size of the Monte Carlo error of the reference value, which is therefore
the error of the reference rather than a bias of the estimators.

\begin{table}[t]
\centering
\small
\caption{RMSE in posterior standard deviations of the main test function
at $B=10^6$ target evaluations, without the training of the flow, over
$100$ estimates for the flow-based estimators and $50$ for NUTS. NUTS+ZV
was more accurate than plain NUTS on every posterior where both were
run. $\dagger$: extrapolated from the reference run as its Monte Carlo
standard error times $\sqrt{C_{\rm ref}/B}$; $\ddagger$: NUTS stays in
one mode. Entries within $10\%$ of the best in a row are in bold.}
\label{tab:nuts}
\begin{tabular}{@{}ll ccccc@{}}
\toprule
Posterior & $F$ & NUTS+ZV & SNIS+CV & TGI-RWM+CV & TESS+CV & NeuTra+ZV \\
\midrule
GARCH & $\beta_1$ & 2.1e-3 & \textbf{1.9e-4} & \textbf{2.0e-4} & 3.5e-4 & 2.9e-4 \\
One-compartment & $\log K_m$ & 1.3e-2 & \textbf{1.3e-3} & \textbf{1.3e-3} & 1.5e-3 & \textbf{1.3e-3} \\
Lotka--Volterra & $\alpha$ & 1.6e-3 & \textbf{1.9e-4} & 3.1e-4 & 4.7e-4 & 2.7e-4 \\
Dugongs & $\beta$ & 5.0e-3 & \textbf{5.0e-4} & 8.5e-4 & 5.8e-4 & 2.0e-3 \\
GP Poisson & $\alpha$ & 1.5e-2 & 1.8e-3 & 1.1e-3 & 1.0e-3 & \textbf{7.5e-4} \\
Eight schools & $\theta_1$ & 1.4e-2 & 1.6e-3 & 1.4e-3 & \textbf{1.2e-3} & 2.1e-3 \\
Sparse logistic & $w_{24}$ & 1.6e-2$^\dagger$ & 5.2e-3 & \textbf{8.4e-4} & 1.1e-3 & \textbf{8.2e-4} \\
$\phi^4$ field & $\mathbf 1\{M>0\}$ & --$^\ddagger$ & \textbf{8.8e-4} & 1.1e-3 & 1.1e-3 & 9.2e-3 \\
\bottomrule
\end{tabular}
\end{table}

\begin{table}[p]
\centering
\footnotesize
\setlength{\tabcolsep}{3pt}
\renewcommand{\arraystretch}{0.93}
\caption{Precision per second on every test function, extending
Table~\ref{tab:wallclock}; $\star$ marks the main test function. For the
sparse logistic regression we give $\log\tau$, $w_1$ and the median,
maximum and minimum over the $25$ coefficients $w_j$. Entries within
$10\%$ of the best in a row are in bold. $\dagger$: the Monte Carlo
error of the reference value accounts for at least half of the mean
squared error; $\ddagger$: the errors of this target are limited by a
bias of about $0.01$ posterior standard deviations in $\log\tau$ shared
by the flow-based estimators; --: RMSE above one posterior standard
deviation.}
\label{tab:wallclock-full}
\begin{tabular}{@{}llrrrrrrrrrrrr@{}}
\toprule
 & & \multicolumn{2}{c}{SNIS} & \multicolumn{2}{c}{TGI-RWM} & \multicolumn{2}{c}{TGI-RWM$^\delta$} & \multicolumn{2}{c}{TGI-MALA} & \multicolumn{2}{c}{TESS} & \multicolumn{2}{c}{NeuTra} \\
\cmidrule(lr){3-4}\cmidrule(lr){5-6}\cmidrule(lr){7-8}\cmidrule(lr){9-10}\cmidrule(lr){11-12}\cmidrule(lr){13-14}
Target & $F$ & plain & +CV & plain & +CV & plain & +CV & plain & +CV & plain & +CV & plain & +ZV \\
\midrule
GARCH & $\beta_1$\,$\star$ & 3.3e6 & \textbf{9.7e7$^\dagger$} & 2.3e6 & 8.4e7$^\dagger$ & 1.6e6 & 6.3e7$^\dagger$ & 3.4e5 & 2.2e7 & 2.3e6 & 2.4e7 & 4.6e5 & 5.0e7 \\
 & $\mu$ & 3.3e6 & \textbf{5.3e8} & 2.3e6 & 2.6e8 & 1.6e6 & 1.8e8 & 3.4e5 & 3.6e7 & 2.5e6 & 3.6e7 & 4.4e5 & \textbf{5.7e8} \\
 & $\alpha_0$ & 3.1e6 & \textbf{3.2e7$^\dagger$} & 2.2e6 & \textbf{3.2e7$^\dagger$} & 1.5e6 & 2.6e7$^\dagger$ & 3.5e5 & 1.4e7 & 2.1e6 & 1.7e7 & 4.7e5 & \textbf{3.3e7} \\
 & $\alpha_1$ & 3.4e6 & \textbf{4.8e8$^\dagger$} & 2.3e6 & 2.4e8$^\dagger$ & 1.6e6 & 1.5e8$^\dagger$ & 3.4e5 & 2.8e7 & 2.3e6 & 3.2e7 & 4.6e5 & 8.1e7 \\
\addlinespace[1.5pt]
One-comp. & $\log K_m$\,$\star$ & 5.7e5 & \textbf{8.9e5} & 4.0e5 & 6.7e5 & 3.9e5 & 6.4e5 & 4.0e4 & 4.6e5 & 2.9e5 & 4.7e5 & 9.9e4 & 5.8e5 \\
 & $k_a$ & 1.4e6 & \textbf{2.5e7} & 8.8e5 & 1.8e7 & 8.9e5 & 1.5e7 & 4.0e4 & 5.0e6 & 8.0e5 & 6.3e6 & 1.1e5 & 1.3e7 \\
 & $\log V_m$ & \textbf{1.1e5} & \textbf{1.2e5} & 8.5e4 & 9.1e4 & 8.5e4 & 9.1e4 & 3.0e4 & 6.9e4 & 7.0e4 & 8.1e4 & 5.7e4 & 8.1e4 \\
 & $\sigma$ & 1.4e6 & \textbf{9.6e6} & 8.9e5 & 7.4e6 & 7.9e5 & 7.3e6 & 4.0e4 & 3.5e6 & 7.7e5 & 4.5e6 & 1.1e5 & 6.5e6 \\
\addlinespace[1.5pt]
Lotka--Volterra & $\alpha$\,$\star$ & 3.7e6 & \textbf{1.0e8} & 1.6e6 & 2.3e7 & 1.5e6 & 3.3e7 & 9.8e4 & 5.6e6 & 1.4e6 & 8.9e6 & 1.2e5 & 2.1e7 \\
 & $\beta$ & 3.7e6 & \textbf{7.7e7} & 1.7e6 & 2.2e7 & 1.5e6 & 2.8e7 & 1.0e5 & 5.4e6 & 1.3e6 & 7.9e6 & 1.2e5 & 1.6e7 \\
 & $\gamma$ & 3.6e6 & \textbf{9.0e7} & 1.5e6 & 9.8e6 & 1.6e6 & 3.0e7 & 9.9e4 & 4.9e6 & 1.4e6 & 8.7e6 & 1.2e5 & 1.9e7 \\
 & $\delta$ & 3.7e6 & \textbf{8.7e7} & 1.5e6 & 1.0e7 & 1.5e6 & 2.8e7 & 9.9e4 & 5.2e6 & 1.4e6 & 8.6e6 & 1.2e5 & 2.0e7 \\
\addlinespace[1.5pt]
Banana & $x_2$\,$\star$ & 1.6e6 & \textbf{4.1e7} & 9.8e5 & 3.4e7 & 9.8e5 & 3.4e7 & 9.4e4 & 1.1e7 & 4.7e5 & 5.9e6 & 1.7e5 & 1.7e7 \\
 & $x_2^2$ & 8.7e5 & 1.8e6 & 9.0e5 & \textbf{4.7e6} & 8.8e5 & \textbf{4.5e6} & 1.6e5 & 1.9e6 & 4.1e5 & 1.7e6 & 2.8e5 & 6.1e5 \\
 & $x_1^2$ & 1.6e6 & \textbf{3.0e7} & 1.0e6 & 2.6e7 & 1.0e6 & 2.6e7 & 1.3e5 & 7.7e6 & 3.4e5 & 4.3e6 & 2.3e5 & 1.1e7 \\
 & $\tfrac14\sum_k x_{2k}$ & 1.6e6 & \textbf{7.5e7} & 9.6e5 & 5.9e7 & 9.5e5 & 5.8e7 & 7.9e4 & 1.9e7 & 6.3e5 & 9.1e6 & 1.4e5 & 3.4e7 \\
 & $\bar x$ & 1.6e6 & \textbf{1.1e8} & 9.2e5 & 7.6e7 & 9.2e5 & 7.3e7 & 7.0e4 & 1.6e7 & 8.2e5 & 1.2e7 & 1.3e5 & 4.9e7 \\
\addlinespace[1.5pt]
Dugongs & $\beta$\,$\star$ & 3.9e6 & \textbf{2.0e7} & 1.6e6 & 5.4e6 & 1.9e6 & 1.4e7 & 7.8e4 & 5.1e4 & 1.9e6 & 9.2e6 & 4.1e5 & 1.3e6 \\
 & $\alpha$ & 3.5e6 & \textbf{9.5e6} & 1.3e6 & 3.1e6 & 1.7e6 & 7.6e6 & 1.1e5 & 1.8e5 & 1.7e6 & 6.6e6 & 2.9e5 & 9.7e5 \\
 & $\lambda$ & 4.4e6 & \textbf{2.8e7} & 2.0e6 & 1.5e7 & 2.0e6 & 2.1e7 & 5.1e5 & 2.7e4 & 1.8e6 & 8.5e6 & 5.5e5 & 4.5e6 \\
 & $\sigma$ & 4.7e6 & \textbf{9.3e7} & 2.1e6 & 5.0e7 & 2.1e6 & 7.1e7 & 2.1e5 & 8.7e5 & 2.2e6 & 1.9e7 & 5.5e5 & 5.2e6 \\
\addlinespace[1.5pt]
Mixture & $x_1$\,$\star$ & 1.2e6 & \textbf{3.9e6} & 6.3e5 & 1.5e6 & 7.3e5 & 2.4e6 & 1.3e4 & -- & 6.6e5 & 1.6e6 & 8.5e3 & 1.6e4 \\
 & $x_1^2$ & 1.3e6 & \textbf{4.5e6} & 7.1e5 & 1.8e6 & 8.0e5 & 2.4e6 & 1.3e4 & -- & 2.8e5 & 1.0e6 & 1.4e4 & 2.2e4 \\
 & $\mathbf 1\{x_1>0\}$ & 1.2e6 & \textbf{3.3e6} & 7.0e5 & 1.6e6 & 7.3e5 & 1.9e6 & 1.5e4 & -- & 6.7e5 & 1.4e6 & 1.0e4 & 1.3e4 \\
 & $\bar x$ & 1.3e6 & \textbf{7.4e6} & 6.6e5 & 2.3e6 & 7.1e5 & 3.8e6 & 1.2e4 & -- & 6.7e5 & 2.1e6 & 8.0e3 & 4.7e4 \\
\addlinespace[1.5pt]
Funnel & $v$\,$\star$ & 7.2e5 & \textbf{1.4e6} & 3.9e5 & 6.6e5 & 2.8e5 & 4.1e5 & 3.0e4 & \textbf{1.4e6} & 5.4e5 & 1.1e6 & 1.4e5 & 8.5e5 \\
 & $\overline{\log(1+x_i^2)}$ & 3.5e5 & 4.5e5 & 1.7e5 & 2.1e5 & 1.1e5 & 1.3e5 & 3.2e4 & \textbf{5.8e5} & 2.7e5 & 4.1e5 & 1.0e5 & 2.8e5 \\
 & $v^2$ & 2.0e5 & 2.3e5 & 9.9e4 & 1.1e5 & 5.9e4 & 6.4e4 & 4.7e4 & \textbf{3.4e5} & 1.3e5 & 2.1e5 & 7.0e4 & 1.5e5 \\
 & $x_1$ & 3.1e5 & \textbf{3.8e5} & 3.8e4 & 4.0e4 & 2.1e4 & 2.2e4 & 4.2e4 & 2.9e5 & 1.9e5 & 2.9e5 & 2.2e4 & 8.1e4 \\
\addlinespace[1.5pt]
$\phi^4$ field & $\mathbf 1\{M>0\}$\,$\star$ & 7.1e5 & \textbf{1.4e6} & 4.5e5 & 8.4e5 & 4.5e5 & 8.7e5 & 9.5e3 & -- & 4.4e5 & 8.0e5 & 4.1e3 & 1.0e4 \\
 & $M$ & 7.5e5 & \textbf{1.6e6} & 4.8e5 & 9.4e5 & 4.7e5 & 9.4e5 & 1.0e4 & -- & 4.5e5 & 8.2e5 & 4.2e3 & 1.0e4 \\
 & $M^2$ & 4.1e5 & \textbf{5.3e5} & 2.0e5 & 2.6e5 & 2.8e5 & 3.8e5 & 7.4e3 & -- & 1.7e5 & 3.1e5 & 3.8e3 & 1.0e4 \\
 & $U/N$ & 3.4e5 & \textbf{4.3e5} & 1.7e5 & 2.1e5 & 2.0e5 & 3.1e5 & 6.4e3 & -- & 1.4e5 & 3.2e5 & 1.1e3 & 1.5e4 \\
\addlinespace[1.5pt]
Eight schools & $\theta_1$\,$\star$ & 4.4e5 & 6.2e5 & 3.8e5 & 6.4e5 & 3.9e5 & 7.3e5 & 6.7e4 & \textbf{1.1e6} & 4.7e5 & 8.7e5 & 7.6e4 & 3.4e5 \\
 & $\mu$ & 2.7e5 & 3.4e5 & 2.9e5 & 4.5e5 & 2.7e5 & 4.0e5 & 5.7e4 & 7.5e5 & 5.0e5 & \textbf{9.8e5} & 5.2e4 & 1.6e5 \\
 & $\log\tau$$^\ddagger$ & 9.8e3 & 9.7e3 & 7.2e3 & 7.3e3 & 7.3e3 & 7.4e3 & 1.1e4 & \textbf{1.2e4} & 1.1e4 & 1.1e4 & 9.1e3 & \textbf{1.3e4} \\
\addlinespace[1.5pt]
GP Poisson & $\alpha$\,$\star$ & 3.2e5 & 4.2e5 & 4.6e5 & 1.0e6 & 4.5e5 & 8.9e5 & 5.2e4 & 1.1e6 & 4.9e5 & 1.1e6 & 1.1e5 & \textbf{2.4e6} \\
 & $\ell$ & 1.1e5 & 1.3e5 & 3.0e5 & 4.2e5 & 2.2e5 & 3.0e5 & 4.8e4 & 4.8e5 & 3.6e5 & \textbf{5.7e5} & 9.9e4 & \textbf{5.9e5} \\
 & $f_1$ & 5.9e5 & 1.1e6 & 6.8e5 & 3.1e6 & 6.9e5 & 3.9e6 & 5.6e4 & 3.7e6 & 6.2e5 & 2.1e6 & 1.3e5 & \textbf{2.2e7} \\
 & $f_6$ & 2.2e5 & 2.6e5 & 6.3e5 & 2.4e6 & 5.6e5 & 1.6e6 & 5.6e4 & 1.9e6 & 5.8e5 & 1.6e6 & 1.3e5 & \textbf{6.2e6} \\
\addlinespace[1.5pt]
Sparse logistic & $w_{24}$\,$\star$ & 1.1e4 & 1.1e4 & 1.2e5 & \textbf{4.5e5} & 1.1e5 & 3.0e5 & 1.6e4 & 2.1e5 & 7.6e4 & 2.2e5 & 1.4e4 & \textbf{4.2e5} \\
 & $\log\tau$ & 5.4e4 & 6.4e4 & 1.1e5 & 4.1e5 & 9.8e4 & 3.5e5 & 1.3e4 & 2.0e5 & 1.2e5 & 2.6e5 & 1.1e4 & \textbf{7.9e5} \\
 & $w_{1}$ & 2.1e5 & 6.6e5 & 1.3e5 & 6.5e5 & 1.2e5 & 4.5e5 & 1.3e4 & 2.4e5 & 1.3e5 & 3.3e5 & 1.2e4 & \textbf{5.2e6} \\
 & median over $w_j$ & 1.4e4 & 1.4e4 & 1.2e5 & \textbf{4.6e5} & 1.0e5 & 3.0e5 & 1.3e4 & 2.1e5 & 1.1e5 & 2.7e5 & 1.2e4 & \textbf{4.2e5} \\
 & max over $w_j$ & 2.3e5 & 9.6e5 & 1.4e5 & 7.3e5 & 1.3e5 & 5.1e5 & 1.6e4 & 2.8e5 & 1.3e5 & 3.3e5 & 1.4e4 & \textbf{5.2e6} \\
 & min over $w_j$ & 4.9e2 & 4.7e2 & 5.7e4 & 9.5e4 & 5.2e4 & 7.5e4 & 1.1e4 & 5.8e4 & 6.7e4 & 1.0e5 & 1.1e4 & \textbf{1.8e5} \\
\addlinespace[1.5pt]
$t$ mixture & $x_1$\,$\star$ & 4.1e3 & 4.1e3 & 9.0e4 & 9.8e4 & 1.0e5 & 1.2e5 & 7.5e3 & -- & 1.8e5 & \textbf{2.2e5} & 6.0e2 & 2.2e3 \\
 & $x_1^2$ & 6.8e3 & 6.7e3 & 1.2e4 & 1.2e4 & 1.3e4 & 1.3e4 & 3.0e3 & -- & \textbf{2.7e4} & \textbf{2.7e4} & 3.9e2 & 1.3e3 \\
 & $\mathbf 1\{x_1>0\}$ & 5.8e3 & 5.7e3 & 1.8e5 & 2.2e5 & 2.0e5 & 2.5e5 & 1.1e4 & -- & 2.8e5 & \textbf{4.0e5} & 1.2e3 & 3.3e3 \\
 & $\bar x$ & 3.3e3 & 3.2e3 & 4.1e4 & 4.3e4 & 4.6e4 & 4.9e4 & 4.7e3 & -- & 1.1e5 & \textbf{1.2e5} & 4.1e2 & 1.7e3 \\
\bottomrule
\end{tabular}
\end{table}

\paragraph{Stability.}

\begin{table*}[t]
\centering
\small
\caption{Stability of the estimators: the gap $\log_{10}(\text{precision}/\text{best of the row})$ of each method with its control variates to the most precise one, over all test functions of Table~\ref{tab:wallclock-full} ($0$: the best; $-1$: ten times less precise). Mean gap, the corresponding factor $10^{-\text{mean}}$, standard deviation and worst row; number of rows within a factor $2$ of the best and within $10\%$ of it (the bold rule); failures (RMSE above one posterior standard deviation) are counted at RMSE $=1$. (A) The main test function of each target. (B) All 70 test functions; the mean and standard deviation are over the $12$ per-target means, so that every target has the same weight, the worst row and the counts over all rows.}
\label{tab:stability}
\begin{tabular}{@{}l rrr l rrr@{}}
\toprule
Method & mean gap & factor & sd & worst row & within $2\times$ & best & failed \\
\midrule
\multicolumn{8}{@{}l}{\emph{(A) main test function, $12$ functions}} \\
SNIS & $-0.36$ & 2.3 & 0.65 & $-1.74$ ($t$ mixture, $x_1$) & \textbf{9/12} & \textbf{8/12} & 0 \\
TGI-RWM & {\boldmath$-0.28$} & \textbf{1.9} & \textbf{0.20} & {\boldmath$-0.65$} (Lotka--Volterra, $\alpha$) & 6/12 & 1/12 & 0 \\
TGI-MALA & $-1.98$ & 96 & 2.49 & $-6.46$ (Mixture, $x_1$) & 3/12 & 2/12 & 3 \\
TESS & $-0.39$ & 2.4 & 0.31 & $-1.06$ (Lotka--Volterra, $\alpha$) & 5/12 & 1/12 & 0 \\
NeuTra & $-0.83$ & 6.8 & 0.87 & $-2.40$ (Mixture, $x_1$) & 5/12 & 2/12 & 0 \\
\midrule
\multicolumn{8}{@{}l}{\emph{(B) all test functions, $70$ functions, targets weighted equally}} \\
SNIS & $-0.36$ & 2.3 & 0.58 & $-2.63$ (Sparse logistic, $w_{6}$) & 37/70 & \textbf{31/70} & 0 \\
TGI-RWM & {\boldmath$-0.33$} & \textbf{2.2} & \textbf{0.20} & {\boldmath$-0.98$} (Funnel, $x_1$) & \textbf{41/70} & 16/70 & 0 \\
TGI-MALA & $-1.99$ & 97 & 2.40 & $-6.73$ (Mixture, $\bar x$) & 12/70 & 4/70 & 12 \\
TESS & $-0.43$ & 2.7 & 0.32 & $-1.20$ (GARCH, $\mu$) & 30/70 & 6/70 & 0 \\
NeuTra & $-0.80$ & 6.3 & 0.79 & $-2.40$ (Mixture, $\mathbb 1\{x_1>0\}$) & 37/70 & 21/70 & 0 \\
\bottomrule
\end{tabular}
\end{table*}

To see how far each estimator can fall behind, we measure, for every
row of Table~\ref{tab:wallclock-full}, its gap
$\log_{10}(\text{precision}/\text{best of the row})$ to the most precise
method, and summarize the gaps over the main test functions and over
all test functions with every target weighted equally
(Table~\ref{tab:stability}). SNIS+CV is the best method most often, but
its gaps are the most uneven: it falls behind by a factor of about $50$
on the main function of the $t$ mixture and over $400$ on a coefficient
of the sparse regression, the targets where its weights are heaviest.
TGI-RWM+CV is rarely the best, yet it has the smallest mean gap, about a
factor of $2$, the smallest spread, and the smallest worst case, a factor
of $4.5$ on the main functions and below $10$ on all functions; it is
within a factor of $2$ of the best on more test functions than any other
method. TESS+CV is nearly as stable. The gradient-based chains are the
least predictable, ahead on a few posteriors and two orders of magnitude
behind on the multimodal targets, where TGI-MALA fails outright. When
the diagnostics of Section~\ref{sec:diagnostics} are inconclusive,
TGI-RWM with the control variates is therefore the safe default.

\newpage
\subsection{The independence limit}
\label{app:indep-limit}

At $\gamma=1$, TGI-RWM proposes $z'\sim\gauss(0,I)$ independently of the
current state and is the independence sampler with the transport as
proposal, and the control variates reduce to $\eta_\kappa=\hat c_\kappa$.
Table~\ref{tab:indep-limit} compares the step $\gamma_J$ selected by $J$
with $\gamma=1$ and with the best step $\gamma^\star$ on a grid of $25$
values, on the analytic families, with the same transport and the same
control variates; we write $\mathrm{R}(\gamma,\gamma')$ for the ratio of
ESS$\times$VRF of the main test function at step $\gamma$ to that at step
$\gamma'$. For TGI-RWM on the accurate members, $J$ selects $\gamma$
between $0.95$ and $1$, and $\gamma=1$ is as good as the best step. On
the coarse members the best step is well below $1$, at $0.60$ and $0.45$
on the banana, $0.10$ on the funnel and $0.25$ on the mixture. $J$ moves
part of the way only on the banana, to $0.91$ and $0.48$, where
$\mathrm{R}(1,\gamma_J)$ is $0.88$ and $0.62$; on the funnel and the
mixture it stays at $0.99$, since the speed measure on which $J$ is built
favors the large step at acceptance rates of $0.6$ to $0.7$. Across the
coarse members $\mathrm{R}(\gamma_J,\gamma^\star)$ is $0.44$ to $0.87$,
so $J$ is a proxy that errs towards the independence limit. For
TGI-MALA, $J$ selects $\gamma$ well below $1$ on every member but the
accurate mixture, with ESS$\times$VRF within $10\%$ of that at the best
step except on the Laplace map, and $\mathrm{R}(1,\gamma_J)$ is $0.01$
to $0.6$, as the drift misdirects the large steps and the acceptance rate
drops. On coarser transports the gain lies at $\gamma<1$; $J$ finds it
for TGI-MALA and, on the banana, for TGI-RWM, whereas on the coarse
funnel and mixture a grid search gains a further factor of up to $2.3$.

\begin{table}[ht]
\centering
\small
\caption{The step $\gamma_J$ selected by $J$ against $\gamma=1$ and against
the best step $\gamma^\star$ on a grid of $25$ values, on the analytic
families: acceptance rate at $\gamma_J$ and at $1$, and the ratios
$\mathrm{R}(\gamma,\gamma')$ of ESS$\times$VRF of the main test function at
$\gamma$ to that at $\gamma'$, from $500$ chains of $5000$ steps with common
random numbers.}
\label{tab:indep-limit}
\setlength{\tabcolsep}{3pt}
\begin{tabular}{@{}llr ccccc ccccc@{}}
\toprule
 & & & \multicolumn{5}{c}{TGI-RWM} & \multicolumn{5}{c}{TGI-MALA} \\
\cmidrule(lr){4-8}\cmidrule(lr){9-13}
Target & Member & $\Fisher$ & $\gamma_J$ & $\gamma^\star$ & acc.\ $\gamma_J$/$1$
 & $\mathrm{R}(1,\gamma_J)$ & $\mathrm{R}(\gamma_J,\gamma^\star)$
 & $\gamma_J$ & $\gamma^\star$ & acc.\ $\gamma_J$/$1$
 & $\mathrm{R}(1,\gamma_J)$ & $\mathrm{R}(\gamma_J,\gamma^\star)$ \\
\midrule
Banana  & $c=0.01$ & $7\times10^{-4}$ & 0.95 & 1.00 & 0.99/0.99 & 1.01 & 0.99 & 0.57 & 0.50 & 1.00/0.99 & 0.60 & 0.99 \\
        & $c=0.3$  & $0.63$           & 0.91 & 0.60 & 0.81/0.81 & 0.88 & 0.87 & 0.37 & 0.35 & 0.92/0.69 & 0.01 & 1.00 \\
        & $c=1$    & $7.0$            & 0.48 & 0.45 & 0.60/0.55 & 0.62 & 0.82 & 0.14 & 0.07 & 0.82/0.28 & 0.12 & 0.40 \\
Funnel  & $\varrho=0.01$ & $1.3\times10^{-3}$ & 0.99 & 0.85 & 0.99/0.99 & 1.00 & 1.00 & 0.40 & 0.30 & 0.99/0.98 & 0.22 & 0.90 \\
        & $\varrho=0.3$  & $1.2$            & 0.99 & 0.10 & 0.70/0.70 & 1.00 & 0.44 & 0.22 & 0.30 & 0.94/0.45 & 0.08 & 0.99 \\
Mixture & $s=0.995$ & $10^{-3}$       & 1.00 & 0.75 & 0.99/0.99 & 1.00 & 1.00 & 0.95 & 1.00 & 0.99/0.99 & 1.06 & 0.94 \\
        & $s=0.854$ & $1.0$          & 0.99 & 0.25 & 0.63/0.63 & 1.00 & 0.84 & 0.57 & 0.50 & 0.90/0.72 & 0.42 & 0.99 \\
\bottomrule
\end{tabular}
\end{table}

\end{document}